\documentclass[11pt]{article} 

\usepackage{fullpage}
\usepackage{xcolor}
\usepackage[T1]{fontenc}
\usepackage{natbib}
\usepackage{float}
\usepackage{hyperref}
\usepackage{placeins}
\usepackage{url}
\usepackage{microtype}
\usepackage{graphicx}
\usepackage{subfigure}
\usepackage{booktabs}
\usepackage{bbm}
\usepackage{comment}
\usepackage{amsmath}
\usepackage{amssymb}
\usepackage{amsthm}
\usepackage{algorithmic}
\usepackage{algorithm}
\usepackage{xcolor}
\usepackage{xfrac}
\usepackage{mathtools}
\usepackage{physics}
\usepackage{wrapfig}
\usepackage{nicefrac}
\usepackage{tikz}
\newtheorem{theorem}{Theorem}[section]
\newtheorem{lemma}[theorem]{Lemma}
\newtheorem{proposition}[theorem]{Proposition}
\newtheorem{corollary}[theorem]{Corollary}

\newtheorem{assumption}[theorem]{Assumption}

\theoremstyle{definition}
\newtheorem{definition}[theorem]{Definition}

\theoremstyle{remark}

\usepackage{float}

\title{Beyond Overparameterization: Provable Learning of Input-Convex Multi-Layer Polynomial Networks with Active Queries}

\author{
Jinqi Tang\textsuperscript{\textdagger}
\quad
Qian Chen\textsuperscript{\textdagger}
\quad
Shihong Ding\textsuperscript{\textdagger}
\quad
Cong Fang\textsuperscript{\textdagger}
\\[0.6em]
\textsuperscript{\textdagger}Peking University
}
\date{}

\begin{document}

\maketitle
\begin{abstract}
The theoretical understanding of multi-layer neural networks is largely confined to overparameterized settings, which obscure parameter identifiability and incur high sample complexity. Neural tangent kernel (NTK)  provides a general theory for wide networks, but does not offer efficient sample-complexity guarantees. Recent feature-learning results go beyond kernel methods for single-neuron, multi-index, and hierarchical targets. However, the analysis is often restricted to shallow or specific architectures and to the overparameterized regime.
We break this paradigm to achieve parameter-level recovery of deep target networks, albeit by using active data queries. Specifically, we study $L$-layer polynomial networks with even degree-$k$ monomial activations and nonnegative higher-layer weights. This structure makes the target network input-convex, while the optimization landscape remains highly nonconvex with respect to the parameters. Leveraging input convexity and active queries, we propose \textbf{ASPIRE} (\textbf{A}ctive \textbf{S}am\textbf{P}ling for \textbf{I}terative \textbf{R}ecovery via \textbf{E}igendirections), a layerwise sampling-based diagonalization algorithm that recovers all network parameters to $\delta$-accuracy with sample complexity
$
\widetilde O_{k,L}\left(d^{L^2+O(L)}\delta^{-2e}\right)
$
in polynomial time. To our knowledge, this is the \emph{first} parameter-recovery guarantee for deep target networks whose exponent grows only polynomially with depth, as well as the \emph{first} justification for the effectiveness of using high-quality data in neural network training, with a remarkably \emph{exponential} separation.
\end{abstract}

\section{Introduction}
Deep neural networks have achieved remarkable empirical success,
yet their theoretical understanding remains incomplete, in part
because training is highly nonconvex. A major line of theory
studies the overparameterized regime, where this difficulty becomes
more tractable. In particular, the neural tangent kernel (NTK)
framework shows that sufficiently wide networks can be approximated
by kernel methods, yielding analyzable training dynamics and
provable guarantees~\citep{jacot2018neural,lee2019wide,jiang2025understanding,guillen2026finitewidth}.

However, kernel approximations largely suppress feature learning
and can therefore fail to exploit the hierarchical structure of
deep networks~\citep{yang2021feature}. Beyond the kernel regime,
recent theory has established provable learning of low-dimensional
and hierarchical features~
\citep{damian2022representations,wang2024hierarchical,fu2025hierarchical,ren2026hierarchy}. However, these analyses are often restricted to shallow or specific network architectures and to the overparameterized regime. Under overparameterization, the model may lose the exact structure of the target, often leading to high sample complexity by ignoring inherent structures, as well as a lack of robustness under diverse environments.  This makes exact-parameterization learning particularly appealing. However, there are technical barriers.  Existing works typically can only explore the low-dimensional structures in the first or second layer, whereas the structures in the remaining layers cannot be explored. This is because that the input distributions for the rest layers are often much more complex (see more discussion in Section~\ref{sec:challengens}) and the core technique, Stein's Lemma~\citep{stein1981estimation,yang2017estimating} cannot be applyied. Thus, it remains widely open how to learn multi-layer networks under exact parameterization, even when the activation function is polynomial and the input data distribution is designed.

In recent years, the dominant success of large language models has witnessed the effectiveness of using useful data for training, as data quality, selection, and organization can substantially affect model training and downstream performance \citep{gu2025data,peng2025dataman,sachdeva2026dataefficient,subramanyam2026scaling}, with data access itself treated as part of the learning algorithm. Theoretically, it would be of great interest to provide separation results demonstrating the efficiency of data for neural network training. We study a strong oracle model where the learner can query the target function at arbitrary inputs. We further consider the target drawn from a depth-$L$  polynomial network with non-increasing widths, monomial activations, and nonnegative higher-layer weights. These assumptions endow the target network with a convex structure \citep{amos2017input}, allowing us to explore the strength of data. Input convex networks have been extensively used in scientific problems, where neural networks provide a flexible alternative to traditional methods for approximating operators. Input-convex neural networks offer a principled way to approximate convex functions, with wide-ranging applications spanning optimal transport \citep{makkuva2020optimal}, single-cell perturbation modeling \citep{bunne2023cellot}, model predictive control \citep{wang2024icnn}, and power-system reliability \citep{christianson2025contingency}. 

The input convexity does not make the training problem convex \citep{amos2017input}, it preserves the core optimization difficulty of deep learning. Instead, we propose a layerwise sampling-based diagonalization algorithm that succesfully recovers all parameters to \(\delta\)-accuracy with sample complexity
$\widetilde O_{k,L}\left(d^{L^2+O(L)}\delta^{-2e}\right) $
in polynomial time. This substantially improves the dimension dependence over previous provable approaches for this multi-layer setting. Treating the network as an unstructured degree-\(k^{L-1}\) polynomial leads to an \(O(d^{k^{L-1}})\) sample complexity through direct polynomial interpolation. Similarly, using kernel methods requires \(d^{k^{L-1} -o(1)}\) samples in the non-degenerate case \citep{ghorbani2020linearized}. Many existing feature-learning guarantees rely on a small intrinsic
dimension \(r\ll d\). When \(r\asymp d\), however, dimension reduction
alone does not remove the exponential-in-\(L\) growth of the exponent
of \(d\) in the sample complexity.

The main idea of our proposed algorithm is to construct two second-order moment matrices from active samples drawn uniformly from chosen convex sublevel sets. Under this sampling distribution, the even structure of the network induces coordinatewise sign symmetry, which diagonalizes the corresponding latent moments and allows the joint spectral structure of the two matrices to reveal the hidden directions.
This enables layerwise recovery through spectral decomposition, with a two-Hessian variant for the final hidden layer and output weights. We further establish finite-sample error bounds and validate the theory numerically. 

As a result, our theory breaks the \emph{doubly exponential}-in-depth sample-complexity barrier of the overparameterized paradigm, reducing the exponent from $k^{L-1}$ to $L^2+O(L)$. It also demonstrates that exact parameter recovery of certain deep networks is provably achievable  under active quesries and provides the \emph{first} justification for the effectiveness of
using high-quality data in neural network training with an  exponential
separation.

\section{Related Work}

\paragraph{NTK and Feature Learning.}
Foundational work established the NTK and lazy-training perspective for sufficiently wide neural networks \citep{jacot2018neural,lee2019wide,arora2019exact,du2019gradient,allenzhu2019convergence,chizat2019lazy}. Subsequent work has further developed this framework by studying initialization, kernel evolution, and finite-width corrections \citep{chen2024impacts,jiang2025understanding,guillen2026finitewidth}.
Beyond this regime, provable feature learning encompasses
single-neuron and multi-index models
\citep{yehudai2020single,wu2022single,damian2022representations,mousavi2023lowdim},
as well as nonlinear and hierarchical features in deeper
compositional models
\citep{nichani2023provable,wang2024hierarchical,fu2025hierarchical,tabanelli2026compositional,ren2026hierarchy,ren2026separation}.
These results primarily concern prediction, representations, or
hidden subspaces. We take the next step
toward estimating the target network itself: our algorithm is able to recover all parameters to arbitrarily small error.

\paragraph{Parameter Recovery in Neural Networks.}
Prior work establishes parameter identifiability up to inherent
symmetries for deep neural networks under suitable
conditions \citep{fefferman1994reconstructing,bonapellissier2023parameter}.
Turning these structural guarantees into provably efficient
reconstruction algorithms, however, is substantially more demanding.
For ReLU networks, \citet{rolnick2020reverse} exploit the geometry
of activation boundaries to reconstruct architectures and parameters,
although their analysis does not establish provably efficient
end-to-end recovery.
For polynomial networks, tensor-decomposition arguments establish
generic finite identifiability of deep architectures under suitable
conditions \citep{usevich2025identifiability}.
Related work provides efficient black-box reconstruction of sums
of powers of low-degree polynomials under suitable degree and
non-degeneracy conditions \citep{garg2020learning}.
For deep polynomial networks, existing results either establish identifiability without an efficient recovery procedure, or efficiently recover polynomial components without exploiting their shared hidden-layer structure.
Our algorithm bridges both gaps by exploiting this shared structure iteratively: each recovered layer unlocks query access to the next, yielding end-to-end recovery of all network parameters with rigorous complexity guarantees.

\paragraph{Input-Convex Architectures and Sampling Methods.}
A standard way to guarantee input convexity is to impose
architectural conditions such as convex nondecreasing activations
and nonnegative higher-layer weights \citep{amos2017input}.
Such architectures have found applications in optimal transport,
computational biology, chemical engineering, and energy systems
\citep{makkuva2020optimal,bunne2023cellot,wang2024icnn,christianson2025contingency}. In our setting, convexity permits sampling from convex body using Hit-and-Run, whose polynomial-time mixing is
classical \citep{lovasz1999hit,lovasz2006hitandrun}.
Recent analyses sharpen its guarantees through localization and
Poincar\'e-based bounds
\citep{chen2026hitandrun,zhang2026hitandrun,kook2026spectral}.
Our noisy Hit-and-Run analysis further makes this sampling step \emph{robust} to
the approximation errors introduced by prefix inversion, enabling
its repeated use throughout parameter recovery.

\section{Problem Setup and Preliminaries}
\label{sec:problem-prelim}
Consider a depth-\(L\) fully connected multi-layer polynomial neural
network with architecture
\[
    d=r_0\ge r=r_1\ge r_2\ge\cdots\ge r_{L-1}>r_L=1,
\]
no biases, and monomial activation \(\sigma(t)=t^k\).
We focus on non-increasing-width networks with at least two neurons per hidden layer, as it's a common structural condition for ensuring parameter identifiability \citep{phuong2020functional,usevich2025identifiability}. The bias-free formulation yields a homogeneous polynomial network, a simplified but canonical setting widely studied in the theory of polynomial neural networks \citep{kileel2019expressive,kubjas2024geometry,finkel2025activation,usevich2025identifiability}.
Let weight matrices
\(
    \mathbf W_\ell\in\mathbb R^{r_{\ell-1}\times r_\ell},
    1\le \ell\le L-1,
\)
and 
\(
    \mathbf W_L:=\mathbf a
    \in\mathbb R^{r_{L-1}\times1}.
\)
Define
\[
    \mathbf h_1(\mathbf x)=\mathbf W_1^\top\mathbf x,
    \quad 
    \mathbf h_\ell(\mathbf x)
    =
    \mathbf W_\ell^\top
    \sigma(\mathbf h_{\ell-1}(\mathbf x)),
    \quad 
    2\le\ell\le L-1,
\]
where the activation is applied coordinatewise. The target
function is
\(
    f^\star(\mathbf x)
    =
    \mathbf a^\top
    \sigma(\mathbf h_{L-1}(\mathbf x)).
\)
We are given only real function-value access to \(f^\star\):
for any \(\mathbf x\in\mathbb R^d\), one query returns
\(f^\star(\mathbf x)\). Our goal is to \emph{recover}
\(\mathbf W_1,\ldots,\mathbf W_{L-1}\) and \(\mathbf a\)
from such queries, i.e., using a network with the identical architecture to learn every parameter of the target function. To describe multi-layer networks more conveniently, we introduce prefix and suffix networks as follows. 

\begin{definition}[Prefix and suffix networks]
\label{def:prefix-suffix}
For \(1\le\ell\le L-1\), define the prefix representation
\[
    \boldsymbol{\phi}_\ell(\mathbf x)
    :=
    \begin{cases}
        \mathbf x,
        & \ell=1,\\[1mm]
        \sigma(\mathbf h_{\ell-1}(\mathbf x)),
        & 2\le\ell\le L-1.
    \end{cases}
\]
We further define the suffix polynomials recursively by
\(
    h^{(L-1)}(\mathbf u)
    :=
    \mathbf a^\top\mathbf u^{\odot k},
\)
and, for \(1\le\ell\le L-2\),
\(
    h^{(\ell)}(\mathbf u)
    :=
    h^{(\ell+1)}
    \left(
        \mathbf W_{\ell+1}^\top
        \mathbf u^{\odot k}
    \right).
\)
Finally, define
\(
    F_\ell(\mathbf y)
    :=
    h^{(\ell)}(\mathbf W_\ell^\top\mathbf y),
    \mathbf y\in\mathbb R^{r_{\ell-1}}.
\)
Then, for every \(1\le\ell\le L-1\),
\(
    f^\star(\mathbf x)
    =
    F_\ell(\boldsymbol{\phi}_\ell(\mathbf x)).
\)
\end{definition}


\subsection{Assumptions}
\label{subsec:assumptions}
In this section, we state the assumptions needed for our analysis.

\begin{assumption}[Evenness and nonnegativity]
\label{ass:input-convexity}
The activation degree \(k\ge 4\) is an even integer. Moreover, all higher-layer weight matrices are entrywise nonnegative:
$
    \mathbf W_\ell \ge \mathbf 0,
    2\le \ell\le L.
$
\end{assumption}

The even monomial activation together with nonnegative higher-layer
weights ensures that the network is
input-convex. Convex activation and nonnegativity is a standard mechanism for enforcing input
convexity in neural networks \citep{amos2017input,makkuva2020optimal,bunne2023cellot,wang2024icnn,christianson2025contingency}.We further exclude \(k=2\), since the final suffix then reduces to a quadratic form whose decomposition into the final-hidden-layer directions and output weights is generally non-unique \citep{ottaviani2025generalized,usevich2025identifiability}.

\begin{assumption}[Normalization]
\label{ass:normalization}
We normalize the first-layer columns by
\(
    \|\mathbf w_{1,j}\|_2=1,
    1\le j\le r.
\)
For every \(2\le\ell\le L\), we impose
\(
    \mathbf1_{r_{\ell-1}}^\top\mathbf w_{\ell,j}=1,
    1\le j\le r_\ell.
\)
\end{assumption}

Assumption~\ref{ass:normalization} serves two purposes. For the hidden layers, it selects a canonical representative from the layerwise rescaling equivalence class induced by homogeneity \citep{kileel2019expressive,usevich2025identifiability}. The normalization of the output weights additionally fixes the overall scale of the realized function, in line with standard scale-normalization conventions in neural-network theory \citep{bach2017breaking,yehudai2020single,damian2022representations}. Together with Assumption~\ref{ass:input-convexity}, these normalizations keep the network at a dimension-independent scale under Gaussian inputs:
$$
    \mathbb E_{\mathbf{x}\sim\mathcal N(\mathbf0,\mathbf I_d)}
    \bigl[f^\star(\mathbf{x})^2\bigr]
    =
    \Theta_{k,L}(1).
$$

\begin{assumption}[Parameter conditioning and relative row balance]
\label{ass:conditioning}
There exist constants \(\kappa_0,\mu>0\) such that each \(\mathbf W_\ell\), \(1\le\ell\le L\), has full column rank, and
\[
    \operatorname{cond}(\mathbf W_\ell)\le\kappa_0
    \quad\text{for all } 1\le\ell\le L-1,
    \qquad
    \min_{1\le i\le r_{\ell-1}}
    \bigl(\mathbf W_\ell\mathbf 1_{r_\ell}\bigr)_i
    \ge \mu\,\frac{r_\ell}{r_{\ell-1}}
    \quad\text{for all } 2\le\ell\le L.
\]
\end{assumption}

Since column normalization fixes the average row sum at \(r_\ell/r_{\ell-1}\), the relative row-balance condition simply prevents any hidden neuron from having vanishingly small aggregate influence on the next layer. The condition-number bound further prevents hidden directions
from becoming nearly linearly dependent. Comparable quantitative
non-degeneracy assumptions are standard in provable feature-learning analyses
\citep{damian2022representations,mousavi2023lowdim,
nichani2023provable,wang2024hierarchical,fu2025hierarchical}.

\begin{assumption}[Uniform relative spectral gap]
\label{ass:spectral-gap}
For each \(1\le\ell\le L-2\), define the latent sublevel set
\(
    H_\ell
    :=
    \left\{
        \mathbf u\in\mathbb R^{r_\ell}:
        h^{(\ell)}(\mathbf u)\le1
    \right\},
\)
and let
\(
    \mathbf U_\ell
    \sim
    \operatorname{Unif}(H_\ell).
\)
For \(1\le i\le r_\ell\), define
\(
    \theta_{\ell,i}
    :=
    \mathbb E[U_{\ell,i}^2]\,
    \mathbb E
    \left[
        \left(
            \partial_i h^{(\ell)}(\mathbf U_\ell)
        \right)^2
    \right].
\)
We assume that there exists \(\gamma\ge1\) such that
\[
    \dfrac{
        \displaystyle
        \min_{i\neq j}
        |\theta_{\ell,i}-\theta_{\ell,j}|
    }{
        \displaystyle
        \max_i\theta_{\ell,i}
    }
    \ge
    \gamma^{-1},
    \qquad
    1\le\ell\le L-2.
\]
\end{assumption}

Heuristically, it requires sufficiently strong mixing among the neurons in the network, thereby ruling out certain non-identifiable cases.
If the normalized spectral signatures are in
random position, their minimum spacing is typically of order
\(r_\ell^{-2}\). This suggests that it is reasonable to assume that \(\gamma\) scales polynomially with \(r\). Thus, in this sense, it does not affect the dominant term in the complexity bound of Theorem~\ref{thm:main}. In Appendix~\ref{sec:weight-based-assumptions}, we will construct a concrete example that satisfies all of these assumptions, demonstrating that they are non-vacuous.

\subsection{Challenges}\label{sec:challengens}
\label{subsec:challenges}
Despite the structural assumptions above, two fundamental challenges remain: the optimization difficulties of direct gradient methods and the limitations of overparameterized approaches.

\paragraph{Difficulties of Direct Gradient Methods.}
Although the target network is convex in its input, direct fitting in parameter space remains highly nonconvex. We show that, under our assumptions, there still exist local minima and high-order saddle points in the learning objective, which implies that local search methods such as gradient descent are inefficient.


\begin{proposition}[Local minima and fourth-order saddle barriers in direct gradient-based fitting]
\label{prop:fourth-order-barrier}
Consider realizable matched two-layer teacher--student networks with quartic activation
$\sigma(t)=t^4$ trained under empirical squared loss.
There exist instances satisfying all applicable assumptions in
Subsection~\ref{subsec:assumptions} for which direct parameter fitting exhibits,
respectively, a non-global local minimum and a non-optimal stationary point
whose first descent appears only at fourth order.
\end{proposition}

For the quartic teacher \(f^\star(\mathbf{x})=\frac12(\mathbf{e}_1^\top\mathbf{x})^4+\frac12(\mathbf{e}_3^\top\mathbf{x})^4\), we construct one instance with a strict non-global local minimum and another with a fourth-order saddle. Thus direct fitting can face both bad local minima and higher-order saddles; more generally, finding a fourth-order local minimum is NP-hard~\citep{anandkumar2016efficient}. See Appendix~\ref{sec:direct-fitting-barriers} for details.

\paragraph{Limitations of Overparameterization.}
Overparameterized methods are not naturally suited to exact parameter recovery: while they may achieve good prediction, the enlarged parameterization obscures the underlying target parameters. Moreover, fixed-kernel methods face a severe sample-complexity barrier. Building on the lower bound of \citet[Theorem~3]{ghorbani2020linearized} for rotationally invariant kernels, we show that the infinite-width NTK of our setting requires \(d^{k^{L-1}-o(1)}\) samples.

\begin{proposition}[Barrier for the NTK]
\label{prop:ntk-gaussian-lower-bound}
For every fixed \(\eta>0\), learning our target class with the infinite-width NTK to constant Gaussian population error requires \(\Omega_{k,L,\eta}(d^{k^{L-1}-\eta})\) samples.
\end{proposition}

The proof, given in Appendix~\ref{subsec:ntk-gaussian-lower-bound}, constructs an admissible target with a nonvanishing degree-\(k^{L-1}\) spherical-harmonic component, applies the rotationally invariant kernel lower bound on the sphere, and transfers it to Gaussian inputs using homogeneity.


Recent feature-learning methods improve this scaling by using low-dimensional or hierarchical structure \citep{damian2022representations,wang2024hierarchical,fu2025hierarchical}, but they do not remove the exponential dependence on depth in the exponent of $d$. In their analysis, by assuming a low-dimensional structure of the first layer, they extract the subspace of the first-layer weights, typically using Stein's Lemma~\citep{stein1981estimation}. However, the lifted space loses information needed to identify the original hidden directions. As only the span of the weights can be discovered, one needs to overparameterize the weights in the first layer to guarantee approximation. This makes the input distributions for the subsequent layers complex, as the input often lies on a low-dimensional manifold where Stein's Lemma can no longer be applied to further extract the span of the weights for the remaining dimensions.

\section{Algorithm}
\label{sec:algorithm-main}

In this section, we present our recovery algorithm and state our main results. The algorithm recovers each non-final hidden layer via column-space recovery followed by ASPIRE, and then recovers the final hidden layer and output weights via two-Hessian decomposition. We begin with some useful oracles derived from function-value queries, which serve as building blocks for both stages. 

\subsection{Derived Oracles from Function-Value Queries}
\label{subsec:trick}

Since \(f^\star\) is a homogeneous polynomial, function-value access yields several useful derived oracles.

\paragraph{Directional Derivatives.}
We first describe how to recover directional derivatives. For a degree-\(q\) polynomial \(p\), fix a point \(\mathbf x\), a direction \(\mathbf v\), and an order \(0\le m\le q\). The restriction \(\psi(t):=p(\mathbf x+t\mathbf v)\) is univariate of degree at most \(q\), so its values at any \(q+1\) distinct nodes \(t_0,\ldots,t_q\) determine \(\psi\) exactly, and hence its \(m\)-th derivative \(\psi^{(m)}(0)=D_{\mathbf v}^m p(\mathbf x)\). Thus each directional derivative is recovered exactly using \(q+1\) function-value queries; in particular, gradients and Hessians follow from first- and second-order directional derivatives.

\paragraph{Complex Queries.}
Function-value access at real inputs can also be used to evaluate the complexification of a homogeneous polynomial. Let \(p\) be homogeneous of degree \(q\), and consider a complex point \(\mathbf z=\mathbf u+i\mathbf v\). Choose an integer \(N>2q\) and let \(\theta_j=2\pi j/N\), \(j=0,\ldots,N-1\). The degree-\(q\) Fourier coefficient of these evaluations then recovers the complex value exactly:
$$
p_{\mathbb C}(\mathbf z)
=
\dfrac{2^q}{N}
\sum_{j=0}^{N-1}
e^{iq\theta_j}p(\mathbf u\cos\theta_j+\mathbf v\sin\theta_j).
$$
Thus, a complex query to \(p_{\mathbb C}\) can be simulated using \(O(q)\) real function-value queries.

\paragraph{Suffix Oracles.}
Given the preceding layers, we can further simulate direct query access to the suffix polynomial $F_\ell$. Fix $2\le\ell\le L-1$, suppose that $\mathbf W_1,\ldots,\mathbf W_{\ell-1}$ are known, and consider a desired suffix query $\mathbf y\in\mathbb R^{r_{\ell-1}}$. Starting from $\mathbf z_{\ell-1}=\mathbf y$, we invert the prefix recursively from layer $\ell-1$ down to layer $1$. At layer $j$, choose coordinatewise principal $k$-th roots $\boldsymbol\xi_j\in\mathbb C^{r_j}$ satisfying $\boldsymbol\xi_j^{\odot k}=\mathbf z_j$, and then set $\mathbf z_{j-1}=\mathbf W_j(\mathbf W_j^\top\mathbf W_j)^{-1}\boldsymbol\xi_j$. Because $\mathbf W_j$ has full column rank, this choice satisfies $\mathbf W_j^\top\mathbf z_{j-1}=\boldsymbol\xi_j$, and hence $(\mathbf W_j^\top\mathbf z_{j-1})^{\odot k}=\mathbf z_j$. Repeating this recursion produces a complex input $\mathbf z_0\in\mathbb C^d$ whose forward propagation through the first $\ell-1$ layers reaches exactly $\mathbf y$. Combining this prefix inversion with the preceding complex-evaluation procedure, we can evaluate $f^\star_{\mathbb C}(\mathbf z_0)=F_\ell(\mathbf y)$ using only real function-value queries to $f^\star$. 

\subsection{Layerwise Recovery}

We now turn to the core layerwise recovery procedure. Assuming the preceding weights \(\mathbf W_1,\ldots,\mathbf W_{\ell-1}\) have been recovered sufficiently accurately, we recover \(\mathbf W_\ell\) in two stages: first estimate its column space from gradient information, and then apply \textsc{ASPIRE} within the recovered subspace to identify the individual hidden directions.
\paragraph{Stage I: Column-Space Recovery.}
At suitable query points, the chain rule yields suffix gradients spanning \(\operatorname{col}(\mathbf W_\ell)\). We choose query points \(\mathbf x_1,\ldots,\mathbf x_m\) such that the recovered prefix Jacobians \(\widehat{\mathbf J}_{\ell,i}:=D\widehat{\boldsymbol{\phi}}_\ell(\mathbf x_i)\) have full row rank. For each query point \(\mathbf x_i\), we estimate the suffix gradient without explicitly evaluating the full gradient \(\nabla f^\star(\mathbf x_i)\). For every \(j=1,\ldots,r_{\ell-1}\), set \(\mathbf v_{i,j}:=\widehat{\mathbf J}_{\ell,i}^\top\mathbf e_j\) and recover the directional derivative \(b_{i,j}:=D_{\mathbf v_{i,j}}f^\star(\mathbf x_i)\) from function-value queries using polynomial interpolation stated in Subsection~\ref{subsec:trick}. Collecting these derivatives into \(\mathbf b_i:=(b_{i,1},\ldots,b_{i,r_{\ell-1}})^\top\), we compute \(\widehat{\mathbf g}_i:=(\widehat{\mathbf J}_{\ell,i}\widehat{\mathbf J}_{\ell,i}^\top)^{-1}\mathbf b_i\), which approximates the corresponding suffix gradient \(\mathbf g_i\). We then form the empirical second moment \(\widehat{\mathbf M}_{\mathrm{span},\ell}:=\frac1m\sum_{i=1}^m\widehat{\mathbf g}_i\widehat{\mathbf g}_i^\top\). Because the true suffix gradients lie in \(\operatorname{col}(\mathbf W_\ell)\), and sufficiently diverse query points make these gradients span this subspace, the leading \(r_\ell\)-dimensional eigenspace of \(\widehat{\mathbf M}_{\mathrm{span},\ell}\) estimates \(\operatorname{col}(\mathbf W_\ell)\). We therefore let \(\widehat{\mathbf B}_\ell\in\mathbb R^{r_{\ell-1}\times r_\ell}\) contain the orthonormal eigenvectors corresponding to the \(r_\ell\) largest eigenvalues of \(\widehat{\mathbf M}_{\mathrm{span},\ell}\), yielding \(\operatorname{col}(\widehat{\mathbf B}_\ell)\approx\operatorname{col}(\mathbf W_\ell)\).

\paragraph{Stage II: Individual-Direction Recovery.}
Recovering the column space does not identify the individual hidden directions within it. Writing
\(\mathbf W_\ell=\mathbf B_\ell\mathbf C_\ell\), \textsc{ASPIRE} (Algorithm~\ref{alg:aspire}) resolves this ambiguity by actively designing a sampling distribution that creates an identifiable spectral structure. We sample approximately uniformly from a convex sublevel set of the reduced suffix and form position and gradient second moments. After mapping the samples into the latent coordinates induced by \(\mathbf C_\ell\), the sublevel set becomes symmetric under independent sign flips of each coordinate. This symmetry cancels all cross-coordinate interactions in both moments, so they share a common diagonal structure in the latent basis. As a result, combining the two moments isolates the latent coordinate directions, which correspond exactly to the columns of \(\mathbf C_\ell\); Assumption~\ref{ass:spectral-gap} ensures that these directions are separated and hence identifiable.

\begin{algorithm}[t]
\caption{\textbf{A}ctive \textbf{S}am\textbf{P}ling for
\textbf{I}terative \textbf{R}ecovery via \textbf{E}igendirections
(\textbf{ASPIRE})}
\label{alg:aspire}
\small
\begin{algorithmic}
\REQUIRE Layer \(1\le\ell\le L-2\), recovered subspace basis
\(\widehat{\mathbf B}_\ell\), simulated suffix oracle
\(\widehat F_\ell\), number of samples \(m\), and sampling
accuracy parameters.

\STATE Let \(G_\ell(\mathbf z):=
F_\ell(\widehat{\mathbf B}_\ell\mathbf z)\) and
\(K_\ell^\sharp:=\{\mathbf z:G_\ell(\mathbf z)\le1\}\)
denote the reduced suffix and target sampling body.
\STATE Define the available reduced oracle
\(\widetilde F_\ell(\mathbf z):=
\widehat F_\ell(\widehat{\mathbf B}_\ell\mathbf z)\).

\FOR{\(i=1,\ldots,m\)}
    \STATE Independently run noisy Hit-and-Run from
    \(\mathbf0\) using \(\widetilde F_\ell\), as in
    Lemma~\ref{lem:noisy-hit-and-run}, and obtain endpoint
    \(\mathbf z_i\).
    \STATE Estimate \(\nabla G_\ell(\mathbf z_i)\) by
    \(\mathbf q_i\), using evaluations of \(\widetilde F_\ell\)
    and the stable interpolation procedure of
    Lemma~\ref{lem:stable-gradient-from-suffix}.
\ENDFOR

\STATE Form the empirical moments
\(\widehat{\boldsymbol{\Sigma}}_{x,\ell}
:=\frac1m\sum_{i=1}^m\mathbf z_i\mathbf z_i^\top\)
and
\(\widehat{\boldsymbol{\Sigma}}_{g,\ell}
:=\frac1m\sum_{i=1}^m\mathbf q_i\mathbf q_i^\top\).
\STATE Compute normalized eigenvectors
\(\widehat{\mathbf v}_{\ell,1},\ldots,
\widehat{\mathbf v}_{\ell,r_\ell}\) of
\(\widehat{\mathbf A}_\ell
:=\widehat{\boldsymbol{\Sigma}}_{g,\ell}
\widehat{\boldsymbol{\Sigma}}_{x,\ell}\).

\FOR{\(j=1,\ldots,r_\ell\)}
    \STATE Set
    \(\widehat{\mathbf u}_{\ell,j}
    :=\widehat{\mathbf B}_\ell\widehat{\mathbf v}_{\ell,j}\).
    \IF{\(\ell=1\)}
        \STATE Set
        \(\widehat{\mathbf w}_{1,j}
        :=\widehat{\mathbf u}_{1,j}/
        \|\widehat{\mathbf u}_{1,j}\|_2\),
        with an arbitrary choice of sign.
    \ELSE
        \STATE Set
        \(\widehat{\mathbf w}_{\ell,j}
        :=|\widehat{\mathbf u}_{\ell,j}|/
        (\mathbf1^\top|\widehat{\mathbf u}_{\ell,j}|)\),
        where the absolute value is coordinatewise.
    \ENDIF
\ENDFOR

\STATE \textbf{return}
\(\widehat{\mathbf W}_\ell
:=(\widehat{\mathbf w}_{\ell,1},\ldots,
\widehat{\mathbf w}_{\ell,r_\ell})\).
\ENSURE An estimate of \(\mathbf W_\ell\), up to column
permutation and, for \(\ell=1\), column signs.
\end{algorithmic}
\end{algorithm}

\subsection{Final Hidden Layer and Output Weights}

The final suffix has coincident spectral signatures, so we use two Hessians. Take $\mathbf y^{(2)}=\mathbf1_{r_{L-2}}$ and $\mathbf y^{(1)}=\mathbf1_{r_{L-2}}+\tau\mathbf u$, where $\mathbf u\sim\operatorname{Unif}(S^{r_{L-2}-1})$ and $\tau>0$ is sufficiently small and fixed. Construct symmetric estimates $\widehat{\mathbf H}_j\approx\nabla^2F_{L-1}(\mathbf y^{(j)})$, $j=1,2$, by fixed-node interpolation of approximate suffix queries (Subsection~\ref{subsec:trick}).

Let $\widehat{\mathbf B}\in\mathbb R^{r_{L-2}\times r_{L-1}}$ contain orthonormal eigenvectors for the $r_{L-1}$ largest eigenvalues of $\widehat{\mathbf H}_2$, and set $\widehat{\mathbf M}_j=\widehat{\mathbf B}^\top\widehat{\mathbf H}_j\widehat{\mathbf B}$, $j=1,2$. Under the prefix-accuracy condition of Appendix~F, $\widehat{\mathbf M}_2$ is invertible. The eigenvectors of $\widehat{\mathbf M}_1\widehat{\mathbf M}_2^{-1}$ estimate the projected hidden directions. Mapping them through $\widehat{\mathbf B}$, taking coordinatewise absolute values, and normalizing columns to unit $\ell_1$-norm yields $\widehat{\mathbf W}_{L-1}$. Finally, set
$
\widehat{\mathbf a}:=
\dfrac{1}{k(k-1)}
\operatorname{diag}\left(
\widehat{\mathbf W}_{L-1}^{\dagger}
\widehat{\mathbf H}_2
\widehat{\mathbf W}_{L-1}^{\dagger\top}
\right).
$

\section{Main Results}
Combining the preceding steps, we are now ready to state our main parameter recovery theorem, together with two corollaries on generalization error.

\begin{theorem}[Parameter recovery]
\label{thm:main}
Let $k\ge4$ be even and $L\ge3$.  Define the conditioning envelope
$
\kappa:=\max\{2,\kappa_0,\mu^{-1}\}
$. Suppose that
Assumptions~\ref{ass:input-convexity},~\ref{ass:normalization},~\ref{ass:conditioning}, and~\ref{ass:spectral-gap} hold.
For every $0<\delta<\dfrac12$, our algorithm \footnote{The algorithm is given the target architecture and the
parameters $\kappa_0,\mu,\gamma$ from
Subsection~\ref{subsec:assumptions}.} makes at most
$
\widetilde O_{k,L}\left(
dr+r^{L^2+O(L)}
\gamma^{O(L)}
\kappa^{O(L^2)}
\delta^{-2\mathrm e}
\right)
\label{eq}
$
real function-value queries to $f^\star$ and, with probability
at least $0.99$, returns
$\widehat{\mathbf W}_1,\cdots,
\widehat{\mathbf W}_{L-1},\widehat{\mathbf a}$
satisfying
$$
\max_{1\le\ell\le L-1}
\|\widehat{\mathbf W}_{\ell}-\mathbf W_{\ell}\|_2
\le\delta,
\qquad
\|\widehat{\mathbf a}-\mathbf a\|_1\le\delta,
$$
after a consistent relabeling of the hidden neurons and
a suitable choice of the first-layer column signs.
For fixed $k,L$, the algorithm runs in polynomial time \footnote{We work in the unit-cost exact real-arithmetic model.}.
\end{theorem}

Using $r\le d$ and suppressing the dependence on $\gamma$ and $\kappa$ gives a simpler
bound appearing in abstract and introduction.
Gaussian stability then converts parameter accuracy into
prediction accuracy, yielding the following corollary.

\begin{corollary}[Gaussian generalization error]
\label{cor:gaussian_generalization_error}
Under the assumptions of Theorem~\ref{thm:main}, for every
$0<\varepsilon<\dfrac12$, our algorithm is able to make at most
$
\widetilde O_{k,L}\left(
dr+r^{L^2+O(L)}
\gamma^{O(L)}
\kappa^{O(L^2)}
\varepsilon^{e}
\right)
$
real function-value queries to $f^\star$ and, with probability
at least $0.99$, returns a polynomial network $\widehat f$
with the target architecture satisfying
$
\mathbb E_{\mathbf x\sim\mathcal N(\mathbf0,\mathbf I_d)}
\left[
\bigl(\widehat f(\mathbf x)-f^\star(\mathbf x)\bigr)^2
\right]
\le\varepsilon.
$
\end{corollary}

The recovered polynomial network can also be converted into
a ReLU network by replacing its monomial activations with
piecewise-linear approximations.

\begin{corollary}[ReLU network]
\label{cor:relu}
Under the assumptions of Theorem~\ref{thm:main}, for every
$0<\varepsilon<\dfrac12$, our algorithm, with ReLU postprocessing,
makes at most
$
\widetilde O_{k,L}\left(
dr+r^{L^2+O(L)}
\gamma^{O(L)}
\kappa^{O(L^2)}
\varepsilon^{e}
\right)
$
real function-value queries to $f^\star$ and, with probability
at least $0.99$, returns a ReLU network
$f_{\mathrm{ReLU}}$ with biases satisfying
$
\mathbb E_{\mathbf x\sim\mathcal N(\mathbf0,\mathbf I_d)}
\left[
\bigl(f_{\mathrm{ReLU}}(\mathbf x)
-f^\star(\mathbf x)\bigr)^2
\right]
\le\varepsilon.
$
The network has $L$ layers
with hidden widths
$
r_\ell^{\mathrm{ReLU}}
=\widetilde O_{k,L}\left(
r_\ell\varepsilon^{-1/4}
\right)
$.
\end{corollary}

\section{Proof Sketch}
\label{sec:proof-sketch}
In this section, we provide a proof sketch of the main parameter recovery theorem.
To establish the correctness of our algorithm, we first consider a simple setting that captures the essential idea, and then explain how the sample complexity of the algorithm is derived.

\subsection{A Simple Setting}
We consider the population-level first-layer recovery when \(d=r\), ignoring column-space estimation, prefix inversion, and finite-sample errors, and directly analyze the correctness of ASPIRE.

Let \(K_1:=\{\mathbf x:f^\star(\mathbf x)\le1\}\) and draw \(\mathbf X\sim\operatorname{Unif}(K_1)\). Since \(f^\star(\mathbf x)=h^{(1)}(\mathbf W_1^\top\mathbf x)\), the change of variables \(\mathbf U:=\mathbf W_1^\top\mathbf X\) gives \(\mathbf U\sim\operatorname{Unif}(H_1)\), where \(H_1:=\{\mathbf u:h^{(1)}(\mathbf u)\le1\}\). Because \(k\) is even, \(h^{(1)}\) is invariant under coordinatewise sign flips. Hence the distribution of \(\mathbf U\) has the same symmetry. Let \(\boldsymbol\Lambda_1:=\mathbb E[\mathbf U\mathbf U^\top]\). For \(i\neq j\), flipping only the \(i\)-th coordinate leaves the distribution unchanged but changes \(U_iU_j\) to \(-U_iU_j\). Therefore \(\mathbb E[U_iU_j]=-\mathbb E[U_iU_j]=0\), so \(\boldsymbol\Lambda_1\) is diagonal, with diagonal entries \(\lambda_{1,i}:=\mathbb E[U_i^2]\). Similarly, let \(\mathbf D_1:=\mathbb E[\nabla h^{(1)}(\mathbf U)\nabla h^{(1)}(\mathbf U)^\top]\). From \(h^{(1)}(\mathbf S\mathbf u)=h^{(1)}(\mathbf u)\) for every diagonal sign matrix \(\mathbf S\), we have \(\nabla h^{(1)}(\mathbf S\mathbf u)=\mathbf S\nabla h^{(1)}(\mathbf u)\). Thus, for \(i\neq j\), flipping the \(i\)-th coordinate changes \(\partial_i h^{(1)}(\mathbf U)\partial_j h^{(1)}(\mathbf U)\) to its negative while preserving its distribution. Hence \((\mathbf D_1)_{ij}=0\), so \(\mathbf D_1\) is diagonal, with \(d_{1,i}:=\mathbb E[(\partial_i h^{(1)}(\mathbf U))^2]\). Since \(\mathbf X=\mathbf W_1^{-\top}\mathbf U\) and \(\nabla f^\star(\mathbf X)=\mathbf W_1\nabla h^{(1)}(\mathbf U)\), we obtain \(\boldsymbol\Sigma_{x,1}=\mathbf W_1^{-\top}\boldsymbol\Lambda_1\mathbf W_1^{-1}\) and \(\boldsymbol\Sigma_{g,1}=\mathbf W_1\mathbf D_1\mathbf W_1^\top\). Therefore
$$
\boldsymbol\Sigma_{g,1}\boldsymbol\Sigma_{x,1}
=
\mathbf W_1
\operatorname{diag}(\theta_{1,1},\ldots,\theta_{1,r})
\mathbf W_1^{-1},\qquad 
\theta_{1,i}:=\lambda_{1,i}d_{1,i}.$$
By Assumption~\ref{ass:spectral-gap}, the \(\theta_{1,i}\)'s are distinct. Hence the eigendirections of \(\boldsymbol\Sigma_{g,1}\boldsymbol\Sigma_{x,1}\) are exactly the columns of \(\mathbf W_1\). After normalization, ASPIRE recovers \(\mathbf W_1\) up to column permutation and sign.

\subsection{Dimension and Accuracy Dependence}
\label{subsec:proof-sketch-finite-sample}

Throughout this subsection, we treat \(k\) and \(L\) as fixed constants and suppress dependence on \(\kappa,\gamma\), logarithmic factors, and failure probabilities, focusing on dimension and accuracy. The full proof is technically involved, so we present only the main ideas here and defer the details to the appendix.

\begin{proposition}[Stable column-space recovery]
\label{prop:sketch-column-space}
Fix $2\le\ell\le L-2$, and $\varepsilon_B<1/2$.
Suppose all preceding layers satisfy
$\|\widehat{\mathbf W}_j-\mathbf W_j\|_2
\lesssim\varepsilon_B r^{-7/2}$.
Then algorithm can recover an orthonormal basis
$\widehat{\mathbf B}_\ell$ such that, after orthogonal alignment,
$\|\widehat{\mathbf B}_\ell-\mathbf B_\ell\|_2\le\varepsilon_B$ using
$\widetilde O(r^7\varepsilon_B^{-2})$ queries.
\end{proposition}

We choose query points so that the estimated first-layer activations are
uniform on $[1,2]^r$, keeping coordinates comparable and the prefix
Jacobian well-conditioned. See Proposition~\ref{prop:stable-column-space}.

\paragraph{Decomposing the Moment Errors.}
Suppress the layer index. Let $\widehat\pi$ be the endpoint distribution
of noisy Hit-and-Run. Write $\mathbf g=\nabla G$, let $\widehat{\mathbf g}$ be its
fixed-node interpolation estimate from an oracle with
multiplicative error $\xi$, and let $\mathbb P_m$ denote
empirical averaging over $m$ independent samples from
$\widehat\pi$, supported on $K$.
The population moments are
$\boldsymbol\Sigma_x=\mathbb E_\pi[\mathbf z\mathbf z^\top]$
and $\boldsymbol\Sigma_g=\mathbb E_\pi[\mathbf g\mathbf g^\top]$.
Their finite-sample errors decompose exactly as
$$
\widehat{\boldsymbol\Sigma}_x-\boldsymbol\Sigma_x
=
\underbrace{
(\mathbb P_m-\mathbb E_{\widehat\pi})
[\mathbf z\mathbf z^\top]
}_{\text{empirical error}}
+
\underbrace{
(\mathbb E_{\widehat\pi}-\mathbb E_\pi)
[\mathbf z\mathbf z^\top]
}_{\text{distribution error}},
$$
$$
\widehat{\boldsymbol\Sigma}_g-\boldsymbol\Sigma_g=
\underbrace{
(\mathbb P_m-\mathbb E_{\widehat\pi})
[\widehat{\mathbf g}\widehat{\mathbf g}^\top]
}_{\text{empirical error}}
+
\underbrace{
\mathbb E_{\widehat\pi}
[\widehat{\mathbf g}\widehat{\mathbf g}^\top
-\mathbf g\mathbf g^\top]
}_{\text{gradient-estimation error}}
+
\underbrace{
(\mathbb E_{\widehat\pi}-\mathbb E_\pi)
[\mathbf g\mathbf g^\top]
}_{\text{distribution error}}.
$$

\begin{proposition}[Accurate estimation of the two moments]
\label{prop:sketch-two-moments}
For $\eta<1/2$, suppose
$d_{\mathrm{TV}}(\widehat\pi,\pi):=\nu \lesssim\eta r^{-2}$
and the reduced suffix oracle has uniform multiplicative error
$\xi\lesssim\eta r^{-5/2}$.
Then $m=\widetilde O(r^2\eta^{-2})$ independent samples as above
give
$\|\widehat{\boldsymbol\Sigma}_a-\boldsymbol\Sigma_a\|_2
\le\eta\|\boldsymbol\Sigma_a\|_2$
for both $a=x,g$.
\end{proposition}

The distribution and interpolation errors are
\(O(r^2\nu)\) and \(O(r^{5/2}\xi)\), while matrix Bernstein controls
the empirical error (Lemma~\ref{lem:empirical-two-moments}).
This gives \(O(r^3\eta)\) error per eigendirection
(Lemmas~\ref{lem:population-moment-scales}
and~\ref{lem:noisy-two-matrix}).

\paragraph{Sampling to the Required Accuracy.}
Classical Hit-and-Run mixes in $\widetilde O(r^4)$ steps for bodies
with radius ratio $O(\sqrt r)$ \citep{lovasz2006hitandrun}.
For our inexact suffix oracle, we use the following noisy variant.

\begin{proposition}[Noisy Hit-and-Run]
\label{prop:sketch-noisy-hit-and-run}
For $\nu<1/2$, suppose the reduced oracle has
multiplicative error
$\xi\lesssim\nu^{1+1/L}r^{-O(1)}$. Starting from the origin, noisy Hit-and-Run produces an endpoint
after $O(r^4\log(r/\nu))$ steps whose distribution has total
variation distance at most $\nu$ from $\operatorname{Unif}(K)$.
\end{proposition}

The multivariate Markov inequality \citep{wilhelmsen1974markov}
controls the polynomial Hessian; together with random-direction
anti-concentration and chord comparisons, this yields negative-moment
bounds on boundary depth along the exact chain. See
Lemmas~\ref{lem:second-derivative-markov} and~\ref{lem:noisy-hit-and-run}.

\begin{proposition}[Stable recovery of the first layer]
\label{prop:sketch-first-layer-recovery}
For any $\varepsilon < 1/2$, the first hidden layer can be
recovered with
$\|\widehat{\mathbf W}_1-\mathbf W_1\|_2\le\varepsilon$, up to inherent symmetries, using
$\widetilde O(dr+r^{13}\varepsilon^{-2})$
queries.
\end{proposition}

By Proposition~\ref{lem:first-layer-column-space}, \(r\) generic
gradients recover the first-layer column space almost surely using
\(O(dr)\) queries, after which the reduced oracle and gradients are exact.
Converting columnwise direction error to operator norm costs \(\sqrt r\),
so take \(\eta\asymp\varepsilon r^{-7/2}\) and
\(\nu\asymp\eta r^{-2}\).
Propositions~\ref{prop:sketch-two-moments} and~\ref{prop:sketch-noisy-hit-and-run}
then give \(m=\widetilde O(r^9\varepsilon^{-2})\) endpoints,
each costing \(\widetilde O(r^4)\) queries, proving the claim
(Appendix~\ref{sec:appendix-first-layer}).
Only this initial column-space recovery depends on \(d\); deeper layers use
the same sampling argument with approximate prefix inversion.

\begin{proposition}[Stable recovery of one intermediate layer]
\label{prop:sketch-layerwise-recovery}
Fix $2\le\ell\le L-2$, $\varepsilon<1/2$, and
$\alpha:=1+1/L$.
Suppose preceding layers are recovered with operator-norm
error at most $\varepsilon_W$.
If $\varepsilon_W\lesssim
\varepsilon^\alpha r^{-\ell/2-O(1)}$,
then column-space recovery followed by ASPIRE returns
$\widehat{\mathbf W}_\ell$ satisfying
$\|\widehat{\mathbf W}_\ell-\mathbf W_\ell\|_2\le\varepsilon$,
up to column permutation, using
$\widetilde O(r^{14}\varepsilon^{-2})$
additional queries.
\end{proposition}

Choose \(\varepsilon_B\asymp\varepsilon/r\) and
\(\eta\asymp\varepsilon r^{-4}\).
The prefix condition gives column-space error \(\varepsilon_B\), while
mapping back and normalizing the recovered directions yields
\(O(r(\varepsilon_B+r^3\eta))\le\varepsilon\)
(Proposition~\ref{prop:stable-convex-body-recovery}).
Approximate prefix inversion gives
\(\xi_\ell\lesssim r^{\ell/2}\varepsilon_W\)
(Lemma~\ref{lem:multiplicative-suffix-oracle}).
Taking \(\nu\asymp\eta r^{-2}\asymp\varepsilon r^{-6}\) in
Proposition~\ref{prop:sketch-noisy-hit-and-run} gives
\(\varepsilon_W\lesssim\varepsilon^\alpha r^{-\ell/2-O(1)}\),
which also ensures gradient accuracy.
Proposition~\ref{prop:sketch-two-moments} then needs
\(m=\widetilde O(r^{10}\varepsilon^{-2})\) endpoints, for total cost
\(\widetilde O(r^{14}\varepsilon^{-2})\).
Thus \(\alpha=1+1/L\) comes from the sampling guarantee, not gradient estimation.

\paragraph{Closing the Recovery and Accumulating Errors.}
By Proposition~\ref{prop:stable-final-layer}, the final two-Hessian stage succeeds with \(O(r^2)\) additional queries provided the preceding layers are accurate to \(\varepsilon_W\lesssim r^{-L/2-O(1)}\delta\). Set \(\tau_{L-2}=c\,r^{-L/2-O(1)}\delta\) and propagate tolerances backwards via Proposition~\ref{prop:sketch-layerwise-recovery}: \(\tau_{\ell-1}\asymp\tau_\ell^\alpha r^{-\ell/2-O(1)}\), where \(\alpha=1+1/L\). Unrolling gives \(p_\star=\alpha^{L-3}<e\) and \(\tau_1^{-1}=\widetilde O\!\left(r^{L^2/2+O(L)}\delta^{-p_\star}\right)\). Since \(\tau_1\) is smallest, it dominates the recovery costs, yielding \(Q_{\mathrm{tot}}\le\widetilde O\!\left(dr+r^{14}\tau_1^{-2}\right)=\widetilde O\!\left(dr+r^{L^2+O(L)}\delta^{-2p_\star}\right)\). Finally, \(p_\star<e\) gives Theorem~\ref{thm:main}; full bookkeeping is deferred to Appendix~\ref{sec:simplified-conditioning-bookkeeping}.

\section{Conclusion}
Our results establish efficient active-query learning for a class of convex functions with deep compositional structure beyond the NTK regime. More broadly, they suggest that active data access offers a fundamentally different route to provable deep learning, beyond the limitations of overparameterization paradigm. Numerical experiment details are provided in Appendix~\ref{app:activation_experiments}.

\subsection*{AI Use Statement}

Generative AI tools were used to assist with language polishing and the generation of \LaTeX{} and Python code. All AI-assisted content was carefully checked and verified by the authors, who take full responsibility for the final content and the correctness of the work.

\bibliographystyle{plainnat}
\bibliography{main}

\newpage 
\tableofcontents

\appendix
\paragraph{Roadmap of the Appendix.}
The appendix is organized as follows.
Appendix~\ref{sec:appendix-perturbation-tools} collects several
mostly standard perturbation bounds and numerical linear-algebra
tools.
Appendix~\ref{sec:appendix-sampling-moments} contains the main
sampling machinery of the paper. In particular, we develop a
quantitative analysis of Hit-and-Run when the defining homogeneous
polynomial is accessed only through a multiplicatively noisy
oracle, and use it to obtain the moment estimates required by the
spectral recovery procedure.
Appendix~\ref{sec:appendix-first-layer} applies these tools to the
first hidden layer, where the relevant column space can be
recovered exactly.
Appendix~\ref{sec:appendix-suffix-stability} then studies the error
introduced by approximately inverting the recovered prefix, which
allows us to simulate access to the suffix networks at subsequent
layers.
Appendix~\ref{sec:appendix-intermediate-layers} combines the
subspace, sampling, and suffix-oracle analyses to establish stable
recovery of the intermediate hidden layers.
Appendix~\ref{sec:appendix-final-layer} treats the final hidden
layer and the output weights separately using a two-Hessian
spectral argument.
Finally,
Appendix~\ref{sec:simplified-conditioning-bookkeeping}
combines the layerwise guarantees to prove 
Theorem~\ref{thm:main},
Corollaries~\ref{cor:gaussian_generalization_error} and~\ref{cor:relu}. Appendix~\ref{sec:weight-based-assumptions} gives an explicit family
satisfying our structural assumptions;
Appendix~\ref{app:Proof_of_Challenge} provides the proofs underlying
the challenges discussed in Appendix~\ref{sec:challengens};
and Appendix~\ref{app:activation_experiments} provides additional
details of the numerical experiments.

\section{Technical Tools for Stability Analysis}
\label{sec:appendix-perturbation-tools}
	Throughout the section, we work under the assumptions
	stated in
	subsection~\ref{subsec:assumptions}. We first record several
	quantitative
	consequences that will be used in the
	perturbation analysis.

\subsection{Consequences of Normalization and Conditioning}

The following lemma translates the normalization and conditioning assumptions into explicit singular-value and right-inverse bounds for the hidden weight matrices.
	
	\begin{lemma}[Singular-value bounds]
		\label{lem:weight-singular-values}
		We have
		\(
		\dfrac{1}{\kappa_0}
		\le
		\sigma_{\min}(\mathbf{W}_1)
		\le
		1
		\le
		\sigma_{\max}(\mathbf{W}_1)
		\le
		\kappa_0,
		\)
		and for every \(2\le\ell\le L-1\),
		\(
		\dfrac1{\kappa_0}
		\sqrt{\dfrac{r_\ell}{r_{\ell-1}}}
		\le
		\sigma_{\min}(\mathbf{W}_\ell)
		\le
		\sigma_{\max}(\mathbf{W}_\ell)
		\le
		\kappa_0.
		\)
		Consequently,
		\[
		\left\|
		\mathbf W_\ell
		(\mathbf W_\ell^\top\mathbf W_\ell)^{-1}
		\right\|_2
		\le
		\kappa_0\sqrt{\dfrac{r_{\ell-1}}{r_\ell}},
		\qquad
		2\le\ell\le L-1.
		\]
	\end{lemma}
	
	\begin{proof}
	For \(\mathbf{W}_1\), every column has been normalized, therefore
	\(
	\|\mathbf{W}_1\|_F^2
	=
	\sum_{j=1}^{r}\|\mathbf{w}_{1,j}\|_2^2
	=
	r.
	\)
	Since the squared Frobenius norm equals the sum of the squared singular values, \(\|\mathbf W_1\|_F^2=\sum_{j=1}^{r}\sigma_j(\mathbf W_1)^2\).
	Consequently, the average squared singular value is one,
	so
	\(
	\sigma_{\min}(\mathbf{W}_1)
	\le1\le
	\sigma_{\max}(\mathbf{W}_1).
	\)
	Since
	\(\operatorname{cond}(\mathbf{W}_1)=\dfrac{\sigma_{\max}(\mathbf{W}_1)}{\sigma_{\min}(\mathbf{W}_1)}
	\le
	\kappa_0,
	\)
	we obtain
	\[
	\dfrac1{\kappa_0}
	\le
	\sigma_{\min}(\mathbf{W}_1)
	\le
	1
	\le
	\sigma_{\max}(\mathbf{W}_1)
	\le
	\kappa_0.
	\]
	Now fix \(2\le\ell\le L-1\). Since \(\mathbf
	W_\ell\ge\mathbf0\)
	and every column has unit \(\ell_1\)-norm,
	\(
	\mathbf W_\ell^\top\mathbf1_{r_{\ell-1}}=\mathbf1_{r_\ell}.
	\)
	Testing \(\mathbf W_\ell^\top\) on the all-ones vector
	yields
	\[
	\sigma_{\max}(\mathbf W_\ell)
	\ge
	\dfrac{
	\|\mathbf W_\ell^\top\mathbf1_{r_{\ell-1}}\|_2
	}{
	\|\mathbf1_{r_{\ell-1}}\|_2
	}
	=
	\sqrt{\dfrac{r_\ell}{r_{\ell-1}}}.
	\]
	The condition-number bound therefore gives
	\(	\sigma_{\min}(\mathbf W_\ell)
	\ge \dfrac1{\kappa_0}\sqrt{\dfrac{r_\ell}{r_{\ell-1}}}.
	\)
	For each column,
	\(
	\|\mathbf w_{\ell,j}\|_2
	\le
	\|\mathbf w_{\ell,j}\|_1
	=1,
	\)
	and hence
	\(
	\|\mathbf W_\ell\|_F^2\le r_\ell.
	\)
	At the same time,
	\[
	\|\mathbf W_\ell\|_F^2
	=
	\sum_{j=1}^{r_\ell}\sigma_j(\mathbf W_\ell)^2
	\ge
	r_\ell\sigma_{\min}(\mathbf W_\ell)^2
	\ge
	\dfrac{r_\ell}{\kappa_0^2}
	\sigma_{\max}(\mathbf W_\ell)^2.
	\]
	Combining the two displays gives
	\(
	\sigma_{\max}(\mathbf W_\ell)\le\kappa_0.
	\)
	The right-inverse bound follows from
	\[
	\left\|
	\mathbf W_\ell(\mathbf W_\ell^\top\mathbf W_\ell)^{-1}
	\right\|_2
	=
	\sigma_{\min}(\mathbf W_\ell)^{-1}.
	\]
	Thus the relative conditioning is dimension independent,
	while the absolute
	singular scale tracks the aspect ratio
	\(r_\ell/r_{\ell-1}\).
	\end{proof}

The next lemma makes the scale of the output weights explicit from normalization and the relative row-balance condition.
	
	\begin{lemma}[Explicit output-weight scale]
		\label{lem:output-weight-scale}
		Under the unified normalization and relative row-balance
		condition, let
		\[
		s:=r_{L-1},
		\qquad
		a_{\min}:=\min_{1\le i\le s}a_i,
		\qquad
		a_{\max}:=\max_{1\le i\le s}a_i.
		\]
		Then
		\[
		\dfrac{\mu}{s}
		\le
		a_{\min}
		\le
		\dfrac1s
		\le
		a_{\max}
		\le
		1-\dfrac{(s-1)\mu}{s}
		\le1.
		\]
	\end{lemma}
	
	\begin{proof}
	For \(\ell=L\), Assumption~\ref{ass:conditioning} reads
	\(
	a_i\ge\mu/s
	\)
	for every \(i\), which gives the first inequality. The
	middle two
	inequalities follow from \(\sum_i a_i=1\). Finally, after
	fixing an index
	attaining \(a_{\max}\), the remaining \(s-1\) entries
	each contribute at
	least \(\mu/s\), and therefore
	\(
	a_{\max}
	\le
	1-\dfrac{(s-1)\mu}{s}
	\le1.\)\end{proof}

\subsection{Basic Perturbation Tools}

The following lemma controls how errors in a full-column-rank matrix propagate to its pseudoinverse.

\begin{lemma}[Perturbation of the pseudoinverse]
		\label{lem:pseudoinverse-perturbation}
		Let $\mathbf{A},\mathbf{B}$ be full column rank. If
		\(\|\mathbf{A}-\mathbf{B}\|_2\le\dfrac12\sigma_{\min}(\mathbf{A}),
		\)
		then
		\(\|\mathbf{A}^\dagger-\mathbf{B}^\dagger\|_2\le\dfrac{6\|\mathbf{A}-\mathbf{B}\|_2}{\sigma_{\min}(\mathbf{A})^2}.
		\)
	\end{lemma}
	
	\begin{proof}
		By Weyl's inequality,
		$
		\sigma_{\min}(\mathbf{B})\ge
		\sigma_{\min}(\mathbf{A})-\|\mathbf{A}-\mathbf{B}\|_2
		\ge \dfrac12\sigma_{\min}(\mathbf{A}),
		$
		and hence $\|\mathbf{B}^\dagger\|_2\le
		2\|\mathbf{A}^\dagger\|_2$. For full-column-rank matrices,
\[
\begin{aligned}
\mathbf B^\dagger-\mathbf A^\dagger
=
-\mathbf B^\dagger(\mathbf B-\mathbf A)\mathbf A^\dagger
+
\mathbf B^\dagger\mathbf B^{\dagger\top}
(\mathbf B-\mathbf A)^\top
(\mathbf I-\mathbf A\mathbf A^\dagger).
\end{aligned}
\]
Since \(\mathbf I-\mathbf A\mathbf A^\dagger\) is an orthogonal
projector,
\[
\|\mathbf A^\dagger-\mathbf B^\dagger\|_2
\le
\left(
\|\mathbf A^\dagger\|_2\|\mathbf B^\dagger\|_2
+
\|\mathbf B^\dagger\|_2^2
\right)
\|\mathbf A-\mathbf B\|_2.
\]
		Substituting the preceding bound on
		$\|\mathbf{B}^\dagger\|_2$ and
		$\|\mathbf{A}^\dagger\|_2=\sigma_{\min}(\mathbf{A})^{-1}$
		proves the claim.
	\end{proof}

The next lemma converts relative errors in a collection of spanning vectors into an error bound for the recovered orthonormal basis of their column space.
	
	\begin{lemma}[Noisy recovery of an orthonormal basis]
		\label{lem:noisy-subspace-recovery}
		Let \(\mathbf{g}_1,\ldots,\mathbf{g}_m\in\mathbb{R}^n\)
		satisfy
		\[\dim\left(\operatorname{span}\{\mathbf{g}_1,\ldots,\mathbf{g}_m\}\right)=p,
		\]
		and define
		$\mathbf{G}=[\mathbf{g}_1,\ldots,\mathbf{g}_m]\in\mathbb{R}^{n\times
		m}.
		$
		Let \(\mathbf{B}\in\mathbb{R}^{n\times p}\) be an orthonormal basis of
		\(\operatorname{col}(\mathbf{G})\), so that
		\(\mathbf{B}^\top\mathbf{B}=\mathbf{I}_p\).
		Suppose that we observe
		\(\widehat{\mathbf{g}}_1,\ldots,\widehat{\mathbf{g}}_m\)
		satisfying
		\[
		\|\widehat{\mathbf{g}}_i-\mathbf{g}_i\|_2
		\le
		\eta\|\mathbf{g}_i\|_2,
		\qquad i=1,\ldots,m.
		\]
		Let
		$
		\operatorname{cond}(\mathbf{G})
		:=\dfrac{\sigma_1(\mathbf{G})}{\sigma_p(\mathbf{G})},
		\Delta_{\mathrm{sub}}
		:=
		\eta\operatorname{cond}(\mathbf{G})\sqrt p.
		$
		If \(\Delta_{\mathrm{sub}}<1\), then one can recover an
		orthonormal matrix
		\(\widehat{\mathbf{B}}\in\mathbb{R}^{n\times p}\), with
		\(\widehat{\mathbf{B}}^\top\widehat{\mathbf{B}}=\mathbf{I}_p\),
		such that
		$
		\|\widehat{\mathbf{B}}-\mathbf{B}\|_2
		\le
		\dfrac{\sqrt{2}\,\eta\operatorname{cond}(\mathbf{G})\sqrt
		p}
		{1-\eta\operatorname{cond}(\mathbf{G})\sqrt p}.
		$
		Here the basis is understood up to an orthogonal
		transformation within the
		recovered \(p\)-dimensional subspace.
	\end{lemma}
	
	\begin{proof}
		Define
		\[
		\widehat{\mathbf{G}}
		=[\widehat{\mathbf{g}}_1,\ldots,\widehat{\mathbf{g}}_m]
		=\mathbf{G}+\mathbf{E},
		\qquad
		\mathbf{E}
		=[\widehat{\mathbf{g}}_1-\mathbf{g}_1,\ldots,
		\widehat{\mathbf{g}}_m-\mathbf{g}_m].
		\]
		By the assumed relative error bound,
		\(
\|\mathbf{E}\|_2
		\le
\|\mathbf{E}\|_F
		\le
\eta\|\mathbf{G}\|_F.
		\)
		Since \(\operatorname{rank}(\mathbf{G})=p\),
		\(
\|\mathbf{G}\|_F
		\le
		\sqrt p\,\sigma_1(\mathbf{G}),
		\)
		and therefore
		\[
		\|\mathbf{E}\|_2
		\le
		\eta\sqrt p\,\sigma_1(\mathbf{G})
		=
		\Delta_{\mathrm{sub}}\,\sigma_p(\mathbf{G}).
		\]
		Compute the top \(p\) left singular vectors of
		\(\widehat{\mathbf{G}}\) and
		collect them as
		\[
		\widetilde{\mathbf{B}}\in\mathbb{R}^{n\times p},
		\qquad
		\widetilde{\mathbf{B}}^\top\widetilde{\mathbf{B}}=\mathbf{I}_p.
		\]
		Let
		\(
		\widehat{\mathcal
		S}=\operatorname{col}(\widetilde{\mathbf{B}})
		\)
		and
		\(
		\mathcal S=\operatorname{col}(\mathbf{B}).
		\)
		By the singular-subspace perturbation theorem,
		\[
		\|\mathbf{P}_{\widehat{\mathcal S}}-\mathbf{P}_{\mathcal
		S}\|_2
		\le
		\dfrac{\|\mathbf{E}\|_2}
		{\sigma_p(\mathbf{G})-\|\mathbf{E}\|_2}
		\le
		\dfrac{\Delta_{\mathrm{sub}}}{1-\Delta_{\mathrm{sub}}}.
		\]
		Finally, align the two orthonormal bases by an
		orthogonal Procrustes
		transformation. Namely, let
		\(
\mathbf{B}^\top\widetilde{\mathbf{B}}
=\mathbf{U}\boldsymbol{\Sigma}\mathbf{V}^\top\)
		be a singular value decomposition and set
		\(\mathbf{Q}=\mathbf{V}\mathbf{U}^\top,\widehat{\mathbf{B}}=\widetilde{\mathbf{B}}\mathbf{Q}.
		\)
		Then
		\[
		\|\widehat{\mathbf{B}}-\mathbf{B}\|_2
		=
		2\sin\!\left(\dfrac{\theta_{\max}}2\right)
		\le
		\sqrt2\sin\theta_{\max}
		=
		\sqrt2
		\|\mathbf{P}_{\widehat{\mathcal S}}-\mathbf{P}_{\mathcal
		S}\|_2,
		\]
		where \(\theta_{\max}\) is the largest principal angle.
		Substitution of
		the projector bound proves the claim.
	\end{proof}

The following corollary records the simpler linear error bound used when the perturbation is sufficiently small.
	
	\begin{corollary}
		\label{cor:noisy-subspace-recovery}
		Under the assumptions of
		Lemma~\ref{lem:noisy-subspace-recovery}, if
		\(\eta\operatorname{cond}(\mathbf{G})\sqrt p\le\dfrac12,
		\)
		then thereexists an orthonormal basis
		\(\widehat{\mathbf{B}}\) of the
		recovered subspace such that
		\[
		\|\widehat{\mathbf{B}}-\mathbf{B}\|_2
		\le
		2\sqrt2\,\eta\operatorname{cond}(\mathbf{G})\sqrt p
		=
		O\!\left(\eta\operatorname{cond}(\mathbf{G})\sqrt
		p\right).
		\]
	\end{corollary}

\subsection{Perturbation of Two-Matrix Diagonalization}

The final lemma provides the spectral perturbation guarantee used to recover hidden directions from noisy pairs of moment matrices.

\begin{lemma}[Noisy two-matrix diagonalization]
	\label{lem:noisy-two-matrix}
	Let \(n\ge2\) be an integer and let
	\[\mathbf{M}_1=\mathbf{C}^{-\top}\boldsymbol{\Lambda}\mathbf{C}^{-1},\mathbf{M}_2=\mathbf{C}\mathbf{D}\mathbf{C}^\top,\]
	where \(\mathbf{C}\in GL_n(\mathbb R)\) and
	\(\mathbf{D},\boldsymbol{\Lambda}\) are positive
	diagonal matrices. Write
	\(\mathbf{D}\boldsymbol{\Lambda}=\operatorname{diag}(\theta_1,\ldots,\theta_n),
	\)
	and suppose that \(\gamma\ge1\) satisfies the
	inverse-gap condition
	\[\dfrac{\min_{i\neq j}|\theta_i-\theta_j|}{\max_i|\theta_i|}\ge\gamma^{-1}.
	\]
	Let
	\(\mathbf{v}_i:=\dfrac{\mathbf{c}_i}{\|\mathbf{c}_i\|_2},\mathbf{V}=[\mathbf{v}_1,\ldots,\mathbf{v}_n],
	\)
	and define
	\(\rho:=\dfrac{\|\mathbf{M}_1\|_2\|\mathbf{M}_2\|_2}{\max_i|\theta_i|}.
	\)
	Suppose that
	\(\widehat{\mathbf{M}}_1,\widehat{\mathbf{M}}_2\in\mathbb R^{n\times n}\)
	satisfy
	\(\|\widehat{\mathbf{M}}_j-\mathbf{M}_j\|_2\le\eta\|\mathbf{M}_j\|_2,j=1,2.
	\)
	If
	\(
	\eta
	\le\dfrac{c}{\gamma\operatorname{cond}(\mathbf{V})^2\rho}
	\)
	for a sufficiently small absolute constant \(c>0\), then
	the eigenvectors
	of \(\widehat{\mathbf{M}}_2\widehat{\mathbf{M}}_1\) can
	be labeled so that,
	after choosing their signs,
	\[
	\max_{1\le i\le n}
	\|\widehat{\mathbf{v}}_i-\mathbf{v}_i\|_2
	\le
	C\gamma\operatorname{cond}(\mathbf{V})^2\rho\eta,
	\]
	where \(\widehat{\mathbf{v}}_i\) denotes the normalized
	recovered
	eigenvector and \(C>0\) is an absolute constant.
\end{lemma}

\begin{proof}
	Set
	\(\mathbf{A}=\mathbf{M}_2\mathbf{M}_1=\mathbf{C}(\mathbf{D}\boldsymbol{\Lambda})\mathbf{C}^{-1}.
	\)
	Thus the column directions of \(\mathbf C\) are
	precisely the
	eigendirections of \(\mathbf A\). Write
	\(\widehat{\mathbf M}_j=\mathbf M_j+\mathbf E_j,\|\mathbf E_j\|_2\le\eta\|\mathbf M_j\|_2.
	\)
	Then
	\(
	\widehat{\mathbf M}_2\widehat{\mathbf M}_1
	=\mathbf A+\boldsymbol{\Delta},
	\)
	where
	\[
	\boldsymbol{\Delta}
	=
	\mathbf M_2\mathbf E_1+
	\mathbf E_2\mathbf M_1+
	\mathbf E_2\mathbf E_1.
	\]
	Since
	\(\max_i|\theta_i|\le\|\mathbf A\|_2
	\le\|\mathbf M_2\|_2\|\mathbf M_1\|_2\),
	we have \(\rho\ge1\). Together with
	\(\gamma\ge1\) and \(\operatorname{cond}(\mathbf V)\ge1\),
	the assumed upper bound on \(\eta\) therefore implies
	\(\eta\le1\) after taking \(c\le1\). Hence
	\[
	\|\boldsymbol{\Delta}\|_2
	\le
	3\eta\|\mathbf M_1\|_2\|\mathbf M_2\|_2
	=
	3\rho\eta\max_i|\theta_i|.
	\]

	The matrix \(\mathbf A\) is diagonalizable with
	normalized eigenvector
	matrix \(\mathbf V\):
	\(\mathbf A=\mathbf V\boldsymbol{\Theta}\mathbf V^{-1},\boldsymbol{\Theta}=\operatorname{diag}(\theta_1,\ldots,\theta_n).
	\)
	Let
	\(g:=\min_{i\neq j}|\theta_i-\theta_j|.
	\)
	The inverse-gap hypothesis gives
	\(g\ge\gamma^{-1}\max_i|\theta_i|.
	\)
	Set
	\(
	\mathbf E:=\mathbf V^{-1}\boldsymbol{\Delta}\mathbf V.
	\)
	Then
	\(\|\mathbf E\|_2
	\le
	\operatorname{cond}(\mathbf V)\|\boldsymbol{\Delta}\|_2.
	\)
	By the assumed upper bound on \(\eta\), after decreasing
	\(c\) if
	necessary, \(\|\mathbf E\|_2\le\dfrac g4.\)
	Since
	\(\mathbf A+\boldsymbol{\Delta}
	=
	\mathbf V(\boldsymbol{\Theta}+\mathbf E)\mathbf V^{-1},
	\)
	it suffices to study the eigenvectors of
	\(\boldsymbol{\Theta}+\mathbf E\). By the Bauer--Fike
	theorem, every
	eigenvalue of \(\boldsymbol{\Theta}+t\mathbf E\), \(0\le
	t\le1\), lies in
	\[
	\bigcup_{i=1}^n
	\{z\in\mathbb C:|z-\theta_i|\le\|\mathbf E\|_2\}.
	\]
	Since \(\|\mathbf E\|_2\le g/4\) and the centers are
	separated by at least \(g\), these disks are pairwise disjoint.
	For every \(i\),
	let
	\(
	\Gamma_i:=\{z\in\mathbb C:|z-\theta_i|=g/3\},
	\)
	oriented counterclockwise.
	The circle \(\Gamma_i\) remains disjoint from the
	spectrum of
	\(\boldsymbol{\Theta}+t\mathbf E\) for every
	\(t\in[0,1]\). Hence the
	Riesz projection
	\[
	\boldsymbol{\Pi}_i(t)
	:=
	\dfrac{1}{2\pi i}
	\int_{\Gamma_i}
	\bigl(z\mathbf I-(\boldsymbol{\Theta}+t\mathbf
	E)\bigr)^{-1}\,dz
	\]
	is well defined and continuous in \(t\). Its rank equals
	its trace, which is continuous and integer valued, and
	therefore constant. At \(t=0\) it has rank one, so each
	disk contains
	exactly one eigenvalue \(\widehat\theta_i\), counted
	with algebraic
	multiplicity, and
	\(|\widehat\theta_i-\theta_i|\le\|\mathbf E\|_2.\)
	Because \(\boldsymbol{\Theta}+\mathbf E\) is real,
	nonreal eigenvalues
	occur in conjugate pairs. Each disk contains exactly one
	eigenvalue and is
	centered on the real axis; hence every
	\(\widehat\theta_i\) is real and the
	associated eigenvector may be chosen real.
	Let \(\mathbf y_i\in\mathbb R^n\setminus\{\mathbf0\}\)
	be an eigenvector associated with
	\(\widehat\theta_i\), and let
	\(\mathbf Q_i:=\mathbf I-\mathbf e_i\mathbf e_i^\top.\)
	From the
	eigenvalue equation,
	\[
	\begin{aligned}
	(\theta_j-\widehat\theta_i)(\mathbf y_i)_j
	&=
	-(\mathbf E\mathbf y_i)_j,\qquad j\neq i,\\
	\min_{j\neq i}|\theta_j-\widehat\theta_i|
	&\ge
	g-\|\mathbf E\|_2
	\ge
	\dfrac{3g}{4}.
	\end{aligned}
	\]
	Consequently, summing the squared coordinate bounds gives
	\begin{equation}
	\|\mathbf Q_i\mathbf y_i\|_2
	\le
	\dfrac{4}{3g}\|\mathbf Q_i\mathbf E\mathbf y_i\|_2
	\le
	\dfrac{2\|\mathbf E\|_2}{g}\|\mathbf y_i\|_2.
	\end{equation}

	The corresponding eigenvector of \(\mathbf
	A+\boldsymbol{\Delta}\) is
	proportional to \(\mathbf V\mathbf y_i\). Since
	\(\mathbf v_i=\mathbf V\mathbf e_i\) is a unit vector,
	we have
	\[
	\begin{aligned}
	\sin\angle(\widehat{\mathbf v}_i,\mathbf v_i)
	\le
	\dfrac{\|\mathbf V\mathbf Q_i\mathbf y_i\|_2}
	{\|\mathbf V\mathbf y_i\|_2}
	\le
	2\dfrac{\operatorname{cond}(\mathbf V)^2
	\|\boldsymbol{\Delta}\|_2}{g}
	\le
	C\gamma\operatorname{cond}(\mathbf V)^2\rho\eta.
	\end{aligned}
	\]
	Finally, after choosing the sign of each normalized
	recovered eigenvector,
	\[
	\|\widehat{\mathbf v}_i-\mathbf v_i\|_2
	\le
	\sqrt2\sin\angle(\widehat{\mathbf v}_i,\mathbf v_i),
	\]
	which proves the claim.
\end{proof}

\section{Convex-body Sampling and Moment Estimation}

	\label{sec:appendix-sampling-moments}
 This section develops the convex-body sampling and moment-estimation mechanism underlying ASPIRE, including a noisy Hit-and-Run algorithm that remains provably accurate under multiplicative errors in the membership oracle.

	\paragraph{Notation.}
	Throughout the remainder of the appendix, whenever an
	orthonormal matrix
	\(
	\widehat{\mathbf{B}}_\ell\in\mathbb R^{r_{\ell-1}\times
	r_\ell}
	\)
	is under consideration, define
	\(\mathbf{C}_\ell^\sharp
	:=
\widehat{\mathbf{B}}_\ell^\top \mathbf{W}_\ell,
G_\ell(\mathbf{z})
:=F_\ell(\widehat{\mathbf{B}}_\ell \mathbf{z})
	=
h^{(\ell)}\!\left((\mathbf{C}_\ell^\sharp)^\top
\mathbf{z}\right),
	\)
	and
	\(K_\ell^\sharp
	:=
	\left\{\mathbf{z}\in\mathbb R^{r_\ell}:G_\ell(\mathbf{z})\le1\right\}.
	\)
	In the first-layer stage, where the column space is
	recovered exactly, we use
	the same notation with
	\(\widehat{\mathbf{B}}_1:=\mathbf{B}_1\).
	
\subsection{Geometry of the Reduced Sampling Body}

We first control the conditioning and Euclidean geometry of the reduced coordinates in which ASPIRE performs convex-body sampling.

The following lemma shows that an approximate column-space basis preserves the conditioning of the reduced weight matrix.

	\begin{lemma}[Conditioning of the reduced weight matrix]
		\label{lem:reduced-weight-conditioning}
		Let \(\mathbf{B}_\ell\) and
		\(\widehat{\mathbf{B}}_\ell\) have orthonormal columns,
		suppose
		\(\operatorname{col}(\mathbf{B}_\ell)=\operatorname{col}(\mathbf{W}_\ell)\),
		and assume,
		after orthogonal alignment, that
		\(\|\widehat{\mathbf{B}}_\ell-\mathbf{B}_\ell\|_2\le\varepsilon_{B,\ell}<1.
		\)
		For
		\(\mathbf{C}_\ell^\sharp:=\widehat{\mathbf{B}}_\ell^\top\mathbf{W}_\ell,
		\)
		one has
		\[
		(1-\varepsilon_{B,\ell})\sigma_{\min}(\mathbf{W}_\ell)
		\le
		\sigma_{\min}(\mathbf{C}_\ell^\sharp)
		\le
		\sigma_{\max}(\mathbf{C}_\ell^\sharp)
		\le
		(1+\varepsilon_{B,\ell})\sigma_{\max}(\mathbf{W}_\ell).
		\]
		Consequently,
		\(\operatorname{cond}(\mathbf{C}_\ell^\sharp)
		\le
\dfrac{1+\varepsilon_{B,\ell}}{1-\varepsilon_{B,\ell}}\kappa_0.
		\)
		In particular, if \(\varepsilon_{B,\ell}\le1/2\), then
		\(\operatorname{cond}(\mathbf{C}_\ell^\sharp)\le3\kappa_0\).
	\end{lemma}
	
	\begin{proof}
	Write
	\(\mathbf{W}_\ell=\mathbf{B}_\ell\mathbf{C}_\ell,
\mathbf{C}_\ell:=\mathbf{B}_\ell^\top \mathbf{W}_\ell.
	\)
	Because \(\mathbf{B}_\ell\) has orthonormal columns,
	\(\mathbf{C}_\ell\) has exactly the
	nonzero singular values of \(\mathbf{W}_\ell\). Define
	\(
	\mathbf{Q}_\ell:=\widehat{\mathbf{B}}_\ell^\top
	\mathbf{B}_\ell.
	\)
	Then
	\(
	\mathbf{C}_\ell^\sharp
	=
	\mathbf{Q}_\ell \mathbf{C}_\ell.
	\)
	Moreover,
	\[
	\begin{aligned}
	\|\mathbf{Q}_\ell-\mathbf{I}\|_2
	=
	\|\widehat{\mathbf{B}}_\ell^\top
	\mathbf{B}_\ell-\mathbf{B}_\ell^\top \mathbf{B}_\ell\|_2
	\le
	\|\widehat{\mathbf{B}}_\ell-\mathbf{B}_\ell\|_2 
	\le
	\varepsilon_{B,\ell}.
	\end{aligned}
	\]
	Weyl's inequality therefore gives
	\[
	1-\varepsilon_{B,\ell}
	\le
	\sigma_{\min}(\mathbf{Q}_\ell)
	\le
	\sigma_{\max}(\mathbf{Q}_\ell)
	\le
	1+\varepsilon_{B,\ell}.
	\]
	The singular-value product inequalities now imply
	\[
	\begin{aligned}
	\sigma_{\min}(\mathbf{C}_\ell^\sharp)
	\ge
	\sigma_{\min}(\mathbf{Q}_\ell)\sigma_{\min}(\mathbf{C}_\ell)
	\ge
	(1-\varepsilon_{B,\ell})\sigma_{\min}(\mathbf{W}_\ell),
	\end{aligned}
	\]
	and
	\[
	\begin{aligned}
	\sigma_{\max}(\mathbf{C}_\ell^\sharp)
	\le
	\sigma_{\max}(\mathbf{Q}_\ell)\sigma_{\max}(\mathbf{C}_\ell)
	\le
	(1+\varepsilon_{B,\ell})\sigma_{\max}(\mathbf{W}_\ell).
	\end{aligned}
	\]
	Taking the ratio gives
	\[
	\operatorname{cond}(\mathbf{C}_\ell^\sharp)
	\le
	\dfrac{1+\varepsilon_{B,\ell}}
	{1-\varepsilon_{B,\ell}}
	\operatorname{cond}(\mathbf{W}_\ell)
	\le
	\dfrac{1+\varepsilon_{B,\ell}}
	{1-\varepsilon_{B,\ell}}
	\kappa_0.
	\]
	When \(\varepsilon_{B,\ell}\le1/2\), the ratio in front
	of \(\kappa_0\)
	is at most three.
\end{proof}
	
The next lemma converts this conditioning control into explicit inner and outer Euclidean radii for the reduced sampling body.

	\begin{lemma}[Euclidean radii of the sampling body]
		\label{lem:sampling-body-radii}
		Fix \(1\le \ell\le L-2\), and let
		\[
		q_\ell:=k^{L-\ell},
		\qquad
		S_\ell:=\sum_{j=1}^{L-\ell-1}k^j
		=
		\dfrac{q_\ell-k}{k-1}.
		\]
		Let \(\widehat{\mathbf B}_\ell\) satisfy the hypotheses
		of
		Lemma~\ref{lem:reduced-weight-conditioning} with
		\(\varepsilon_{B,\ell}\le1/2\). Then
		\(\dfrac{1}
		{(1+\varepsilon_{B,\ell})\sigma_{\max}(\mathbf W_\ell)}
		B_2^{r_\ell}
		\subseteq
		K_\ell^\sharp
		\)
		and
		\(K_\ell^\sharp
		\subseteq
		\dfrac{
		\mu^{-(S_\ell+1)/q_\ell}\sqrt{r_\ell}
		}{
		(1-\varepsilon_{B,\ell})\sigma_{\min}(\mathbf W_\ell)
		}
		B_2^{r_\ell}.
		\)
		Consequently, one may take radii satisfying
		\[
		r_{\mathrm{in},\ell}
		\ge
		\dfrac{1}
		{(1+\varepsilon_{B,\ell})\sigma_{\max}(\mathbf W_\ell)},
		R_{\mathrm{out},\ell}
		\le
		\dfrac{
		\mu^{-(S_\ell+1)/q_\ell}\sqrt{r_\ell}
		}{
		(1-\varepsilon_{B,\ell})\sigma_{\min}(\mathbf W_\ell)
		}.
		\]
		In particular,
		\begin{align*}
		\omega_\ell
		:=
		\dfrac{R_{\mathrm{out},\ell}}{r_{\mathrm{in},\ell}}
		\le
		3\kappa_0
		\mu^{-(S_\ell+1)/q_\ell}
		\sqrt{r_\ell}
		\le
		C_{k,L}\kappa^4\sqrt{r_\ell}
		\le
		C_{k,L}\kappa^4\sqrt r.
		\end{align*}
	\end{lemma}
	
	\begin{proof}
	For \(\mathbf u\in\mathbb R^{r_\ell}\), define the
	nonnegative forward
	quantities
	\[
	\mathbf s_\ell:=\mathbf u^{\odot k},
	\qquad
	\mathbf s_j
	:=
	\left(\mathbf W_j^\top\mathbf s_{j-1}\right)^{\odot k},
	\quad
	\ell+1\le j\le L-1.
	\]
	Then
	\(h^{(\ell)}(\mathbf u)
	=
	\mathbf a^\top\mathbf s_{L-1}.
	\)We first derive an upper bound that uses only the
	column-stochastic
	normalization. If \(\mathbf v\ge\mathbf0\), every
	coordinate of
	\(\mathbf W_j^\top\mathbf v\) is a convex combination of
	the coordinates
	of \(\mathbf v\), and therefore
	\(\|\mathbf W_j^\top\mathbf v\|_\infty
	\le
	\|\mathbf v\|_\infty.
	\)
	It follows inductively that
	\(\|\mathbf s_{L-1}\|_\infty
	\le
	\|\mathbf u\|_\infty^{q_\ell}.
	\)
	Since \(\mathbf a\ge\mathbf0\) and \(\mathbf1^\top\mathbf
	a=1\),
	\(h^{(\ell)}(\mathbf u)
	=
	\mathbf a^\top\mathbf s_{L-1}
	\le
	\|\mathbf s_{L-1}\|_\infty
	\le
	\|\mathbf u\|_\infty^{q_\ell}
	\le
	\|\mathbf u\|_2^{q_\ell}.
	\)
	For the reverse inequality, define the average mass at
	layer \(j\) by
	\[
	\overline s_j
	:=
	\dfrac{\|\mathbf s_j\|_1}{r_j},
	\qquad
	\ell\le j\le L-1.
	\]
	At the initial step, the comparison between \(\ell_k\)-
	and
	\(\ell_2\)-norms gives
	\(\overline s_\ell
	=
	\dfrac{\|\mathbf u\|_k^k}{r_\ell}
	\ge
	\left(\dfrac{\|\mathbf u\|_2}{\sqrt{r_\ell}}\right)^k.
	\)
	For every \(j>\ell\), the relative row-balance condition
	and
	nonnegativity yield
	\[
	\begin{aligned}
	\dfrac1{r_j}
	\|\mathbf W_j^\top\mathbf s_{j-1}\|_1
	=
	\dfrac1{r_j}
	(\mathbf W_j\mathbf1_{r_j})^\top\mathbf s_{j-1}
	\ge
	\dfrac{\mu}{r_{j-1}}\|\mathbf s_{j-1}\|_1
	=
	\mu\overline s_{j-1}.
	\end{aligned}
	\]
	Jensen's inequality then gives
	\[
	\begin{aligned}
	\overline s_j
	=
	\dfrac1{r_j}
	\|\mathbf W_j^\top\mathbf s_{j-1}\|_k^k\ge
	\left(
	\dfrac1{r_j}
	\|\mathbf W_j^\top\mathbf s_{j-1}\|_1
	\right)^k
	\ge
	\mu^k\overline s_{j-1}^{\,k}.
	\end{aligned}
	\]
	Iterating through the remaining hidden layers gives
	\(\overline s_{L-1}
	\ge
	\mu^{S_\ell}
	\left(\dfrac{\|\mathbf
	u\|_2}{\sqrt{r_\ell}}\right)^{q_\ell}.
	\)
	The row-balance condition for \(\mathbf W_L=\mathbf a\)
	gives
	\(a_i\ge\mu/r_{L-1}\), and hence
	\[
	\begin{aligned}
	h^{(\ell)}(\mathbf u)
	=
	\mathbf a^\top\mathbf s_{L-1}
	\ge
	\dfrac{\mu}{r_{L-1}}\|\mathbf s_{L-1}\|_1
	=
	\mu\overline s_{L-1}
	\ge
	\mu^{S_\ell+1}
	\left(\dfrac{\|\mathbf
	u\|_2}{\sqrt{r_\ell}}\right)^{q_\ell}.
	\end{aligned}
	\]

	Let
	\(
	H_\ell
	:=
	\{\mathbf u\in\mathbb R^{r_\ell}:h^{(\ell)}(\mathbf
	u)\le1\}.
	\)
	The two polynomial bounds imply
	\(B_2^{r_\ell}
	\subseteq
	H_\ell
	\subseteq
	\mu^{-(S_\ell+1)/q_\ell}\sqrt{r_\ell}\,B_2^{r_\ell}.
	\)
	By definition,
	\(
	K_\ell^\sharp=(\mathbf C_\ell^\sharp)^{-\top}H_\ell.
	\)
	For an invertible matrix \(\mathbf C\), the image
	\((\mathbf C^{-\top})(\rho B_2)\) contains
	\((\rho/\sigma_{\max}(\mathbf C))B_2\), whereas the image
	of
	\(R B_2\) is contained in
	\((R/\sigma_{\min}(\mathbf C))B_2\). Applying
	Lemma~\ref{lem:reduced-weight-conditioning} gives the two
	radius bounds in
	the statement.

	Taking their ratio yields
	\[
	\omega_\ell
	\le
	\dfrac{1+\varepsilon_{B,\ell}}{1-\varepsilon_{B,\ell}}
	\operatorname{cond}(\mathbf W_\ell)
	\mu^{-(S_\ell+1)/q_\ell}\sqrt{r_\ell}.
	\]
	Since \(\varepsilon_{B,\ell}\le1/2\),
	\(\operatorname{cond}(\mathbf W_\ell)\le\kappa_0\), and
	\(
	\dfrac{S_\ell+1}{q_\ell}
	=
	\dfrac{q_\ell-1}{(k-1)q_\ell}
	<
	\dfrac1{k-1},
	\)
	the dimension-free factor is at most \(C_{k,L}\kappa^4\).
	Finally,
	\(r_\ell\le r\), proving the stated simplification.
	\end{proof}
	
\subsection{Hit-and-Run with a Multiplicative Oracle}

We next develop the geometric and perturbative ingredients needed to sample from the true convex body when only a multiplicatively inaccurate oracle is available.

The following Markov-type bound controls the curvature of a polynomial on a convex body from its sup norm and inradius.

\begin{lemma}[Markov bound for second derivatives]
		\label{lem:second-derivative-markov}
		Let \(K\subset\mathbb R^n\) be a convex body satisfying
		\[
		\rho B_2^n\subseteq K
		\]
		for some \(\rho>0\), and let \(p:\mathbb R^n\to\mathbb
		R\) be a polynomial
		of degree at most \(q\ge2\). Then
		\[
		\sup_{\mathbf{x}\in K}
		\|\nabla^2p(\mathbf{x})\|_{\mathrm{op}}
		\le
		\dfrac{4q^2(q-1)^2}{\rho^2}
		\|p\|_{L^\infty(K)}
		\le
		\dfrac{4q^4}{\rho^2}
		\|p\|_{L^\infty(K)}.
		\]
	\end{lemma}
	
	\begin{proof}
		Let \(w(K)\) denote the minimum width of \(K\). Since
		\(
		\rho B_2^n\subseteq K,
		\)
		we have
		\(
		w(K)\ge2\rho.
		\)
		By the multivariate Markov inequality of
		\citet{wilhelmsen1974markov}, every polynomial \(g\) of
		degree at most
		\(m\) satisfies
		\[
		\|\nabla g\|_{L^\infty(K)}
		\le
		\dfrac{4m^2}{w(K)}
		\|g\|_{L^\infty(K)}
		\le
		\dfrac{2m^2}{\rho}
		\|g\|_{L^\infty(K)}.
		\]
		Applying this first to \(p\) gives
		\[
		\|\nabla p\|_{L^\infty(K)}
		\le
		\dfrac{2q^2}{\rho}
		\|p\|_{L^\infty(K)}.
		\]
		
		Now fix \(\mathbf{u}\in S^{n-1}\). The directional
		derivative
		\(
		D_{\mathbf{u}}p:=\mathbf{u}^\top\nabla p
		\)
		has degree at most \(q-1\), and therefore another
		application of the
		same inequality yields
		\[
		\begin{aligned}
			\|\nabla(D_{\mathbf{u}}p)\|_{L^\infty(K)}
			\le
			\dfrac{2(q-1)^2}{\rho}
\|D_{\mathbf{u}}p\|_{L^\infty(K)}
			\le
			\dfrac{2(q-1)^2}{\rho}
			\|\nabla p\|_{L^\infty(K)}\le
			\dfrac{4q^2(q-1)^2}{\rho^2}
			\|p\|_{L^\infty(K)}.
		\end{aligned}
		\]
		Since
		\(
		\nabla(D_{\mathbf{u}}p)(\mathbf{x})=\nabla^2p(\mathbf{x})\mathbf{u},
		\)
		taking the supremum over \(\mathbf{u}\in S^{n-1}\) and
		\(\mathbf{x}\in K\) proves the
		claim.
	\end{proof}
	
Using this curvature control, the next lemma shows that noisy Hit-and-Run remains close to uniform sampling from the true body under sufficiently small multiplicative oracle error.

	\begin{lemma}[Hit-and-Run with multiplicative oracle noise]
		\label{lem:noisy-hit-and-run}
		Let
		\(
		F:\mathbb R^n\to\mathbb R_{\ge 0}
		\)
		be a positive definite convex homogeneous polynomial of
		degree \(q\ge2\), and define
		\(
		K:=\{\mathbf{x}\in\mathbb R^n:F(\mathbf{x})\le1\}.
		\)
		Assume
		\[
		r_{\mathrm{in}}B_2^n
		\subseteq K
		\subseteq
		R_{\mathrm{out}}B_2^n,
		\qquad
		\omega:=\dfrac{R_{\mathrm{out}}}{r_{\mathrm{in}}},
		\]
		and suppose that the available oracle
\(\widehat F:\mathbb R^n\to\mathbb R\) is Borel measurable and
satisfies
		\[
		(1-\xi)F(\mathbf{x})
		\le
		\widehat F(\mathbf{x})
		\le
		(1+\xi)F(\mathbf{x}),
		\qquad
		\mathbf{x}\in\mathbb R^n,
		\]
		for some \(0<\xi\le1/2\).
		
		Fix \(\delta\in(0,1]\). Then there exist an absolute
		constant
		\(C>0\) and a constant \(c_\delta>0\), depending only on
		\(\delta\),
		such that, for every \(\varepsilon\in(0,1/2)\), if
		\[
		\xi
		\le
		\dfrac{
			c_\delta\,\varepsilon^{1+\delta}
		}{
			q^2
			n^{4+3\delta}
			\omega^{4+2\delta}
			\left(
			\log\dfrac{2\omega}{\varepsilon}
			\right)^{1+\delta}
		},
		\]
		then the noisy Hit-and-Run chain constructed in the proof,
		initialized at
		\(
		\widehat{\mathbf{X}}_0=\mathbf{0},
		\)
		satisfies
		\[
		d_{\mathrm{TV}}
		\left(
		\mathcal L(\widehat{\mathbf{X}}_T),
		\operatorname{Unif}(K)
		\right)
		\le\varepsilon
		\]
		for
		\(
		T
		=
		\left\lceil
		Cn^3\omega^2
		\log\dfrac{2\omega}{\varepsilon}
		\right\rceil.
		\)
  With
\[
\tau
:=
1-
\left(
\dfrac{1-\xi}{1+\xi}
\right)^{1/q},\qquad 
\eta
:=
\dfrac{\varepsilon}{4T},
\]
the construction uses
\(
O\!\left(
T\log\dfrac{\omega}{\eta\tau}
\right)
\)
evaluations of \(\widehat F\).
	\end{lemma}
	
	\begin{proof}
		Set
		\(
		\rho
		:=
		\left(
		\dfrac{1-\xi}{1+\xi}
		\right)^{1/q}\), then\(
		\tau=1-\rho.
		\)
		By homogeneity,
		\(
		\rho K
		\subseteq
		\{\mathbf{x}:\widehat F(\mathbf{x})\le1-\xi\}
		\subseteq
		K.
		\)
		Indeed, if \(\mathbf{x}\in \rho K\), then
		\(
		F(\mathbf{x})\le \rho^q=\dfrac{1-\xi}{1+\xi},
		\)
		and hence
		\(
		\widehat F(\mathbf{x})
		\le
		(1+\xi)F(\mathbf{x})
		\le1-\xi.
		\)
		Conversely,
		\[
		\widehat F(\mathbf{x})\le1-\xi
		\Longrightarrow
		(1-\xi)F(\mathbf{x})\le1-\xi
		\Longrightarrow
		F(\mathbf{x})\le1.
		\]
		Moreover,
		\begin{equation}
			q\tau
			\le
			\log\dfrac{1+\xi}{1-\xi}
			\le
			3\xi.
			\label{eq:tau-xi}
		\end{equation}
		
		Let \(P\) denote the exact Hit-and-Run transition kernel
		on \(K\), and
		let
		\(
		\mathbf{X}_0=\mathbf{0},
		\mathbf{X}_{t+1}\sim P(\mathbf{X}_t,\cdot).
		\)
		Define the boundary depth
		\(
		D(\mathbf{x}):=1-F(\mathbf{x}).
		\)
		Since \(\partial K\) has Lebesgue measure zero,
		\(D(\mathbf{X}_t)>0\) almost surely. We first establish
		a uniform
		negative-moment bound for \(D(\mathbf{X}_t)\).
		Set
		\(
		\alpha:=\dfrac1{1+\delta},
		s:=1-\alpha=\dfrac{\delta}{1+\delta},
		\)
		and define
		\(
		\vartheta_\delta
		:=
		\left(\dfrac{s}{2}\right)^{1/s}.
		\)
		Let \(C_0=\sqrt{\dfrac{2}{\pi}}\). Then for every
		unit vector \(\mathbf{v}\in\mathbb R^n\), every
		\(a\in(0,1)\), and
		\(\boldsymbol{\theta}\sim\operatorname{Unif}(S^{n-1})\),
		\begin{equation}
			\mathbb P
			\left(
			|\langle \mathbf{v},\boldsymbol{\theta}\rangle|
			\le
			\dfrac{a}{\sqrt n}
			\right)
			\le
			C_0a.
			\label{eq:spherical-small-ball}
		\end{equation}
		This is the standard spherical small-ball estimate.
		
		Set
		\(
		a_\delta
		:=
		\dfrac{\vartheta_\delta}{2C_0},
		\Delta_\delta
		:=
		\dfrac{a_\delta^2}
		{32\,n\omega^2q^2},
		d_0
		:=
		\dfrac{\vartheta_\delta}{2}\Delta_\delta.
		\)
		Thus
		\begin{equation}
			d_0
			=
			\dfrac{
				c_0\vartheta_\delta^3
			}{
				n\omega^2q^2
			},
			\qquad
			c_0:=\dfrac{1}{256C_0^2},
			\label{eq:d0-general}
		\end{equation}
		where \(c_0>0\) is an absolute constant.
		
		We claim that whenever
		\(
		\mathbf{x}\in\operatorname{int}(K) \)
		and
		\(D(\mathbf{x})\le d_0,
		\)
		\begin{equation}
			\mathbb P
			\left(
			D(\mathbf{X}_{t+1})\ge d_0
			\,\middle|\,
			\mathbf{X}_t=\mathbf{x}
			\right)
			\ge
			1-\vartheta_\delta.
			\label{eq:escape-general}
		\end{equation}
		Since \(d_0\le1/2\),
		\(
		F(\mathbf{x})\ge\dfrac12.
		\)
		Euler's identity and \(\|\mathbf{x}\|_2\le
		R_{\mathrm{out}}\) give
		\[
		\mathbf{x}^\top\nabla F(\mathbf{x})=qF(\mathbf{x}),
		\qquad
		\|\nabla F(\mathbf{x})\|_2
		\ge
		\dfrac{q}{2R_{\mathrm{out}}}.
		\]
		Since \(F\) is positive definite and
		\(
		K=\{\mathbf{z}:F(\mathbf{z})\le1\},
		\)
		we have
		\(
		\|F\|_{L^\infty(K)}=1.
		\)
		Hence Lemma~\ref{lem:second-derivative-markov} yields
		\[
		\mathcal H
		:=
		\sup_{\mathbf{z}\in K}
		\|\nabla^2F(\mathbf{z})\|_{\mathrm{op}}
		\le
		\dfrac{4q^4}{r_{\mathrm{in}}^2}.
		\]
		Let
		\(
		\mathbf{v}
		=
		\dfrac{\nabla F(\mathbf{x})}
		{\|\nabla F(\mathbf{x})\|_2}.
		\)
		By \eqref{eq:spherical-small-ball}, with probability at
		least
		\(1-\vartheta_\delta/2\),
		\(
		|\langle \mathbf{v},\boldsymbol{\theta}\rangle|
		\ge
		\dfrac{a_\delta}{\sqrt n}.
		\)
		Choose the orientation of the sampled line for which
		\(
		A:=-\nabla F(\mathbf{x})^\top\boldsymbol{\theta}>0.
		\)
		On the above event,
		\(
		A
		\ge
		\dfrac{qa_\delta}
		{2R_{\mathrm{out}}\sqrt n}.
		\)
		Set
		\(
		f(t):=F(\mathbf{x}+t\boldsymbol{\theta}),
		t_*:=\dfrac{A}{\mathcal H},
		\)
		and let
		\(
		b
		:=
		\sup\left\{
			t\ge0:
			\mathbf{x}+s\boldsymbol{\theta}\in K
			\text{ for every }s\in[0,t]
		\right\}
		\)
		be the distance from \(\mathbf{x}\) to the boundary of
		\(K\) along the chosen
		orientation of the line. Since \(K\) is compact and
		\(\mathbf{x}\in\operatorname{int}(K)\), we have
		\[
		0<b<\infty,
		\qquad
		\mathbf{x}+b\boldsymbol{\theta}\in\partial K,
		\qquad
		f(b)=1.
		\]
		
		For every
		\(
		0\le t\le \min\{b,t_*\},
		\)
		the segment
		\(
		\{\mathbf{x}+s\boldsymbol{\theta}:0\le s\le t\}
		\)
		is contained in \(K\). Hence
		\[
		\begin{aligned}
		f'(t)
		=
		f'(0)
		+
		\int_0^t
		\boldsymbol{\theta}^\top
		\nabla^2F(\mathbf{x}+s\boldsymbol{\theta})
		\boldsymbol{\theta}\,ds
		\le
		-A+\mathcal H t.
		\end{aligned}
		\]
		Suppose, for contradiction, that \(b<t_*\). Integrating
		the preceding
		inequality over \([0,b]\) gives
		\[
		\begin{aligned}
		f(b)
		\le
		f(0)-Ab+\dfrac{\mathcal H}{2}b^2
		<
		f(0)
		\le1,
		\end{aligned}
		\]
		because \(b<A/\mathcal H\). This contradicts \(f(b)=1\).
		Therefore
		\(
		b\ge t_*.
		\)
		In particular, \(\mathbf{x}+t_*\boldsymbol{\theta}\in
		K\), and another integration yields
		\(
		f(t_*)
		\le
		f(0)-At_*+\dfrac{\mathcal H}{2}t_*^2
		=
		f(0)-\dfrac{A^2}{2\mathcal H}.
		\)
		Consequently,
		\(
		\dfrac{A^2}{2\mathcal H}
		\ge
		\dfrac{a_\delta^2}
		{32\,n\omega^2q^2}
		=
		\Delta_\delta,
		\)
		so the chord contains a point \(\mathbf{z}\) satisfying
		\(
		F(\mathbf{z})\le1-\Delta_\delta.
		\)
		Let \(\mathbf{z}_-\) and \(\mathbf{z}_+\) be the two
		endpoints of this chord. Since
		\(
		F(\mathbf{z}_-)=F(\mathbf{z}_+)=1,
		\)
		convexity implies, for either sign and \(t\in[0,1]\),
		\[
		F((1-t)\mathbf{z}+t\mathbf{z}_\pm)
		\le
		(1-t)F(\mathbf{z})+tF(\mathbf{z}_\pm)
		\le
		1-(1-t)\Delta_\delta.
		\]
		Hence
		\(
		F((1-t)\mathbf{z}+t\mathbf{z}_\pm)\le1-d_0
		\)
		whenever
		\(
		t
		\le
		1-\dfrac{d_0}{\Delta_\delta}.
		\)
		Therefore the proportion of the entire chord on which
		\(
		D\ge d_0
		\)
		is at least
		\(
		1-\dfrac{d_0}{\Delta_\delta}
		=
		1-\dfrac{\vartheta_\delta}{2}.
		\)
		Combining this with the directional event yields
		\[
		\mathbb P
		\left(
		D(\mathbf{X}_{t+1})\ge d_0
		\,\middle|\,
		\mathbf{X}_t=\mathbf{x}
		\right)
		\ge
		\left(1-\dfrac{\vartheta_\delta}{2}\right)^2
		\ge
		1-\vartheta_\delta,
		\]
		proving \eqref{eq:escape-general}.
		
		For the exact Hit-and-Run chain, we first record a
		deterministic
		homothetic chord comparison. Let
		\(
		0\le \sigma_0<\lambda\le1,
		\mathbf{x}\in\sigma_0K,
		\)
		and let \(I_K\) and \(I_{\lambda K}\) be the
		intersections of a fixed
		line through \(\mathbf{x}\) with \(K\) and \(\lambda
		K\), respectively. Then
		\begin{equation}
			\dfrac{|I_{\lambda K}|}{|I_K|}
			\ge
			\dfrac{\lambda-\sigma_0}{1-\sigma_0}.
			\label{eq:homothetic-chord-comparison}
		\end{equation}
		Indeed, write \(\mathbf{x}=\sigma_0\mathbf{y}\) for some
		\(\mathbf{y}\in K\), and set
		\(
		\beta:=\dfrac{\lambda-\sigma_0}{1-\sigma_0}.
		\)
		For every point \(\mathbf{z}\in I_K\),
		\[
		\begin{aligned}
		\mathbf{x}+\beta(\mathbf{z}-\mathbf{x})
		=(1-\beta)\sigma_0\mathbf{y}+\beta \mathbf{z}
		=\lambda
		\left(
		\dfrac{(1-\beta)\sigma_0}{\lambda}\mathbf{y}
		+
		\dfrac{\beta}{\lambda}\mathbf{z}
		\right).
		\end{aligned}
		\]
		The two coefficients in parentheses are nonnegative and
		sum to one, so
		the point in parentheses belongs to \(K\). Therefore
		\(
		\mathbf{x}+\beta(I_K-\mathbf{x})\subseteq I_{\lambda K},
		\)
		which proves \eqref{eq:homothetic-chord-comparison}.

		Now fix \(\mathbf{x}\in\operatorname{int}(K)\) and write
		\(
		 d:=D(\mathbf{x})\in(0,1].
		\)
		For \(0<u< d\), set
		\(
		s_{ d}:=(1- d)^{1/q},
		\lambda_u:=(1-u)^{1/q}.
		\)
		Then \(\mathbf{x}\in s_{ d}K\) and
		\(s_{ d}<\lambda_u\). Conditional on any
		sampled
		direction, the event \(D(\mathbf{X}_{t+1})\le u\) is
		contained, up to chord
		endpoints of one-dimensional measure zero, in the part
		of \(I_K\) outside
		\(I_{\lambda_uK}\). Thus
		\eqref{eq:homothetic-chord-comparison}, followed by
		averaging over the
		direction, gives
		\[
		P\bigl(\mathbf{x},\{\mathbf{y}\in K:D(\mathbf{y})\le
		u\}\bigr)
		\le
		\dfrac{1-\lambda_u}{1-s_{ d}}.
		\]
		Since
		\(
		v\mapsto1-(1-v)^{1/q}
		\)
		is convex and vanishes at \(0\),
		\begin{equation}
			P\bigl(\mathbf{x},\{\mathbf{y}\in K:D(\mathbf{y})\le
			u\}\bigr)
			\le
			\dfrac{u}{ d},
			\qquad
			0<u< d.
			\label{eq:shell-general}
		\end{equation}

		Define the one-step negative moment at \(\mathbf{x}\) by
		\(
		\mathcal M(\mathbf{x})
		:=
		\int_K
		D(\mathbf{y})^{-\alpha}\,P(\mathbf{x},d\mathbf{y}).
		\)
		If \( d\le d_0\), then
		\eqref{eq:escape-general} and
		\eqref{eq:shell-general} imply
		\[
		P\bigl(\mathbf{x},\{\mathbf{y}:D(\mathbf{y})\le
		u\}\bigr)
		\le
		\min\left\{
		\vartheta_\delta,\dfrac{u}{ d}
		\right\},
		\qquad
		0<u\le d_0.
		\]
		Here, when \(u\ge  d\), the second bound
		follows from the trivial estimate
		\(P(\mathbf{x},\cdot)\le1\le u/ d\).
		For \(0<Y\le1\),
		\(
		\mathbb E[Y^{-\alpha}]
		=
		1+
		\alpha
		\int_0^1
		u^{-\alpha-1}
		\mathbb P(Y\le u)\,du.
		\)
		Applying this identity to the one-step law
		\(P(\mathbf{x},\cdot)\), and splitting
		the integral at \(u=\vartheta_\delta  d\) and
		\(u=d_0\), gives
		\begin{align*}
		\mathcal M(\mathbf{x})
		\le
		1+
		\dfrac{\alpha}{ d}
		\int_0^{\vartheta_\delta  d}
		u^{-\alpha}\,du
		+
		\alpha\vartheta_\delta
		\int_{\vartheta_\delta  d}^{d_0}
		u^{-\alpha-1}\,du
		+
		\alpha
		\int_{d_0}^1
		u^{-\alpha-1}\,du
		\le
		\dfrac{\vartheta_\delta^{1-\alpha}}{1-\alpha}
		 d^{-\alpha}
		+
		d_0^{-\alpha}.
		\end{align*}
		Thus
		\begin{equation}
			\mathcal M(\mathbf{x})
			\le
			\beta_\delta  d^{-\alpha}
			+
			d_0^{-\alpha},
			\label{eq:negative-moment-small}
		\end{equation}
		where
		\(
		\beta_\delta
		:=
		\dfrac{\vartheta_\delta^{1-\alpha}}{1-\alpha}.
		\)
		By the definition of \(\vartheta_\delta\),
		\(
		\beta_\delta
		=
		\dfrac{\vartheta_\delta^s}{s}
		=
		\dfrac12.
		\)
		If \( d>d_0\), then \eqref{eq:shell-general}
		gives, for \(u\le d_0\),
		\[
		P\bigl(\mathbf{x},\{\mathbf{y}:D(\mathbf{y})\le
		u\}\bigr)
		\le
		\dfrac{u}{ d}
		\le
		\dfrac{u}{d_0}.
		\]
		Therefore
		\begin{align*}
		\mathcal M(\mathbf{x})
		\le
		1+
		\dfrac{\alpha}{d_0}
		\int_0^{d_0}u^{-\alpha}\,du
		+
		\alpha
		\int_{d_0}^1u^{-\alpha-1}\,du
		=
		\dfrac1{1-\alpha}d_0^{-\alpha}
		=
		\dfrac1s d_0^{-\alpha}.
		\end{align*}
		Hence
		\begin{equation}
			\mathcal M(\mathbf{x})
			\le
			\dfrac1s d_0^{-\alpha}.
			\label{eq:negative-moment-large}
		\end{equation}

		The two cases imply the pointwise bound
		\(
		\mathcal M(\mathbf{x})
		\le
		\dfrac12D(\mathbf{x})^{-\alpha}
		+
		\dfrac1s d_0^{-\alpha}.
		\)
		Writing
		\(
		M_t:=\mathbb E[D(\mathbf{X}_t)^{-\alpha}],
		\)
		and using the Markov property, we obtain
		\(
		M_{t+1}
		=
		\mathbb E[\mathcal M(\mathbf{X}_t)]
		\le
		\dfrac12M_t
		+
		\dfrac1s d_0^{-\alpha}.
		\)
		Since \(D(\mathbf{X}_0)=1\),
		\begin{equation}
			\sup_{t\ge0}
			\mathbb E[D(\mathbf{X}_t)^{-\alpha}]
			\le
			\dfrac{2}{s}
			d_0^{-\alpha}.
			\label{eq:negative-moment-general}
		\end{equation}
		
		Use the values of \(\tau\) and \(\eta\) from the statement, and set
\(
\delta_{\mathrm{bis}}
:=
\eta\tau r_{\mathrm{in}}.
\)
		At a current point \(\mathbf{x}\in K\), draw
		\(
		\boldsymbol{\theta}\sim\operatorname{Unif}(S^{n-1}).
		\)
		For each \(\sigma\in\{+1,-1\}\), initialize
		\(
		L_0^\sigma:=0,
		\qquad
		U_0^\sigma:=3R_{\mathrm{out}}.
		\)
		Given \(L_j^\sigma,U_j^\sigma\), set
		\(
		m_j^\sigma
		:=
		\dfrac{L_j^\sigma+U_j^\sigma}{2}
		\)
		and update
		\[
		(L_{j+1}^\sigma,U_{j+1}^\sigma)
		=
		\begin{cases}
			(m_j^\sigma,U_j^\sigma),
			&
			\widehat F(\mathbf{x}+\sigma
			m_j^\sigma\boldsymbol{\theta})\le1-\xi,
			\\[1mm]
			(L_j^\sigma,m_j^\sigma),
			&
			\widehat F(\mathbf{x}+\sigma
			m_j^\sigma\boldsymbol{\theta})>1-\xi.
		\end{cases}
		\]
		The procedure stops once
		\(
		U_j^\sigma-L_j^\sigma
		\le
		\delta_{\mathrm{bis}},
		\)
		and returns
		\(
		\widehat b_\sigma:=L_j^\sigma.
		\)
		The next state is sampled uniformly on the resulting
		approximate
		chord:
		\[
		\widehat t
		\sim
		\operatorname{Unif}
		([-\widehat b_-,\widehat b_+]),
		\qquad
		\widehat{\mathbf{X}}_{t+1}
		=
		\mathbf{x}+\widehat t\boldsymbol{\theta}.
		\]
		Whenever a midpoint is accepted, it lies in \(K\), since
		\(
		\{\widehat F\le1-\xi\}\subseteq K.
		\)
		Moreover \(\mathbf{x}\in K\), so convexity of \(K\)
		implies that the entire
		approximate chord is contained in \(K\).
		
		To compare the noisy and exact chords, fix
		\(\mathbf{x}\in K\) and a
		direction \(\boldsymbol{\theta}\). Let
		\(
		I_K=[-b_-,b_+]
		\)
		be the exact chord of \(K\) through \(\mathbf{x}\) in
		direction \(\boldsymbol{\theta}\).
		Whenever \(\mathbf{x}\in\rho K\), let
		\(
		I_{\rho K}
		=
		[-b_-^\rho,b_+^\rho]
		\)
		be the corresponding chord of \(\rho K\).
		
		We first record the following deterministic property of
		the above
		bisection procedure. For each
		\(\sigma\in\{+1,-1\}\), if \(\mathbf{x}\in\rho K\), then
		throughout the
		bisection,
		\begin{equation}
			L_j^\sigma\le b_\sigma,
			\qquad
			U_j^\sigma\ge b_\sigma^\rho.
			\label{eq:bisection-invariants}
		\end{equation}
		Indeed, the inequalities hold initially. Since
		\(K\subseteq R_{\mathrm{out}}B_2^n\) and
		\(\|\mathbf{x}\|_2\le R_{\mathrm{out}}\),
		the point
		\(
		\mathbf{x}+\sigma(3R_{\mathrm{out}})\boldsymbol{\theta}
		\)
		lies outside \(K\), so
		\[
		3R_{\mathrm{out}}\ge b_\sigma\ge b_\sigma^\rho.
		\]
		
		Suppose \eqref{eq:bisection-invariants} holds at step
		\(j\).
		If \(m_j^\sigma\) is accepted, then
		\(
		\mathbf{x}+\sigma m_j^\sigma\boldsymbol{\theta}\in K,
		\)
		and therefore
		\(
		m_j^\sigma\le b_\sigma.
		\)
		Hence replacing \(L_j^\sigma\) by \(m_j^\sigma\)
		preserves the first
		invariant.
		
		If \(m_j^\sigma\) is rejected, then
		\(
		\mathbf{x}+\sigma
		m_j^\sigma\boldsymbol{\theta}\notin\rho K,
		\)
		because every point of \(\rho K\) is necessarily
		accepted by the
		sandwich
		\(
		\rho K
		\subseteq
		\{\widehat F\le1-\xi\}.
		\)
		Since \(\mathbf{x}\in\rho K\) and \(\rho K\) is convex,
		the intersection of
		the ray
		\(
		\{\mathbf{x}+\sigma t\boldsymbol{\theta}:t\ge0\}
		\)
		with \(\rho K\) is exactly the interval
		\([0,b_\sigma^\rho]\). Thus rejection implies
		\(
		m_j^\sigma\ge b_\sigma^\rho,
		\)
		so replacing \(U_j^\sigma\) by \(m_j^\sigma\) preserves
		the second
		invariant. This proves \eqref{eq:bisection-invariants}.
		
		Consequently, upon termination,
		\[
		b_\sigma^\rho-\delta_{\mathrm{bis}}
		\le
		\widehat b_\sigma
		\le
		b_\sigma,
		\qquad
		\sigma\in\{+1,-1\}.
		\]
		In particular, if
		\(
		\widehat I
		:=
		[-\widehat b_-,\widehat b_+],
		\)
		then
		\begin{equation}
			\widehat I\subseteq I_K,
			\qquad
			|\widehat I|
			\ge
			|I_{\rho K}|-2\delta_{\mathrm{bis}}.
			\label{eq:noisy-chord-length}
		\end{equation}
		Notice that this argument uses only
		\(
		\rho K
		\subseteq
		\{\widehat F\le1-\xi\}
		\subseteq K
		\)
		and the convexity of \(K\) and \(\rho K\); no convexity
		or
		linewise monotonicity of the noisy sublevel set
		\(\{\widehat F\le1-\xi\}\) is required.
		
		Let \(\widehat P\) denote the noisy transition kernel
		and define
		\[
		e(\mathbf{x})
		:=
		d_{\mathrm{TV}}
		\bigl(
		P(\mathbf{x},\cdot),\widehat P(\mathbf{x},\cdot)
		\bigr).
		\]
		For the one-step kernel bound, set
		\(
		 d:=D(\mathbf{x})=1-F(\mathbf{x}).
		\)
		If
		\(
		 d\le q\tau,
		\)
		then trivially
		\[
		e(\mathbf{x})\le1
		=
		\min\left\{
		1,\dfrac{q\tau}{ d}
		\right\}
		\le
		\min\left\{
		1,\dfrac{q\tau}{ d}
		\right\}
		+\eta.
		\]
		
		Suppose now that
		\(
		 d>q\tau.
		\)
		Since
		\(
		\rho^q=(1-\tau)^q\ge1-q\tau,
		\)
		we have
		\(
		F(\mathbf{x})=1- d<1-q\tau\le\rho^q,
		\)
		and hence
		\(
		\mathbf{x}\in\rho K.
		\)
		Thus the preceding bisection invariants apply.
		
		Write
		\(
		s_{ d}:=(1- d)^{1/q}.
		\)
		The same homothetic chord comparison used above gives
		\(
		\dfrac{|I_{\rho K}|}{|I_K|}
		\ge
		\dfrac{\rho-s_{ d}}{1-s_{ d}}.
		\)
		Therefore
		\begin{align}
			1-\dfrac{|I_{\rho K}|}{|I_K|}
			\le
			\dfrac{1-\rho}{1-s_{ d}}
			\nonumber=
			\dfrac{\tau}
			{1-(1- d)^{1/q}}.
			\label{eq:inner-chord-loss}
		\end{align}
		Since
		\(
		g( d):=1-(1- d)^{1/q}
		\)
		is convex, \(g(0)=0\), and \(g'(0)=1/q\),
		\(
		1-(1- d)^{1/q}
		\ge
		\dfrac{ d}{q}.
		\)
		Hence
		\begin{equation}
			1-\dfrac{|I_{\rho K}|}{|I_K|}
			\le
			\dfrac{q\tau}{ d}.
			\label{eq:inner-chord-relative-loss}
		\end{equation}
		
		We also need a lower bound on the exact chord length.
		Since
		\(\mathbf{x}\in\rho K\), write
		\(
		\mathbf{x}=\rho \mathbf{y}
		\qquad
		\text{for some }\mathbf{y}\in K.
		\)
		For every \(\mathbf{z}\in r_{\mathrm{in}}B_2^n\subseteq
		K\), convexity gives
		\(
		\rho \mathbf{y}+\tau \mathbf{z}\in K.
		\)
		Thus
		\(
		\mathbf{x}+\tau r_{\mathrm{in}}B_2^n
		\subseteq K,
		\)
		and consequently, in every direction,
		\(
		|I_K|
		\ge
		2\tau r_{\mathrm{in}}.
		\)
		Using
		\(
		\delta_{\mathrm{bis}}
		=\eta\tau r_{\mathrm{in}},
		\)
		\begin{equation}
			\dfrac{2\delta_{\mathrm{bis}}}{|I_K|}
			\le
			\eta.
			\label{eq:bisection-relative-loss}
		\end{equation}
		
		Conditional on the sampled direction
		\(\boldsymbol{\theta}\), the exact and noisy
		Hit-and-Run transitions are respectively uniform on
		\(I_K\) and
		\(\widehat I\). Since \(\widehat I\subseteq I_K\),
		\[
		d_{\mathrm{TV}}
		\left(
		\operatorname{Unif}(I_K),
		\operatorname{Unif}(\widehat I)
		\right)
		=
		1-\dfrac{|\widehat I|}{|I_K|}.
		\]
		By \eqref{eq:noisy-chord-length},
		\eqref{eq:inner-chord-relative-loss}, and
		\eqref{eq:bisection-relative-loss},
		\[
		1-\dfrac{|\widehat I|}{|I_K|}
		\le
		1-\dfrac{|I_{\rho K}|}{|I_K|}
		+
		\dfrac{2\delta_{\mathrm{bis}}}{|I_K|}
		\le
		\dfrac{q\tau}{ d}
		+\eta.
		\]
		Averaging over \(\boldsymbol{\theta}\) therefore yields
		\begin{equation}
			e(\mathbf{x})
			\le
			\min\left\{
			1,\dfrac{q\tau}{D(\mathbf{x})}
			\right\}
			+\eta.
			\label{eq:kernel-error-general}
		\end{equation}
		
		Since
		\(
		\min\{1,a\}\le a^\alpha,
		a\ge0,
		\)
		\eqref{eq:negative-moment-general} and
		\eqref{eq:kernel-error-general} imply
		\begin{equation}
			\sup_{t\ge0}
			\mathbb
			E_{\mathbf{X}_t\sim\delta_{\mathbf{0}}P^t}[e(\mathbf{X}_t)]
			\le
			\dfrac{2}{s}
			\left(
			\dfrac{q\tau}{d_0}
			\right)^\alpha
			+\eta.
			\label{eq:average-kernel-error-general}
		\end{equation}
		
		Using the telescoping identity
		\(
		\widehat P^T-P^T
		=
		\sum_{j=0}^{T-1}
		P^j(\widehat P-P)\widehat P^{T-1-j}
		\)
		and contraction of total variation under Markov kernels,
		\eqref{eq:average-kernel-error-general} gives
		\begin{equation}
			d_{\mathrm{TV}}
			\left(
			\delta_{\mathbf{0}}\widehat P^T,
			\delta_{\mathbf{0}}P^T
			\right)
			\le
			\dfrac{2T}{s}
			\left(
			\dfrac{q\tau}{d_0}
			\right)^\alpha
			+
			T\eta.
			\label{eq:multi-step-error-general}
		\end{equation}
		
		By Corollary~1.2 of \citet{lovasz2006hitandrun}, the
		exact
		Hit-and-Run chain on a convex body containing a
		Euclidean ball of
		radius \(r_{\mathrm{in}}\) and contained in a Euclidean
		ball of radius
		\(R_{\mathrm{out}}\), when initialized at a point
		\(\mathbf{x}_0\in K\), satisfies
		\[
		d_{\mathrm{TV}}
		\left(
		\delta_{\mathbf{x}_0}P^T,
		\operatorname{Unif}(K)
		\right)
		\le \dfrac{\varepsilon}{2}
		\]
		provided
		\[
		T
		\ge
		C n^3
		\dfrac{R_{\mathrm{out}}^2}{r_{\mathrm{in}}^2}
		\log
		\dfrac{
			2R_{\mathrm{out}}
		}{
			\varepsilon\,\operatorname{dist}(\mathbf{x}_0,\partial
			K)
		},
		\]
		for an absolute constant \(C>0\).
		In our setting \(\mathbf{x}_0=\mathbf{0}\), and
		\(
		r_{\mathrm{in}}B_2^n\subset K
		\)
		implies
		\(
		\operatorname{dist}(\mathbf{0},\partial K)\ge
		r_{\mathrm{in}}.
		\)
		Therefore, with
		\(
		\omega
		:=
		\dfrac{R_{\mathrm{out}}}{r_{\mathrm{in}}},
		\)
		it suffices to take
		\(
		T
		=
		\left\lceil
		C n^3\omega^2
		\log\dfrac{2\omega}{\varepsilon}
		\right\rceil,
		\)
		for which
		\[
		d_{\mathrm{TV}}
		\left(
		\delta_{\mathbf{0}}P^T,
		\operatorname{Unif}(K)
		\right)
		\le
		\dfrac{\varepsilon}{2}.
		\]
		Since
		\(
		T\eta=\dfrac{\varepsilon}{4},
		\)
		it remains to impose
		\(
		\dfrac{2T}{s}
		\left(
		\dfrac{q\tau}{d_0}
		\right)^\alpha
		\le
		\dfrac{\varepsilon}{4}.
		\)
		By \eqref{eq:tau-xi}, it is sufficient that
		\(
		\xi
		\le
		c\,d_0
		\left(
		\dfrac{s\varepsilon}{T}
		\right)^{1/\alpha}.
		\)
		Using
		\(
		\dfrac1\alpha=1+\delta,
		d_0
		=
		\dfrac{
			c_0\vartheta_\delta^3
		}{
			n\omega^2q^2
		},
		\)
		and
		\(
		T
		=
		O\!\left(
		n^3\omega^2
		\log\dfrac{2\omega}{\varepsilon}
		\right),
		\)
		we obtain
		$$
		\xi
		\le
		\dfrac{
			c_\delta\,\varepsilon^{1+\delta}
		}{
			q^2
			n^{4+3\delta}
			\omega^{4+2\delta}
			\left(
			\log\dfrac{2\omega}{\varepsilon}
			\right)^{1+\delta}
		},
		$$
		where
		\[
		c_\delta
		=
		c
		s^{1+\delta}
		\vartheta_\delta^3
		=
		c
		\left(\dfrac{\delta}{1+\delta}\right)^{1+\delta}
		\left(
		\dfrac{\delta}{2(1+\delta)}
		\right)^{3(1+\delta)/\delta}.
		\]
		The triangle inequality together with
		\eqref{eq:multi-step-error-general} therefore gives
		$
		d_{\mathrm{TV}}
		\left(
		\mathcal L(\widehat{\mathbf{X}}_T),
		\operatorname{Unif}(K)
		\right)
		\le
		\varepsilon.
		$
		
		Finally, each endpoint bisection requires at most
		\(
		\left\lceil
		\log_2
		\dfrac{3R_{\mathrm{out}}}{\delta_{\mathrm{bis}}}
		\right\rceil
		\)
		oracle evaluations. Hence the total number of oracle
		calls is
		\(
		O\left(
		T
		\log
		\dfrac{\omega}{\eta\tau}
		\right).
		\)
	\end{proof}
	
\subsection{Population Moments and Spectral Conditioning}

We now quantify the population position and gradient moments generated by uniform sampling and verify that the resulting spectral pencil is sufficiently well conditioned for diagonalization.

The following lemma provides the moment scales, envelope bounds, and spectral conditioning parameters used in the perturbation analysis of ASPIRE.

	\begin{lemma}[Population moment bounds and spectral conditioning]
		\label{lem:population-moment-scales}
		Fix \(1\le\ell\le L-2\), and write
		\(
		n:=r_\ell,
		q:=q_\ell,
		G:=G_\ell,
		K:=K_\ell^\sharp.
		\)
		Let
		\(
		\rho B_2^n\subseteq K\subseteq R B_2^n,
		\omega:=\dfrac R\rho,
		\)
		and let \(\mathbf{Z}\sim\operatorname{Unif}(K)\). Define
		\[
		\boldsymbol{\Sigma}_x:=\mathbb
		E[\mathbf{Z}\mathbf{Z}^\top],
		\qquad
		\boldsymbol{\Sigma}_g:=\mathbb E[\nabla
		G(\mathbf{Z})\nabla G(\mathbf{Z})^\top].
		\]
		Then
		\[
		\dfrac{\rho^2}{n+2}
		\le
		\|\boldsymbol{\Sigma}_x\|_2
		\le
		R^2,
		\qquad
		\dfrac{q^2}{R^2(n+2q-2)}
		\le
		\|\boldsymbol{\Sigma}_g\|_2
		\le
		\dfrac{q^2}{\rho^2}.
		\]
		In particular,
		\begin{equation}
		\dfrac{\sup_{\mathbf{z}\in
		K}\|\mathbf{z}\mathbf{z}^\top\|_2}{\|\boldsymbol{\Sigma}_x\|_2}
		\le
		(n+2)\omega^2,
		\qquad
		\dfrac{
		\sup_{\mathbf{z}\in K}\|\nabla G(\mathbf{z})\nabla
		G(\mathbf{z})^\top\|_2
		}{
		\|\boldsymbol{\Sigma}_g\|_2
		}
		\le
		(n+2q-2)\omega^2.
		\label{eq:moment-envelope-ratios}
		\end{equation}
		
		Let
		\(
		\mathbf{C}:=\mathbf{C}_\ell^\sharp,
		\mathbf{U}:=\mathbf{C}^\top \mathbf{Z},
		\)
		and write
		\(
		\boldsymbol{\Sigma}_x=\mathbf{C}^{-\top}\boldsymbol{\Lambda}
		\mathbf{C}^{-1},\boldsymbol{\Sigma}_g=\mathbf{C}\mathbf{D}\mathbf{C}^\top,
		\)
		where
		\[\boldsymbol{\Lambda}=\operatorname{diag}(\lambda_1,\ldots,\lambda_n),\qquad 
		\mathbf{D}=\operatorname{diag}(d_1,\ldots,d_n)\]. If
		\(
		\theta_i:=\lambda_i d_i,
		\mathbf{v}_i:=\dfrac{\mathbf{c}_i}{\|\mathbf{c}_i\|_2},
		\mathbf{V}:=[\mathbf{v}_1,\ldots,\mathbf{v}_n],
		\)
		then
		\begin{equation}
		\operatorname{cond}(\mathbf{V})
		\le
		\operatorname{cond}(\mathbf{C})^2
		\le
		9\kappa_0^2,
		\qquad
		\rho_\ell
		:=
		\dfrac{\|\boldsymbol{\Sigma}_x\|_2\|\boldsymbol{\Sigma}_g\|_2}{\max_i\theta_i}
		\le
		(n+q)^2\omega^2.
		\label{eq:explicit-pencil-conditioning}
		\end{equation}
		Under the standing assumptions, this simplifies to
		\(
		\operatorname{cond}(\mathbf{V})\le 9\kappa^2,
		\rho_\ell
		\le
		C_{k,L}\kappa^8r^3.
		\)
	\end{lemma}
	
	\begin{proof}
	Set
	\(
	p(\mathbf{z}):=G(\mathbf{z})^{1/q}.
	\)
	Because \(G\) is positive definite, convex, even, and
	\(q\)-homogeneous,
	\(p\) is a norm whose unit ball is \(K\). Homogeneity of
	Lebesgue measure
	then gives
	\(
	\mathbb P(p(\mathbf{Z})\le t)
	=
	t^n,
	0\le t\le1.
	\)
	Consequently, for every \(a>-n\),
	\[
	\begin{aligned}
	\mathbb E[p(\mathbf{Z})^a]
	=
	\int_0^1 t^a\,d(t^n) 
	=
	n\int_0^1 t^{n+a-1}\,dt 
	=
	\dfrac{n}{n+a}.
	\end{aligned}
	\]

	The inclusions
	\(
	\rho B_2^n\subseteq K\subseteq R B_2^n
	\)
	are equivalent to
	\(
	\rho p(\mathbf{z})
	\le
	\|\mathbf{z}\|_2
	\le
	R p(\mathbf{z}),
	\mathbf{z}\in\mathbb R^n.
	\)
	Therefore
	\[
	\begin{aligned}
	\operatorname{tr}(\boldsymbol{\Sigma}_x)
	=
	\mathbb E\|\mathbf{Z}\|_2^2 
	\ge
	\rho^2\mathbb E[p(\mathbf{Z})^2] 
	=
	\rho^2\dfrac{n}{n+2}.
	\end{aligned}
	\]
	Since \(\boldsymbol{\Sigma}_x\succeq\mathbf{0}\),
	\[
	\|\boldsymbol{\Sigma}_x\|_2
	\ge
	\dfrac1n\operatorname{tr}(\boldsymbol{\Sigma}_x)
	\ge
	\dfrac{\rho^2}{n+2}.
	\]
	On the other hand, \(\mathbf{Z}\in K\subseteq RB_2^n\)
	almost surely, so
	\(
	\mathbf{0}\preceq \mathbf{Z}\mathbf{Z}^\top\preceq
	R^2\mathbf{I}
	\)
	and hence
	\(
	\|\boldsymbol{\Sigma}_x\|_2\le R^2.
	\)
	These estimates also give
	\[
	\dfrac{\sup_{\mathbf{z}\in
	K}\|\mathbf{z}\mathbf{z}^\top\|_2}{\|\boldsymbol{\Sigma}_x\|_2}
	\le
	(n+2)\dfrac{R^2}{\rho^2}
	=
	(n+2)\omega^2.
	\]

	For \(\mathbf{z}\ne\mathbf{0}\), differentiation of
	\(G=p^q\) gives
	\(
	\nabla G(\mathbf{z})
	=
	q p(\mathbf{z})^{q-1}\nabla p(\mathbf{z}).
	\)
	Let \(p_*\) denote the dual norm. At every
	differentiability point of the
	norm \(p\),
	\(
	p_*(\nabla p(\mathbf{z}))=1.
	\)
	The inclusion \(\rho B_2^n\subseteq K\) implies
	\[
	p_*(\mathbf{v})
	=
	\sup_{p(\mathbf{x})\le1}\mathbf{v}^\top \mathbf{x}
	\ge
	\rho\|\mathbf{v}\|_2.
	\]
	Thus
	\(
	\|\nabla p(\mathbf{z})\|_2
	\le
	\dfrac1\rho,
	\)
	and consequently
	\(
	\|\nabla G(\mathbf{z})\|_2
	\le
	\dfrac{q}{\rho}p(\mathbf{z})^{q-1}
	\le
	\dfrac q\rho,
	\mathbf{z}\in K.
	\)
	This proves
	\(
	\|\boldsymbol{\Sigma}_g\|_2
	\le
	\dfrac{q^2}{\rho^2}.
	\)
	For the reverse estimate, Euler's identity gives
	\(
	\mathbf{z}^\top\nabla G(\mathbf{z})
	=
	qG(\mathbf{z})
	=
	q p(\mathbf{z})^q.
	\)
	Since \(\|\mathbf{z}\|_2\le R p(\mathbf{z})\),
	Cauchy--Schwarz yields
	\(
	\|\nabla G(\mathbf{z})\|_2
	\ge
	\dfrac qR p(\mathbf{z})^{q-1}.
	\)
	Therefore
	\[
	\begin{aligned}
	\operatorname{tr}(\boldsymbol{\Sigma}_g)
    =
	\mathbb E\|\nabla G(\mathbf{Z})\|_2^2 
	\ge
	\dfrac{q^2}{R^2}
	\mathbb E[p(\mathbf{Z})^{2q-2}] 
	=
	\dfrac{q^2}{R^2}
	\dfrac{n}{n+2q-2}.
	\end{aligned}
	\]
	Hence
	\[
	\|\boldsymbol{\Sigma}_g\|_2
	\ge
	\dfrac1n\operatorname{tr}(\boldsymbol{\Sigma}_g)
	\ge
	\dfrac{q^2}{R^2(n+2q-2)}.
	\]
	Combining the upper envelope \(\|\nabla
	G(\mathbf{z})\|_2^2\le q^2/\rho^2\)
	with this lower bound gives
	\[
	\dfrac{
	\sup_{\mathbf{z}\in K}\|\nabla G(\mathbf{z})\nabla
	G(\mathbf{z})^\top\|_2
	}{
	\|\boldsymbol{\Sigma}_g\|_2
	}
	\le
	(n+2q-2)\omega^2.
	\]
	This proves \eqref{eq:moment-envelope-ratios}.

	We next control the spectral pencil. The random vector
	\(
	\mathbf{U}=\mathbf{C}^\top \mathbf{Z}
	\)
	is uniform on
	\(
	H_\ell
	:=
	\{\mathbf{u}\in\mathbb R^n:h^{(\ell)}(\mathbf{u})\le1\}.
	\)
	Coordinatewise sign symmetry implies that
	\(
	\boldsymbol{\Lambda}=\mathbb E[\mathbf{U}\mathbf{U}^\top]
	\)
	and
	\[
	\mathbf{D}
	=
	\mathbb E[
	\nabla h^{(\ell)}(\mathbf{U})\nabla
	h^{(\ell)}(\mathbf{U})^\top
	]
	\]
	are diagonal. Let \(h:=h^{(\ell)}\). On \(\partial
	H_\ell\),
	\(
	h(\mathbf{u})=1.
	\)
	Euler's identity gives
	\(
	\mathbf{u}^\top\nabla h(\mathbf{u})=q,
	\qquad
	\mathbf{u}\in\partial H_\ell,
	\)
	so \(\nabla h\) doesn't vanish on the boundary. The
	divergence theorem is
	applicable. For every coordinate \(i\),
	\[
	\begin{aligned}
	\int_{H_\ell}\partial_i(u_i h(\mathbf{u}))\,d\mathbf{u}
	=
	\int_{\partial H_\ell}u_i
	h(\mathbf{u})\nu_i(\mathbf{u})\,dS(\mathbf{u}) 
	=
	\int_{\partial
	H_\ell}u_i\nu_i(\mathbf{u})\,dS(\mathbf{u}) 
	=
	|H_\ell|.
	\end{aligned}
	\]
	Expanding the derivative on the left gives
	\[
	\int_{H_\ell}h(\mathbf{u})\,d\mathbf{u}
	+
	\int_{H_\ell}u_i\partial_i h(\mathbf{u})\,d\mathbf{u}
	=
	|H_\ell|.
	\]
	The same radial integration used above gives
	\(
	\dfrac1{|H_\ell|}
	\int_{H_\ell}h(\mathbf{u})\,d\mathbf{u}
	=
	\dfrac{n}{n+q}.
	\)
	It follows that
	\(
	\mathbb E[U_i\partial_i h(\mathbf{U})]
	=
	1-\dfrac{n}{n+q}
	=
	\dfrac{q}{n+q}.
	\)
	By Cauchy--Schwarz,
	\[
	\begin{aligned}
	\theta_i
	=
	\mathbb E[U_i^2]
	\mathbb E[(\partial_i h(\mathbf{U}))^2]
	\ge
	\bigl(\mathbb E[U_i\partial_i h(\mathbf{U})]\bigr)^2 
	=
	\dfrac{q^2}{(n+q)^2}.
	\end{aligned}
	\]
	Consequently,
	\[
	\begin{aligned}
	\rho_\ell
	=
	\dfrac{\|\boldsymbol{\Sigma}_x\|_2\|\boldsymbol{\Sigma}_g\|_2}
	{\max_i\theta_i} 
	\le
	\dfrac{R^2(q^2/\rho^2)}{q^2/(n+q)^2} 
	=
	(n+q)^2\omega^2.
	\end{aligned}
	\]

	Finally, define
	\(
	\mathbf{S}
	:=
	\operatorname{diag}(\|\mathbf{c}_1\|_2,\ldots,\|\mathbf{c}_n\|_2).
	\)
	Then
	\(
	\mathbf{V}=\mathbf{C}\mathbf{S}^{-1}.
	\)
	For every \(i\),
	\(
	\sigma_{\min}(\mathbf{C})
	\le
	\|\mathbf{c}_i\|_2
	\le
	\sigma_{\max}(\mathbf{C}),
	\)
	so
	\[
	\|\mathbf{S}\|_2\le\sigma_{\max}(\mathbf{C}),
	\qquad
	\|\mathbf{S}^{-1}\|_2\le\sigma_{\min}(\mathbf{C})^{-1}.
	\]
	Therefore
	\[
	\begin{aligned}
	\operatorname{cond}(\mathbf{V})
	\le
	\operatorname{cond}(\mathbf{C})\operatorname{cond}(\mathbf{S})
	\le
	\operatorname{cond}(\mathbf{C})^2.
	\end{aligned}
	\]
	Lemma~\ref{lem:reduced-weight-conditioning} gives
	\(
	\operatorname{cond}(\mathbf{C})
	\le
	3\kappa_0,
	\)
	and hence
	\(
	\operatorname{cond}(\mathbf{V})
	\le
	9\kappa_0^2.
	\)
	Together with
	\[
	\omega\le C_{k,L}\kappa^4\sqrt r,
	\qquad
	n\le r,
	\qquad
	q=O_{k,L}(1),
	\]
	the pencil bound becomes
	\(
	\rho_\ell
	\le
	C_{k,L}\kappa^8r^3.
	\)
	This proves \eqref{eq:explicit-pencil-conditioning} and
	its stated
	simplification.
\end{proof}

\subsection{Stable Gradient and Empirical Moment Estimation}

We finally convert multiplicative suffix-oracle accuracy and approximate sampling into accurate empirical estimates of the two moment matrices used by ASPIRE.

The next lemma shows that polynomial interpolation converts a multiplicative function-value error into a controlled additive error in the recovered gradient.

	\begin{lemma}[Stable gradient recovery from a multiplicative oracle]
		\label{lem:stable-gradient-from-suffix}
		Suppose an available oracle satisfies
		\[
		|\widehat G_\ell(\mathbf{z})-G_\ell(\mathbf{z})|
		\le
		\xi_\ell G_\ell(\mathbf{z}),
		\qquad
		\mathbf{z}\in\mathbb R^{r_\ell}.
		\]
		Then, for every \(\mathbf{z}\in K_\ell^\sharp\), one can
		construct
		\(\widehat{\mathbf{g}}_\ell(\mathbf{z})\) using
		\(O(q_\ell r_\ell)\) evaluations of \(\widehat G_\ell\)
		such that
		\[
		\|\widehat{\mathbf{g}}_\ell(\mathbf{z})-\nabla
		G_\ell(\mathbf{z})\|_2
		\le
		C_{q_\ell}
		\dfrac{\sqrt{r_\ell}}{r_{\mathrm{in},\ell}}
		\xi_\ell.
		\]
	\end{lemma}
	
	\begin{proof}
		Write
		\(
		n:=r_\ell,
		q:=q_\ell,
\rho:=r_{\mathrm{in},\ell}.
		\)
		Fix distinct nodes
		\(
		t_0,\ldots,t_q\in[-1,1],
		\)
		and let \(L_0,\ldots,L_q\) be the associated Lagrange
		basis
		polynomials. Set
		\(
		\beta_a:=L_a'(0).
		\)
		For \(1\le j\le n\), define
		\(
		P_j(t)
		:=
		G_\ell\!\left(\mathbf{z}+\dfrac{\rho}{2}t\mathbf{e}_j\right).
		\)
		Since \(G_\ell\) has degree \(q\),
		\(
		P_j'(0)
		=
		\sum_{a=0}^q
		\beta_a P_j(t_a).
		\)
		Define
		\[
		\widehat g_{\ell,j}(\mathbf{z})
		:=
		\dfrac2\rho
		\sum_{a=0}^q
		\beta_a
		\widehat G_\ell
		\!\left(\mathbf{z}+\dfrac{\rho}{2}t_a\mathbf{e}_j\right),\qquad p_\ell:=G_\ell^{1/q}
		\]
		 Since
		\(
		p_\ell(\mathbf{z})\le1,
p_\ell(\rho\mathbf{e}_j)\le1,
		\)
		we have
		\(
		p_\ell\!\left(
		\mathbf{z}+\dfrac{\rho}{2}t_a\mathbf{e}_j
		\right)
		\le
		\dfrac32.
		\)
		Hence
		\(
		G_\ell\left(
		\mathbf{z}+\dfrac{\rho}{2}t_a\mathbf{e}_j
		\right)
		\le
		\left(\dfrac32\right)^q.
		\)
		Therefore
		$$
		\left|
		\widehat g_{\ell,j}(\mathbf{z})
		-\partial_jG_\ell(\mathbf{z})
		\right|
		\le
		C_q\dfrac{\xi_\ell}{\rho}.
		$$
		Summing over the coordinates proves $\|\widehat{\mathbf{g}}_\ell(\mathbf{z})-\nabla
		G_\ell(\mathbf{z})\|_2
		\le
		C_q
		\dfrac{\sqrt n}{\rho}\xi_\ell.
		$
	\end{proof}

Combining gradient stability, total-variation control of the sampling distribution, and matrix concentration, the final lemma gives simultaneous relative-error guarantees for the two empirical moments.

	\begin{lemma}[Empirical recovery of the two moments]
		\label{lem:empirical-two-moments}
		Fix \(1\le\ell\le L-2\), write
		\(		n:=r_\ell,
		q:=q_\ell,
		\omega:=\omega_\ell,
		\)
		and let
		\(\pi_\ell:=\operatorname{Unif}(K_\ell^\sharp)\).
		Suppose the
		endpoint distribution \(\widehat\pi_\ell\) of each
		independent noisy
		Hit-and-Run chain satisfies
		\[
		d_{\mathrm{TV}}(\widehat\pi_\ell,\pi_\ell)\le\nu,
		\]
		and suppose that
		\[
		|\widehat G_\ell(\mathbf{z})-G_\ell(\mathbf{z})|
		\le
		\xi_\ell G_\ell(\mathbf{z}),
		\qquad \mathbf{z}\in\mathbb R^n.
		\]
		Let \(\eta,\zeta\in(0,1/2)\). There are constants
		\(c_q,C_q>0\), depending only on \(q\), such that if
		\begin{equation}
		\nu
		\le
		\dfrac{c_q\eta}{(n+2q)\omega^2},
		\qquad
		\xi_\ell
		\le
		\dfrac{c_q\eta}{\sqrt n\,(n+2q)\omega^2},
		\label{eq:empirical-moment-noise-conditions}
		\end{equation}
		and
		\begin{equation}
		m
		\ge
		C_q(n+2q)\omega^2\eta^{-2}
		\log\dfrac{2n}{\zeta},
		\label{eq:empirical-moment-sample-size}
		\end{equation}
		then \(m\) independent endpoints are sufficient to
		construct
		\[
		\widehat{\boldsymbol{\Sigma}}_{x,\ell}
		=
		\dfrac1m\sum_{i=1}^m \mathbf{Z}_i\mathbf{Z}_i^\top,
		\qquad
		\widehat{\boldsymbol{\Sigma}}_{g,\ell}
		=
		\dfrac1m\sum_{i=1}^m
		\widehat{\mathbf{g}}_\ell(\mathbf{Z}_i)\widehat{\mathbf{g}}_\ell(\mathbf{Z}_i)^\top
		\]
		such that, with probability at least \(1-\zeta\),
		\[
		\|\widehat{\boldsymbol{\Sigma}}_{x,\ell}-\boldsymbol{\Sigma}_{x,\ell}\|_2
		\le
		\eta\|\boldsymbol{\Sigma}_{x,\ell}\|_2,\qquad 
		\|\widehat{\boldsymbol{\Sigma}}_{g,\ell}-\boldsymbol{\Sigma}_{g,\ell}\|_2
		\le
		\eta\|\boldsymbol{\Sigma}_{g,\ell}\|_2.
		\]
		In particular, under the standing assumptions it is
		sufficient that
		\[
		\nu\le c_{k,L}\kappa^{-8}\eta r^{-2},
		\qquad
		\xi_\ell\le c_{k,L}\kappa^{-8}\eta r^{-5/2},\qquad 
		m\ge C_{k,L}\kappa^8r^2\eta^{-2}
		\log\dfrac{2r}{\zeta}.
		\]
	\end{lemma}
	
	\begin{proof}
	Write
	\(
	\rho:=r_{\mathrm{in},\ell},
	R:=R_{\mathrm{out},\ell},
	\mathbf{g}(\mathbf{z}):=\nabla G_\ell(\mathbf{z}),
	\mathbf{e}(\mathbf{z}):=\widehat{\mathbf{g}}_\ell(\mathbf{z})-\mathbf{g}(\mathbf{z}).
	\)
	We first control the bias caused by approximate sampling
	and by the
	multiplicative suffix oracle. For the position moment, every point of \(K_\ell^\sharp\)
	has norm at most
	\(R\), so
	\(
\|\mathbf{Z}\mathbf{Z}^\top\|_2\le R^2.
	\)
	The total-variation bound therefore implies
	\(
	\left\|
	\mathbb E_{\widehat\pi_\ell}[\mathbf{Z}\mathbf{Z}^\top]
	-
	\mathbb E_{\pi_\ell}[\mathbf{Z}\mathbf{Z}^\top]
	\right\|_2
	\le
	2R^2\nu.
	\)
	Using
	\(
	\|\boldsymbol{\Sigma}_{x,\ell}\|_2
	\ge
	\dfrac{\rho^2}{n+2},
	\)
	we obtain
	\[
	\dfrac{
	\left\|
	\mathbb E_{\widehat\pi_\ell}[\mathbf{Z}\mathbf{Z}^\top]
	-
	\boldsymbol{\Sigma}_{x,\ell}
	\right\|_2
	}{
	\|\boldsymbol{\Sigma}_{x,\ell}\|_2
	}
	\le
	2(n+2)\omega^2\nu.
	\]

	For the gradient moment,
	Lemmas~\ref{lem:population-moment-scales} and
	\ref{lem:stable-gradient-from-suffix} give, for every
	\(\mathbf{z}\in K_\ell^\sharp\),
	\(
	\|\mathbf{g}(\mathbf{z})\|_2
	\le
	\dfrac q\rho
	\)
	and
	\(
	\|\mathbf{e}(\mathbf{z})\|_2
	\le
	C_q\dfrac{\sqrt n}{\rho}\xi_\ell.
	\)
	The second inequality in
	\eqref{eq:empirical-moment-noise-conditions} implies,
	after decreasing
	\(c_q\), that
	\(
	\|\mathbf{e}(\mathbf{z})\|_2
	\le
	\dfrac q\rho.
	\)
	Consequently,
	\[
	\|\widehat{\mathbf{g}}_\ell(\mathbf{z})\|_2
	\le
	\|\mathbf{g}(\mathbf{z})\|_2+\|\mathbf{e}(\mathbf{z})\|_2
	\le
	\dfrac{2q}{\rho}.
	\]
	Moreover,
	\[
	\begin{aligned}
\left\|\widehat{\mathbf{g}}_\ell(\mathbf{z})\widehat{\mathbf{g}}_\ell(\mathbf{z})^\top
	-
	\mathbf{g}(\mathbf{z})\mathbf{g}(\mathbf{z})^\top
	\right\|_2 
	\le
	\|\mathbf{g}(\mathbf{z})\mathbf{e}(\mathbf{z})^\top\|_2
	+
	\|\mathbf{e}(\mathbf{z})\mathbf{g}(\mathbf{z})^\top\|_2
	+
	\|\mathbf{e}(\mathbf{z})\mathbf{e}(\mathbf{z})^\top\|_2 
	\le
	C_q\dfrac{\sqrt n}{\rho^2}\xi_\ell.
	\end{aligned}
	\]
	Define the population moment under the noisy endpoint
	distribution by
	\(\overline{\boldsymbol{\Sigma}}_{g,\ell}
	:=
	\mathbb E_{\widehat\pi_\ell}
	\left[
	\widehat{\mathbf{g}}_\ell(\mathbf{Z})\widehat{\mathbf{g}}_\ell(\mathbf{Z})^\top
	\right].
	\)
	Insert and subtract
	\(\mathbb
	E_{\widehat\pi_\ell}[\mathbf{g}(\mathbf{Z})\mathbf{g}(\mathbf{Z})^\top]\).
	The gradient-oracle
	contribution is bounded by
	\[
	\left\|
	\mathbb E_{\widehat\pi_\ell}
	[
	\widehat{\mathbf{g}}_\ell\widehat{\mathbf{g}}_\ell^\top-\mathbf{g}\mathbf{g}^\top
	]
	\right\|_2
	\le
	C_q\dfrac{\sqrt n}{\rho^2}\xi_\ell.
	\]
	The total-variation contribution is bounded by
	\[
	\left\|
	\mathbb E_{\widehat\pi_\ell}[\mathbf{g}\mathbf{g}^\top]
	-
	\mathbb E_{\pi_\ell}[\mathbf{g}\mathbf{g}^\top]
	\right\|_2
	\le
	2\dfrac{q^2}{\rho^2}\nu.
	\]
	Since
	\(
	\|\boldsymbol{\Sigma}_{g,\ell}\|_2
	\ge
	\dfrac{q^2}{R^2(n+2q-2)},
	\)
	we conclude that
	\[
	\dfrac{
	\|\overline{\boldsymbol{\Sigma}}_{g,\ell}-\boldsymbol{\Sigma}_{g,\ell}\|_2
	}{
	\|\boldsymbol{\Sigma}_{g,\ell}\|_2
	}
	\le
	C_q(n+2q)\omega^2
	\left(
	\nu+\sqrt n\,\xi_\ell
	\right).
	\]
	Together with the position-moment bound, the two
	conditions in
	\eqref{eq:empirical-moment-noise-conditions} imply
	\[
	\left\|
	\mathbb E_{\widehat\pi_\ell}[\mathbf{Z}\mathbf{Z}^\top]
	-
	\boldsymbol{\Sigma}_{x,\ell}
	\right\|_2
	\le
	\dfrac\eta2\|\boldsymbol{\Sigma}_{x,\ell}\|_2,\qquad \|\overline{\boldsymbol{\Sigma}}_{g,\ell}-\boldsymbol{\Sigma}_{g,\ell}\|_2
	\le\dfrac\eta2\|\boldsymbol{\Sigma}_{g,\ell}\|_2.
	\]

	It remains to control finite-sample fluctuations. Set
	\(\overline{\boldsymbol{\Sigma}}_{x,\ell}
	:=
	\mathbb E_{\widehat\pi_\ell}[\mathbf{Z}\mathbf{Z}^\top].
	\)
	The population-bias estimates imply
	\[
	\|\overline{\boldsymbol{\Sigma}}_{x,\ell}\|_2
	\le
	\left(1+\dfrac\eta2\right)
	\|\boldsymbol{\Sigma}_{x,\ell}\|_2
	\le\dfrac32\|\boldsymbol{\Sigma}_{x,\ell}\|_2,\qquad \|\overline{\boldsymbol{\Sigma}}_{g,\ell}\|_2
	\le
	\dfrac32\|\boldsymbol{\Sigma}_{g,\ell}\|_2.
	\]

	We first consider the position moment. Let
	\(\mathbf{X}_i:=\mathbf{Z}_i\mathbf{Z}_i^\top,
	\mathbf{S}_i:=\mathbf{X}_i-\overline{\boldsymbol{\Sigma}}_{x,\ell}.
	\)
	Because each noisy Hit-and-Run endpoint lies in
	\(K_\ell^\sharp\subseteq RB_2^n\),
	\(
	\mathbf{0}\preceq \mathbf{X}_i\preceq R^2\mathbf{I}.
	\)
	Furthermore,
	\(
	\mathbf{X}_i^2
	=
	\|\mathbf{Z}_i\|_2^2\mathbf{X}_i
	\preceq
	R^2\mathbf{X}_i.
	\)
	We have
	\(
	\mathbb E \mathbf{S}_i=\mathbf{0}
	\)
	and
	\(
	\|\mathbf{S}_i\|_2\le R^2.
	\)
	The matrix variance satisfies
	\[
	\begin{aligned}
	\mathbb E[\mathbf{S}_i^2]
	=
	\mathbb E[\mathbf{X}_i^2]
	-\overline{\boldsymbol{\Sigma}}_{x,\ell}^{\,2} 
	\preceq
	\mathbb E[\mathbf{X}_i^2] 
	\preceq
	R^2\overline{\boldsymbol{\Sigma}}_{x,\ell}.
	\end{aligned}
	\]
	Hence
	\[
	\left\|
	\sum_{i=1}^m\mathbb E[\mathbf{S}_i^2]
	\right\|_2
	\le
	mR^2\|\overline{\boldsymbol{\Sigma}}_{x,\ell}\|_2.
	\]
	The self-adjoint matrix Bernstein inequality gives, for
	every \(t>0\),
	\[
	\begin{aligned}
	\mathbb P\left(
	\left\|
	\dfrac1m\sum_{i=1}^m\mathbf{X}_i
	-
	\overline{\boldsymbol{\Sigma}}_{x,\ell}
	\right\|_2
	\ge t
	\right) 
	\le
	2n\exp\left(
	-
	\dfrac{
	mt^2
	}{
	2R^2\|\overline{\boldsymbol{\Sigma}}_{x,\ell}\|_2
	+\dfrac23R^2t
	}
	\right).
	\end{aligned}
	\]
	Take
	\(
	t
	:=
	\dfrac\eta2\|\boldsymbol{\Sigma}_{x,\ell}\|_2.
	\)
	Using \(\eta<1/2\) and the bound on
	\(\|\overline{\boldsymbol{\Sigma}}_{x,\ell}\|_2\), we
	obtain
	\[
	\begin{aligned}
	\mathbb P\left(
	\|\widehat{\boldsymbol{\Sigma}}_{x,\ell}-\overline{\boldsymbol{\Sigma}}_{x,\ell}\|_2
	\ge
	\dfrac\eta2\|\boldsymbol{\Sigma}_{x,\ell}\|_2
	\right) 
	\le
	2n\exp\left(
	-cm\eta^2
	\dfrac{\|\boldsymbol{\Sigma}_{x,\ell}\|_2}{R^2}
	\right),
	\end{aligned}
	\]
	where \(c>0\) is an absolute constant. Since
	\(
	\dfrac{R^2}{\|\boldsymbol{\Sigma}_{x,\ell}\|_2}
	\le
	(n+2)\omega^2,
	\)
	it follows that
	\(
	m
	\ge
	C_q(n+2q)\omega^2\eta^{-2}
	\log\dfrac{2n}{\zeta}
	\)
	makes the position fluctuation at most
	\(\dfrac\eta2\|\boldsymbol{\Sigma}_{x,\ell}\|_2\) with
	failure probability at most
	\(\zeta/2\), after changing \(C_q\) by an absolute
	factor.

	We next consider the gradient moment. Let
	\(
	\mathbf{Y}_i:=\widehat{\mathbf{g}}_\ell(\mathbf{Z}_i)\widehat{\mathbf{g}}_\ell(\mathbf{Z}_i)^\top,
	\mathbf{T}_i
	:=
	\mathbf{Y}_i-\overline{\boldsymbol{\Sigma}}_{g,\ell}.
	\)
	The uniform gradient bound gives
	\(
	\mathbf{0}\preceq
	\mathbf{Y}_i\preceq\dfrac{4q^2}{\rho^2}\mathbf{I}
	\)
	and
	\(
	\mathbf{Y}_i^2
	=\|\widehat{\mathbf{g}}_\ell(\mathbf{Z}_i)\|_2^2\mathbf{Y}_i
	\preceq
	\dfrac{4q^2}{\rho^2}\mathbf{Y}_i.
	\)
	Thus
	\(
	\mathbb E \mathbf{T}_i=\mathbf{0},
\|\mathbf{T}_i\|_2\le\dfrac{4q^2}{\rho^2},
	\)
	and
	\(
	\mathbb E[\mathbf{T}_i^2]
	=
	\mathbb E[\mathbf{Y}_i^2]
	-
	\overline{\boldsymbol{\Sigma}}_{g,\ell}^{\,2} 
	\preceq
	\mathbb E[\mathbf{Y}_i^2] 
	\preceq
	\dfrac{4q^2}{\rho^2}\overline{\boldsymbol{\Sigma}}_{g,\ell}.
	\)
	Consequently,
	\[
	\left\|
	\sum_{i=1}^m\mathbb E[\mathbf{T}_i^2]
	\right\|_2
	\le
	\dfrac{4q^2m}{\rho^2}
	\|\overline{\boldsymbol{\Sigma}}_{g,\ell}\|_2.
	\]
	Matrix Bernstein now yields
	\[
	\begin{aligned}
	\mathbb P\left(
	\|\widehat{\boldsymbol{\Sigma}}_{g,\ell}-\overline{\boldsymbol{\Sigma}}_{g,\ell}\|_2
	\ge t
	\right) 
	\le
	2n\exp\left(
	-
	\dfrac{
	mt^2
	}{
	C_q\rho^{-2}
	(\|\overline{\boldsymbol{\Sigma}}_{g,\ell}\|_2+t)
	}
	\right).
	\end{aligned}
	\]
	Taking
	\(
	t
	:=
	\dfrac\eta2\|\boldsymbol{\Sigma}_{g,\ell}\|_2
	\)
	gives
	\[
	\begin{aligned}
	\mathbb P\left(
	\|\widehat{\boldsymbol{\Sigma}}_{g,\ell}-\overline{\boldsymbol{\Sigma}}_{g,\ell}\|_2
	\ge
	\dfrac\eta2\|\boldsymbol{\Sigma}_{g,\ell}\|_2
	\right) 
	\le
	2n\exp\left(
	-c_qm\eta^2\rho^2\|\boldsymbol{\Sigma}_{g,\ell}\|_2
	\right).
	\end{aligned}
	\]
	By Lemma~\ref{lem:population-moment-scales},
	\(
	\|\boldsymbol{\Sigma}_{g,\ell}\|_2
	\ge
	\dfrac{q^2}{R^2(n+2q-2)}.
	\)
	Therefore
	\[
	\dfrac1{\rho^2\|\boldsymbol{\Sigma}_{g,\ell}\|_2}
	\le
	\dfrac{R^2(n+2q-2)}{\rho^2q^2}
	\le
	C_q(n+2q)\omega^2.
	\]
	The same sample-size condition thus makes the gradient
	fluctuation at most
	\(\dfrac\eta2\|\boldsymbol{\Sigma}_{g,\ell}\|_2\) with
	failure probability at most
	\(\zeta/2\).

	Combining the two concentration events with the two
	population-bias bounds
	and applying a union bound gives, with probability at
	least \(1-\zeta\),
	\[
	\|\widehat{\boldsymbol{\Sigma}}_{x,\ell}-\boldsymbol{\Sigma}_{x,\ell}\|_2
	\le
	\eta\|\boldsymbol{\Sigma}_{x,\ell}\|_2,\qquad \|\widehat{\boldsymbol{\Sigma}}_{g,\ell}-\boldsymbol{\Sigma}_{g,\ell}\|_2
	\le
	\eta\|\boldsymbol{\Sigma}_{g,\ell}\|_2.
	\]

	Finally, Lemma~\ref{lem:sampling-body-radii} gives
	\(
	\omega^2
	\le
	C_{k,L}\kappa^8r.
	\)
	Since \(n\le r\) and \(q=O_{k,L}(1)\),
	\(
	(n+2q)\omega^2
	\le
	C_{k,L}\kappa^8r^2
	\)
	and
	\(
	\sqrt n\,(n+2q)\omega^2
	\le
	C_{k,L}\kappa^8r^{5/2}.
	\)
	Substituting these two inequalities into
	\eqref{eq:empirical-moment-noise-conditions} and
	\eqref{eq:empirical-moment-sample-size} gives the three
	sufficient
	conditions stated at the end of the lemma.
\end{proof}

\section{Recovery of the First Hidden Layer}
This section establishes recovery of the first hidden layer, from identifying its column space to recovering the individual weight directions.
	\label{sec:appendix-first-layer}
	\begin{proposition}[Exact recovery of the first-layer column space]
		\label{lem:first-layer-column-space}
		Suppose that \(\mathbf{W}_1\in\mathbb R^{d\times r}\)
		has full column rank and
		\[
		f^\star(\mathbf{x})=h^{(1)}(\mathbf{W}_1^\top
		\mathbf{x}),
		\]
		where \(h^{(1)}:\mathbb R^{r}\to\mathbb R\) is a
		positive definite
		homogeneous polynomial. Assume exact function-value
		queries to \(f^\star\),
		so that \(\nabla f^\star(\mathbf{x})\) can be recovered
		exactly by polynomial
		interpolation.
		
		Let \(\mathbf{x}_1,\ldots,\mathbf{x}_m\), \(m\ge r\), be
		independently sampled from an
		absolutely continuous distribution whose support
		contains a nonempty open set.
		Define
		\[
		\mathbf{g}_i:=\nabla f^\star(\mathbf{x}_i),
		\qquad
		\boldsymbol{
		G}_1:=[\mathbf{g}_1,\ldots,\mathbf{g}_m].
		\]
		Then, with probability one,
		\(
		\operatorname{col}(\boldsymbol{
		G}_1)=\operatorname{col}(\mathbf{W}_1).
		\)
		Hence \(\operatorname{col}(\mathbf{W}_1)\) is recovered
		exactly from the top \(r\)
		left singular vectors of \(\boldsymbol{ G}_1\).
	\end{proposition}
	
	\begin{proof}
		By the chain rule,
		\(
		\mathbf{g}_i
		=
		\mathbf{W}_1\nabla h^{(1)}(\mathbf{W}_1^\top
		\mathbf{x}_i),
		\)
		and therefore
		\(
		\operatorname{col}(\boldsymbol{G}_1)\subseteq\operatorname{col}(\mathbf{W}_1).
		\)
		Since \(h^{(1)}\) is positive definite and homogeneous,
		\[
		\operatorname{span}
		\left\{
		\nabla h^{(1)}(\mathbf{u}):\mathbf{u}\in\mathbb R^{r}
		\right\}
		=
		\mathbb R^{r}.
		\]
		Indeed, otherwise there would exist
		\(\mathbf{c}\neq\mathbf{0}\) such that
		\(\mathbf{c}^\top\nabla h^{(1)}(\mathbf{u})\equiv0\).
		The function
		\(t\mapsto h^{(1)}(\mathbf{u}+t\mathbf{c})\) would then
		have zero derivative for every
		\(\mathbf{u}\) and \(t\), and hence would be constant.
		Taking \(\mathbf{u}=\mathbf{0}\) would give
		\(h^{(1)}(t\mathbf{c})=h^{(1)}(\mathbf{0})=0\) for every
		\(t\), contradicting positive
		definiteness.

		It follows that there exist
		\(\mathbf{u}_1,\ldots,\mathbf{u}_{r}\in\mathbb R^{r}\)
		for which
		\[
		\det\!\left[
		\nabla h^{(1)}(\mathbf{u}_1),\ldots,
		\nabla h^{(1)}(\mathbf{u}_{r})
		\right]
		\neq0.
		\]
		Define the polynomial on \((\mathbb R^d)^{r}\) by
		\[
		\Psi(\mathbf{x}_1,\ldots,\mathbf{x}_{r})
		:=
		\det\!\left[
		\nabla h^{(1)}(\mathbf{W}_1^\top \mathbf{x}_1),\ldots,
		\nabla h^{(1)}(\mathbf{W}_1^\top \mathbf{x}_{r})
		\right].
		\]
		Because \(\mathbf{W}_1^\top\) is surjective, for each
		\(\mathbf{u}_i\) one can choose
		\(\mathbf{x}_i\) satisfying \(\mathbf{W}_1^\top
		\mathbf{x}_i=\mathbf{u}_i\). Hence \(\Psi\) is not the
		zero
		polynomial. The zero set of a nonzero real polynomial
		has Lebesgue measure
		zero, while the joint law of
		\((\mathbf{x}_1,\ldots,\mathbf{x}_{r})\) is absolutely
		continuous. Therefore
		\(
		\mathbb
		P\bigl(\Psi(\mathbf{x}_1,\ldots,\mathbf{x}_{r})=0\bigr)=0.
		\)
		Thus the first \(r\) latent gradients are linearly
		independent almost
		surely. Since \(m\ge r\),
		\[
		\operatorname{rank}
		\left[
		\nabla h^{(1)}(\mathbf{W}_1^\top \mathbf{x}_1),\ldots,
		\nabla h^{(1)}(\mathbf{W}_1^\top \mathbf{x}_m)
		\right]
		=r
		\]
		almost surely, which yields
		\(
		\operatorname{col}(\boldsymbol{
		G}_1)=\operatorname{col}(\mathbf{W}_1).
		\)
	\end{proof}
	
	Thus, an orthonormal basis $\mathbf{B}_1$ of
	$\operatorname{col}(\mathbf{W}_1)$ can be recovered
	exactly.
	
	\begin{proposition}[Stable recovery of the first hidden layer]
	\label{prop:first-layer-recovery}
	Let \(d\ge2\) and \(\varepsilon,\zeta\in(0,1/2)\).
	Suppose that the
	standing assumptions hold. Then
	\begin{equation}
	\widetilde O_{k,L}\!\left(
	dr+
	\dfrac{\gamma^2r^{13}\kappa^{O(1)}}{\varepsilon^2}
	\right)
	\label{eq:first-layer-simplified-query}
	\end{equation}
	real function-value queries suffice to recover
	\(\widehat{\mathbf W}_1\)
	such that, with probability at least \(1-\zeta\), after
	suitably relabeling
	the recovered columns and choosing their signs,
	$
	\|\widehat{\mathbf W}_1-\mathbf W_1\|_2\le\varepsilon.
	$
\end{proposition}

\begin{proof}
	By Proposition~\ref{lem:first-layer-column-space},
	\(\operatorname{col}(\mathbf W_1)\) is recovered exactly
	with probability
	one. Let \(\mathbf B_1\in\mathbb R^{d\times r}\) be an
	orthonormal basis of
	this subspace and set
	\[
	\mathbf C_1:=\mathbf B_1^\top\mathbf W_1,
	\
	\mathbf C_1^\sharp:=\mathbf C_1,
	\
	G_1(\mathbf z):=f^\star(\mathbf B_1\mathbf z),
	\
	K_1^\sharp:=K_1:=\{\mathbf z:G_1(\mathbf z)\le1\}.
	\]
	Then
	\[
	\mathbf W_1=\mathbf B_1\mathbf C_1,
	\
	\operatorname{cond}(\mathbf C_1)
	=\operatorname{cond}(\mathbf W_1)
	\le\kappa_0\le\kappa,
	\
	\|\mathbf c_{1,i}\|_2=1.
	\]
	Moreover, Lemma~\ref{lem:sampling-body-radii} gives
	\[
	\omega_1
	:=
	\dfrac{R_{\mathrm{out},1}}{r_{\mathrm{in},1}}
	\le
	C_{k,L}\kappa^4\sqrt r.
	\]
	Let \(\mathbf Z\sim\operatorname{Unif}(K_1)\) and define
	\[
	\boldsymbol{\Sigma}_{x,1}:=\mathbb E[\mathbf Z\mathbf
	Z^\top],
	\qquad
	\boldsymbol{\Sigma}_{g,1}
	:=
	\mathbb E[\nabla G_1(\mathbf Z)\nabla G_1(\mathbf
	Z)^\top].
	\]
	Lemma~\ref{lem:population-moment-scales} yields
	\[
	\boldsymbol{\Sigma}_{x,1}
	=
	\mathbf C_1^{-\top}\boldsymbol{\Lambda}_1\mathbf
	C_1^{-1},
	\qquad
	\boldsymbol{\Sigma}_{g,1}
	=
	\mathbf C_1\mathbf D_1\mathbf C_1^\top,
	\]
	where \(\mathbf D_1,\boldsymbol{\Lambda}_1\) are
	positive diagonal
	matrices and
	\(
	\mathbf D_1\boldsymbol{\Lambda}_1
	=
	\operatorname{diag}(\theta_{1,1},\ldots,\theta_{1,r}).
	\)
	Furthermore,
	\(
	\rho_1
	:=\dfrac{\|\boldsymbol{\Sigma}_{x,1}\|_2
	\|\boldsymbol{\Sigma}_{g,1}\|_2}
	{\max_i|\theta_{1,i}|}
	\le
	C_{k,L}\kappa^8r^3.
	\)
	Set
	\(
	\eta
	:=
	c_{k,L}
	\dfrac{\varepsilon}{\gamma\kappa^{12}r^{7/2}},
	\nu
	:=
	c_{k,L}\kappa^{-8}\eta r^{-2}.
	\)
	Since \(G_1\) is available exactly, it satisfies
	\[
	(1-\xi)G_1(\mathbf z)
	\le
	G_1(\mathbf z)
	\le
	(1+\xi)G_1(\mathbf z)
	\]
	for every \(\xi\in(0,1/2]\). Let
	\(
	q_1:=k^{L-1},
	\)
	and apply Lemma~\ref{lem:noisy-hit-and-run} with
	\(\delta_{\rm HR}=1\) and target total-variation error \(\nu\).
	Let
	\(
	T_{\rm HR}
	:=
	\left\lceil
	C_{\rm HR}r^3\omega_1^2
	\log\dfrac{2\omega_1}{\nu}
	\right\rceil
	\)
	be the number of Hit-and-Run steps for each independent chain,
	where \(C_{\rm HR}>0\) is the absolute constant supplied by
	Lemma~\ref{lem:noisy-hit-and-run}, and set
	\(
	\eta_{\rm HR}:=\dfrac{\nu}{4T_{\rm HR}}.
	\)
	Define
	\[
	\xi_{\rm aux}
	:=
	\min\left\{
	\frac14,\,
	\frac{
	c_1\nu^2
	}{
	2q_1^2r^7\omega_1^6
	\left(
	\log\frac{2\omega_1}{\nu}
	\right)^2
	}
	\right\},
	\]
	where \(c_1>0\) is the constant supplied by
	Lemma~\ref{lem:noisy-hit-and-run} for \(\delta_{\rm HR}=1\).
	Then the multiplicative-oracle hypothesis of that lemma holds
	with \(\xi=\xi_{\rm aux}\). Moreover, if
	\[
	\tau_{\rm aux}
	:=
	1-
	\left(
	\frac{1-\xi_{\rm aux}}{1+\xi_{\rm aux}}
	\right)^{1/q_1},
	\]
	then, since \(0<\xi_{\rm aux}\le1/4\),
	\(
	\tau_{\rm aux}
	\ge
	c_{q_1}\xi_{\rm aux}
	\)
	for a constant \(c_{q_1}>0\) depending only on \(q_1\).
	Thus
	\(
	\log\dfrac{\omega_1}{\eta_{\rm HR}\tau_{\rm aux}}
	\)
	is logarithmic in
	\(r,\kappa,\gamma,\varepsilon^{-1},\zeta^{-1}\).
	Using
	\(
	\omega_1^2\le C_{k,L}\kappa^8r,
	\)
	Lemma~\ref{lem:noisy-hit-and-run} therefore generates each
	sample within total variation \(\nu\) of
	\(\operatorname{Unif}(K_1)\) using
	\(
	\widetilde O_{k,L}(\kappa^8r^4)
	\)
	function-value queries.

	Since the oracle is exact, the gradients \(\nabla G_1\)
	are recovered
	exactly by polynomial interpolation. Applying
	Lemma~\ref{lem:empirical-two-moments} with
	\(\widehat G_1=G_1\) and \(\xi_1=0\) shows that
	\[
	\begin{aligned}
	m
	=
	O_{k,L}\!\left(
	\kappa^8r^2\eta^{-2}
	\log\dfrac{2r}{\zeta}
	\right)=
	O_{k,L}\!\left(
	\dfrac{\gamma^2\kappa^{32}r^9}{\varepsilon^2}
	\log\dfrac{2r}{\zeta}
	\right)
	\end{aligned}
	\]
	samples suffice to obtain relative \(\eta\)-accurate
	estimates of both
	population moment matrices, with probability at least
	\(1-\zeta\).

	The normalized eigenvector matrix in
	Lemma~\ref{lem:noisy-two-matrix} is precisely \(\mathbf
	C_1\), because
	its columns have unit Euclidean norm, and hence its
	condition number is at
	most \(\kappa\). Assumption~\ref{ass:spectral-gap} and
	the choice of
	\(\eta\) therefore imply that, after relabeling and
	choosing signs,
	\[
	\max_i
	\|\widehat{\mathbf c}_{1,i}-\mathbf c_{1,i}\|_2
	\le
	C_{k,L}\gamma\kappa^{12}r^3\eta
	\le
	c_{k,L}\dfrac{\varepsilon}{\sqrt r}.
	\]
	Therefore
	\[
	\|\widehat{\mathbf C}_1-\mathbf C_1\|_2
	\le
	\|\widehat{\mathbf C}_1-\mathbf C_1\|_F
	\le
	\sqrt r
	\max_i\|\widehat{\mathbf c}_{1,i}-\mathbf c_{1,i}\|_2
	\le
	\varepsilon.
	\]
	Setting
	\(
	\widehat{\mathbf W}_1:=\mathbf B_1\widehat{\mathbf C}_1
	\)
	and using \(\mathbf B_1^\top\mathbf B_1=\mathbf I_r\)
	gives the claimed
	operator-norm error.

	Finally, exact recovery of the first-layer column space
	uses \(r\)
	gradients, each obtained with \(O_{k,L}(d)\)
	real
	function-value queries, for a total of \(O_{k,L}(dr)\).
	The sampling
	stage uses
	\[
	m\,\widetilde O_{k,L}(\kappa^8r^4)
	=
	\widetilde O_{k,L}\!\left(
	\dfrac{\gamma^2\kappa^{40}r^{13}}{\varepsilon^2}
	\log\dfrac1\zeta
	\right)
	\]
	queries. This proves the claim.
\end{proof}

\section{Stability of the Simulated Suffix Oracle}

\label{sec:appendix-suffix-stability}
This section quantifies how errors in the recovered prefix propagate through prefix inversion and the resulting simulated suffix oracle.
	\begin{lemma}[Stability of approximate prefix inversion]
		\label{lem:approximate-prefix-inversion}
		Fix \(2\le\ell\le L-1\). Suppose that, for every
		\(1\le j\le\ell-1\),
		\[
		\|\widehat{\mathbf W}_j-\mathbf W_j\|_2
		\le
		\varepsilon_{<\ell}.
		\]
		Define
		\[
		\widehat{\mathbf R}_j
		:=
		\widehat{\mathbf W}_j
		(\widehat{\mathbf W}_j^\top\widehat{\mathbf W}_j)^{-1}.
		\]
		For \(\mathbf y\in\mathbb R^{r_{\ell-1}}\), set
		\(\mathbf z_{\ell-1}=\mathbf y\), and for
		\(j=\ell-1,\ldots,1\), choose coordinatewise principal $k$-th root
		\(\boldsymbol\xi_j\in\mathbb C^{r_j}\) such
		that
		\(\boldsymbol\xi_j^{\odot k}=\mathbf z_j\), and set
		\(\mathbf z_{j-1}=\widehat{\mathbf
		R}_j\boldsymbol\xi_j\).
		Starting from \(\widetilde{\mathbf z}_0=\mathbf z_0\),
		define the true
		forward iterates
		\[
		\widetilde{\mathbf z}_j
		=
		(\mathbf W_j^\top\widetilde{\mathbf z}_{j-1})^{\odot k},
		\qquad
		1\le j\le\ell-1.
		\]
		Let
		\[
		\vartheta_k:=\dfrac{k-1}{2k},
		\qquad
		P_j
		:=
		\sqrt{\dfrac r{r_j}}
		\prod_{t=1}^j r_t^{\vartheta_k},
		\qquad
		1\le j\le\ell-1.
		\]
		There are constants \(c_{k,L},C_{k,L}>0\), independent
		of the
		conditioning parameters, such that if
		\[
		\kappa^{O(\ell)}P_{\ell-1}\varepsilon_{<\ell}
		\le c_{k,L},
		\]
		then
		\[
		\|\widetilde{\mathbf z}_{\ell-1}-\mathbf y\|_2
		\le
		C_{k,L}\kappa^{O(\ell)}P_{\ell-1}
		\varepsilon_{<\ell}\|\mathbf y\|_2.
		\]
		Since
		\(P_{\ell-1}\le r^{\ell/2}\), it is sufficient that
		\begin{equation}
		\kappa^{O(\ell)}r^{\ell/2}\varepsilon_{<\ell}
		\le c_{k,L},
		\label{eq:prefix-simplified-smallness}
		\end{equation}
		in which case
		\begin{equation}
		\|\widetilde{\mathbf z}_{\ell-1}-\mathbf y\|_2
		\le
		C_{k,L}\kappa^{O(\ell)}r^{\ell/2}
		\varepsilon_{<\ell}\|\mathbf y\|_2.
		\label{eq:prefix-simplified-error}
		\end{equation}
	\end{lemma}
	
	\begin{proof}
	If \(\mathbf y=\mathbf0\), choose every coordinatewise
	\(k\)-th root in
	the backward recursion to be zero. Then all backward and
	forward iterates are
	zero, and the conclusion is immediate. Hence assume
	\(\mathbf y\neq\mathbf0\).

	Lemma~\ref{lem:weight-singular-values} gives
	\[
	\sigma_{\min}(\mathbf W_1)\ge\kappa_0^{-1},
	\qquad
	\sigma_{\min}(\mathbf W_j)
	\ge
	\kappa_0^{-1}\sqrt{\dfrac{r_j}{r_{j-1}}},
	\quad 2\le j\le\ell-1,
	\]
	and \(\|\mathbf W_j\|_2\le\kappa_0\). Since the widths
	are
	nonincreasing,
	\[
	\sqrt{\dfrac{r_{j-1}}{r_j}}
	\le
	\sqrt{\dfrac r{r_{\ell-1}}}
	\le
	P_{\ell-1}.
	\]
	Thus the assumed smallness condition, after decreasing
	\(c_{k,L}\), makes
	\(\varepsilon_{<\ell}\le\sigma_{\min}(\mathbf W_j)/2\)
	for every
	\(j<\ell\). Weyl's inequality gives full column rank of
	each
	\(\widehat{\mathbf W}_j\), together with
	\[
	\|\widehat{\mathbf R}_1\|_2\le2\kappa_0,
	\qquad
	\|\widehat{\mathbf R}_j\|_2
	\le
	2\kappa_0\sqrt{\dfrac{r_{j-1}}{r_j}},
	\quad 2\le j\le\ell-1.
	\]
	In particular, the nonzero terminal value implies that
	every backward
	iterate is nonzero.

	Set
	\[
	\varepsilon:=\varepsilon_{<\ell},
	\qquad
	\mathbf e_j:=\widetilde{\mathbf z}_j-\mathbf z_j,
	\qquad
	E_j:=\dfrac{\|\mathbf e_j\|_2}{\|\mathbf z_j\|_2}.
	\]
	Then \(E_0=0\). Define the layer factors
	\[
	b_1:=r^{\vartheta_k},
	\qquad
	b_j:=\sqrt{\dfrac{r_{j-1}}{r_j}}\,r_j^{\vartheta_k},
	\quad j\ge2.
	\]
	Their products telescope:
	\[
	\prod_{t=1}^j b_t
	=
	\sqrt{\dfrac r{r_j}}
	\prod_{t=1}^j r_t^{\vartheta_k}
	=
	P_j.
	\]

	Because
	\(
	\widehat{\mathbf W}_j^\top\mathbf
	z_{j-1}=\boldsymbol\xi_j,
	\)
	we may write
	\[
	\mathbf W_j^\top\widetilde{\mathbf z}_{j-1}
	=
	\boldsymbol\xi_j+\boldsymbol\Delta_j,
	\]
	where
	\[
	\boldsymbol\Delta_j
	:=
	(\mathbf W_j-\widehat{\mathbf W}_j)^\top\mathbf z_{j-1}
	+
	\mathbf W_j^\top\mathbf e_{j-1}.
	\]
	Consequently,
	\[
	\|\boldsymbol\Delta_j\|_2
	\le
	(\varepsilon+\kappa E_{j-1})\|\mathbf z_{j-1}\|_2.
	\]
	The right-inverse bounds and the comparison between
	\(\ell_2\) and
	\(\ell_{2k}\) give
	\[
	\|\mathbf z_{j-1}\|_2
	\le
	C\kappa b_j\|\boldsymbol\xi_j\|_{2k}.
	\]
	Thus
	\[
	\|\boldsymbol\Delta_j\|_{2k}
	\le
	C\kappa b_j
	(\varepsilon+\kappa E_{j-1})
	\|\boldsymbol\xi_j\|_{2k}.
	\]
	For \(a,b\in\mathbb C\),
\[
|(a+b)^k-a^k|
\le
k|b|(|a|+|b|)^{k-1}.
\]
Applying this coordinatewise and using Hölder's inequality gives
\[
\left\|
(\boldsymbol\xi_j+\boldsymbol\Delta_j)^{\odot k}
-
\boldsymbol\xi_j^{\odot k}
\right\|_2
\le
k\|\boldsymbol\Delta_j\|_{2k}
\left(
\|\boldsymbol\xi_j\|_{2k}
+
\|\boldsymbol\Delta_j\|_{2k}
\right)^{k-1}.
\]
Set
\[
a_j
:=
C\kappa b_j(\varepsilon+\kappa E_{j-1}).
\]
The preceding estimate gives
\[
\|\boldsymbol\Delta_j\|_{2k}
\le
a_j\|\boldsymbol\xi_j\|_{2k}.
\]
Whenever \(a_j\le1\),
\[
\begin{aligned}
\|\mathbf e_j\|_2
=
\left\|
(\boldsymbol\xi_j+\boldsymbol\Delta_j)^{\odot k}
-
\boldsymbol\xi_j^{\odot k}
\right\|_2
\le
k a_j(1+a_j)^{k-1}
\|\boldsymbol\xi_j\|_{2k}^k
\le
C_k\kappa b_j
(\varepsilon+\kappa E_{j-1})
\|\boldsymbol\xi_j\|_{2k}^k.
\end{aligned}
\]
	Since
	\(
	\|\mathbf z_j\|_2=\|\boldsymbol\xi_j\|_{2k}^k,
	\)
	we obtain
	\begin{equation}
	E_j
	\le
	C_k\kappa b_j(\varepsilon+\kappa E_{j-1}).
	\label{eq:prefix-error-recursion}
	\end{equation}

	We verify the temporary smallness condition and solve the
	recursion
	simultaneously. We claim that constants
	\(A_0,\ldots,A_{\ell-1}\), depending
	only on \(k,L\), can be chosen so that
	\begin{equation}
	E_j
	\le
	A_j\kappa^{3j}P_j\varepsilon,
	\qquad
	0\le j\le\ell-1.
	\label{eq:prefix-inductive-bound}
	\end{equation}
	where \(P_0:=1\). The assertion is immediate for \(j=0\).
	If it holds
	through \(j-1\), then \(b_jP_{j-1}=P_j\),
	\(P_{j-1}\ge1\), and
	\[
	C\kappa b_j(\varepsilon+\kappa E_{j-1})
	\le
	C_{k,L}\kappa^{3j}P_j\varepsilon.
	\]
	The assumed smallness condition makes the last quantity
	at most one,
	validating \eqref{eq:prefix-error-recursion}. Applying
	that recursion once
	more yields the claimed bound at \(j\).

	Taking \(j=\ell-1\) proves the profile-dependent
	estimate. Finally,
	because \(r_j\le r\), \(r_{\ell-1}\ge1\), and
	\(\vartheta_k=(k-1)/(2k)<1/2\),
	\[
	P_{\ell-1}
	=
	\sqrt{\dfrac r{r_{\ell-1}}}
	\prod_{j=1}^{\ell-1}r_j^{\vartheta_k}
	\le
	r^{1/2+\vartheta_k(\ell-1)}
	\le
	r^{\ell/2}.
	\]
	This proves the simplified statements.
	\end{proof}
	
\begin{lemma}[Relative perturbation of a homogeneous polynomial]
		\label{lem:relative-homogeneous-perturbation}
		Let \(G:\mathbb R^n\to\mathbb R\) be a
		positive-definite, even,
		homogeneous polynomial of degree \(q\), and suppose that
		\(p(\mathbf{x}):=G(\mathbf{x})^{1/q}
		\)
		is a norm on \(\mathbb R^n\).
		Let \(G_{\mathbb C}\) denote the complexification of
		\(G\), and define
		\[
		p_{\mathbb
		C}(\mathbf{u}+i\mathbf{v}):=p(\mathbf{u})+p(\mathbf{v}),
		\qquad
		\mathbf{u},\mathbf{v}\in\mathbb R^n.
		\]
		Then there exists a constant \(C_q>0\), depending only
		on \(q\),
		such that for every \(\mathbf{z}\in\mathbb
		R^n\setminus\{\mathbf{0}\}\) and every
		\(\boldsymbol{\Delta}\in\mathbb C^n\) satisfying
		\(
		p_{\mathbb C}(\boldsymbol{\Delta})\le p(\mathbf{z}),
		\)
		one has
		\[
		\dfrac{
			\left|G_{\mathbb
			C}(\mathbf{z}+\boldsymbol{\Delta})-G(\mathbf{z})\right|
		}{
			G(\mathbf{z})
		}
		\le
		C_q
		\dfrac{p_{\mathbb
		C}(\boldsymbol{\Delta})}{p(\mathbf{z})}.
		\]
	\end{lemma}
	
	\begin{proof}
		Let
		\(
		A:(\mathbb R^n)^q\to\mathbb R
		\)
		be the symmetric \(q\)-linear form associated with
		\(G\), so that
		\(
		G(\mathbf{x})=A(\mathbf{x},\ldots,\mathbf{x}).
		\)
		For \(\mathbf{x}_1,\ldots,\mathbf{x}_q\in\mathbb R^n\),
		the real polarization identity is
		\begin{equation}
		A(\mathbf{x}_1,\ldots,\mathbf{x}_q)
		=
		\dfrac{1}{2^q q!}
		\sum_{\varepsilon\in\{-1,1\}^q}
		\left(\prod_{j=1}^q\varepsilon_j\right)
		G\!\left(\sum_{j=1}^q\varepsilon_j\mathbf{x}_j\right).
		\label{eq:real-polarization-identity}
		\end{equation}
		Suppose first that \(p(\mathbf{x}_j)\le1\) for every
		\(j\). The triangle
		inequality for \(p\) gives
		\(
		p\!\left(\sum_{j=1}^q\varepsilon_j\mathbf{x}_j\right)
		\le q,
		\)
		and hence
		\(
		\left|G\!\left(\sum_{j=1}^q\varepsilon_j\mathbf{x}_j\right)\right|
		\le q^q.
		\)
		Applying this bound to
		\eqref{eq:real-polarization-identity} yields
		\(
		|A(\mathbf{x}_1,\ldots,\mathbf{x}_q)|
		\le
		\dfrac{q^q}{q!}.
		\)
		If some \(\mathbf{x}_j=\mathbf{0}\), the desired bound
		is immediate. Otherwise, apply the
		preceding estimate to \(\mathbf{x}_j/p(\mathbf{x}_j)\)
		for each \(j\) and use
		multilinearity. Thus, for arbitrary
		\(\mathbf{x}_1,\ldots,\mathbf{x}_q\),
		\begin{equation}
		|A(\mathbf{x}_1,\ldots,\mathbf{x}_q)|
		\le
		C_q^{(0)}
		\prod_{j=1}^q p(\mathbf{x}_j),
		\qquad
		C_q^{(0)}:=\dfrac{q^q}{q!}.
		\label{eq:polarization-multilinear-bound}
		\end{equation}

		Let \(A_{\mathbb C}\) be the complex multilinear
		extension of \(A\).
		Write \(\mathbf{w}_j=\mathbf{u}_j+i\mathbf{v}_j\), and
		set
		\(
		\mathbf{x}_{j,0}:=\mathbf{u}_j,
		\qquad
		\mathbf{x}_{j,1}:=\mathbf{v}_j.
		\)
		Multilinearity and
		\eqref{eq:polarization-multilinear-bound} give
		\begin{align*}
		|A_{\mathbb C}(\mathbf{w}_1,\ldots,\mathbf{w}_q)|
		\le
		C_q^{(0)}
		\sum_{\alpha\in\{0,1\}^q}
		\prod_{j=1}^q p(\mathbf{x}_{j,\alpha_j})=
		C_q^{(0)}\prod_{j=1}^q\bigl(p(\mathbf{u}_j)+p(\mathbf{v}_j)\bigr)
		=
		C_q^{(0)}
		\prod_{j=1}^q p_{\mathbb C}(\mathbf{w}_j).
		\end{align*}
		Therefore,
		\[
		\begin{aligned}
			\left|G_{\mathbb
			C}(\mathbf{z}+\boldsymbol{\Delta})-G(\mathbf{z})\right|
			&=
			\left|
			\sum_{j=1}^q
			\binom qj
			A_{\mathbb C}
			\bigl(
			\mathbf{z}^{\otimes(q-j)},
			\boldsymbol{\Delta}^{\otimes j}
			\bigr)
			\right| \le
			C_q^{(0)}
			\sum_{j=1}^q
			\binom qj
			p(\mathbf{z})^{q-j}
			p_{\mathbb C}(\boldsymbol{\Delta})^j.
		\end{aligned}
		\]
		Set
		\(
		t:=
		\dfrac{p_{\mathbb
		C}(\boldsymbol{\Delta})}{p(\mathbf{z})}.
		\)
		Since \(t\le1\),
		\(
		(1+t)^q-1
		\le
		(2^q-1)t.
		\)
		Hence
		\[
		\left|G_{\mathbb
		C}(\mathbf{z}+\boldsymbol{\Delta})-G(\mathbf{z})\right|
		\le
		C_q
		p(\mathbf{z})^q
		\dfrac{p_{\mathbb
		C}(\boldsymbol{\Delta})}{p(\mathbf{z})}.
		\]
		Finally, \(p(\mathbf{z})^q=G(\mathbf{z})\), which proves
		the claim.
	\end{proof}
	
	\begin{lemma}[Multiplicative error of the simulated suffix oracle]
		\label{lem:multiplicative-suffix-oracle}
		Fix \(2\le\ell\le L-2\), and suppose
		\(\widehat{\mathbf B}_\ell\) satisfies the hypotheses of
		Lemma~\ref{lem:sampling-body-radii}. Let
		\[
		\varepsilon_{<\ell}
		:=
		\max_{1\le j\le\ell-1}
		\|\widehat{\mathbf W}_j-\mathbf W_j\|_2.
		\]
		For \(\mathbf z\in\mathbb R^{r_\ell}\), set
		\(\mathbf y=\widehat{\mathbf B}_\ell\mathbf z\),
		construct
		\(\mathbf z_0\in\mathbb C^d\) by approximate prefix
		inversion, and
		define
		\[
		\widehat G_\ell(\mathbf z)
		:=
		\operatorname{Re}f^\star_{\mathbb C}(\mathbf z_0).
		\]
		There are constants \(c_{k,L},C_{k,L}>0\) such that if
		\begin{equation}
		\kappa^{O(\ell)}r^{\ell/2}\varepsilon_{<\ell}
		\le c_{k,L},
		\label{eq:suffix-simplified-smallness}
		\end{equation}
		then
		\begin{equation}
		|\widehat G_\ell(\mathbf z)-G_\ell(\mathbf z)|
		\le
		C_{k,L}\kappa^{O(\ell)}r^{\ell/2}
		\varepsilon_{<\ell}G_\ell(\mathbf z),
		\qquad
		\mathbf z\in\mathbb R^{r_\ell}.
		\label{eq:suffix-simplified-error}
		\end{equation}
	\end{lemma}
	
	\begin{proof}
	Let
	\[
	\vartheta_k:=\dfrac{k-1}{2k},
	\qquad
	P_{\ell-1}
	:=
	\sqrt{\dfrac r{r_{\ell-1}}}
	\prod_{j=1}^{\ell-1}r_j^{\vartheta_k},
	\qquad
	\mathbf e_{\ell-1}
	:=
	\widetilde{\mathbf z}_{\ell-1}-\mathbf y.
	\]
	The profile-dependent form of
	Lemma~\ref{lem:approximate-prefix-inversion} gives
	\[
	\|\mathbf e_{\ell-1}\|_2
	\le
	C_{k,L}\kappa^{O(\ell)}P_{\ell-1}
	\varepsilon_{<\ell}\|\mathbf y\|_2.
	\]
	Define the induced perturbation in reduced coordinates by
	\[
	\boldsymbol\Delta_\ell
	:=
	(\mathbf C_\ell^\sharp)^{-\top}
	\mathbf W_\ell^\top\mathbf e_{\ell-1}.
	\]
	Then
	\(
	(\mathbf C_\ell^\sharp)^\top\boldsymbol\Delta_\ell
	=
	\mathbf W_\ell^\top\mathbf e_{\ell-1}.
	\)
	Since \(\mathbf y=\widehat{\mathbf B}_\ell\mathbf z\) and
	\(\mathbf C_\ell^\sharp=\widehat{\mathbf
	B}_\ell^\top\mathbf W_\ell\),
	\(
	(\mathbf C_\ell^\sharp)^\top\mathbf z
	=
	\mathbf W_\ell^\top\mathbf y.
	\)
	Consequently,
	\[
	(\mathbf C_\ell^\sharp)^\top
	(\mathbf z+\boldsymbol\Delta_\ell)
	=
	\mathbf W_\ell^\top(\mathbf y+\mathbf e_{\ell-1}).
	\]
	The true forward pass from the approximately inverted
	input reaches
	\(\mathbf y+\mathbf e_{\ell-1}\) immediately before layer
	\(\ell\), and
	hence
	\[
	f^\star_{\mathbb C}(\mathbf z_0)
	=
	G_{\ell,\mathbb C}(\mathbf z+\boldsymbol\Delta_\ell).
	\]

	To bound the reduced perturbation without introducing an
	artificial inverse
	aspect ratio, write
	\[
	\mathbf W_\ell=\mathbf B_\ell\mathbf C_\ell,
	\qquad
	\mathbf Q_\ell:=\widehat{\mathbf B}_\ell^\top\mathbf
	B_\ell.
	\]
	Then
	\(\mathbf C_\ell^\sharp=\mathbf Q_\ell\mathbf C_\ell\),
	and the exact
	identity
	\[
	(\mathbf C_\ell^\sharp)^{-\top}\mathbf W_\ell^\top
	=
	\mathbf Q_\ell^{-\top}\mathbf B_\ell^\top
	\]
	holds. Lemma~\ref{lem:reduced-weight-conditioning} gives
	\(\sigma_{\min}(\mathbf
	Q_\ell)\ge1-\varepsilon_{B,\ell}\ge1/2\),
	so
	\[
	\|\boldsymbol\Delta_\ell\|_2
	\le
	2\|\mathbf e_{\ell-1}\|_2
	\le
	C_{k,L}\kappa^{O(\ell)}P_{\ell-1}
	\varepsilon_{<\ell}\|\mathbf z\|_2.
	\]
	Here we used the orthonormality of \(\widehat{\mathbf
	B}_\ell\), which gives
	\(\|\mathbf y\|_2=\|\mathbf z\|_2\).

	If \(\mathbf z=\mathbf0\), choose the zero \(k\)-th root
	at every step of
	the approximate prefix inversion. Then both sides vanish
	exactly. Henceforth
	assume \(\mathbf z\neq\mathbf0\). Because \(G_\ell\) is
	positive definite,
	convex, even, and \(q_\ell\)-homogeneous,
	\[
	p_\ell(\mathbf z)
	:=
	G_\ell(\mathbf z)^{1/q_\ell}
	\]
	is a norm with unit ball \(K_\ell^\sharp\). The radius
	inclusions imply
	\[
	\dfrac{\|\mathbf z\|_2}{R_{\mathrm{out},\ell}}
	\le
	p_\ell(\mathbf z)
	\le
	\dfrac{\|\mathbf z\|_2}{r_{\mathrm{in},\ell}}.
	\]
	For \(\mathbf w=\mathbf u+i\mathbf v\), define
	\(
	p_{\ell,\mathbb C}(\mathbf w):=p_\ell(\mathbf
	u)+p_\ell(\mathbf v).
	\)
	Then
	\(
	\dfrac{p_{\ell,\mathbb C}(\boldsymbol\Delta_\ell)}
	{p_\ell(\mathbf z)}
	\le
	C_{k,L}\omega_\ell\kappa^{O(\ell)}P_{\ell-1}
	\varepsilon_{<\ell}.
	\)
	Lemma~\ref{lem:sampling-body-radii} and the monotonicity
	\(r_\ell\le r_{\ell-1}\) give
	\[
	\begin{aligned}
	\omega_\ell P_{\ell-1}
	\le
	C_{k,L}\kappa^4
	\sqrt{r_\ell}
	\sqrt{\dfrac r{r_{\ell-1}}}
	\prod_{j=1}^{\ell-1}r_j^{\vartheta_k}
	\le
	C_{k,L}\kappa^4
	\sqrt r\prod_{j=1}^{\ell-1}r_j^{\vartheta_k}
	\le C_{k,L}\kappa^4r^{1/2+\vartheta_k(\ell-1)}
	\le
	C_{k,L}\kappa^4r^{\ell/2}.
	\end{aligned}
	\]
	Thus
	\[
	\dfrac{p_{\ell,\mathbb C}(\boldsymbol\Delta_\ell)}
	{p_\ell(\mathbf z)}
	\le
	C_{k,L}\kappa^{O(\ell)}r^{\ell/2}
	\varepsilon_{<\ell}.
	\]
	Condition~\eqref{eq:suffix-simplified-smallness} makes
	this quantity at most
	one. Lemma~\ref{lem:relative-homogeneous-perturbation}
	yields
	\[
	\dfrac{
	|G_{\ell,\mathbb C}(\mathbf
	z+\boldsymbol\Delta_\ell)-G_\ell(\mathbf z)|
	}{
	G_\ell(\mathbf z)
	}
	\le
	C_{k,L}\kappa^{O(\ell)}r^{\ell/2}
	\varepsilon_{<\ell}.
	\]
	Finally,
	\(\widehat G_\ell(\mathbf
	z)=\operatorname{Re}f^\star_{\mathbb C}(\mathbf z_0)\)
	and \(G_\ell(\mathbf z)\) is real, so taking the real
	part cannot increase
	the absolute error. This proves
	\eqref{eq:suffix-simplified-error}.
	\end{proof}

\section{Recovery of Intermediate Hidden Layers}
\label{sec:appendix-intermediate-layers}
This section establishes the inductive recovery step for each intermediate hidden layer, combining column-space estimation with convex-body sampling and spectral diagonalization.
\subsection{Stable Recovery of the Intermediate-Layer Column Space}

We first show that suitably chosen query points and directional derivatives recover the column space of the next hidden weight matrix despite errors in the previously recovered prefix.

	\begin{proposition}[Stable recovery of the intermediate-layer column space]
		\label{prop:stable-column-space}
		Fix \(2\le\ell\le L-2\) and \(d\ge2\). Suppose that the
		standing
		assumptions hold and that, after the consistent
		relabeling inherited from
		the previous stages,
		\[
		\max_{1\le j<\ell}
		\|\widehat{\mathbf W}_j-\mathbf W_j\|_2
		\le\varepsilon_W.
		\]
		For any \(\varepsilon_B,\zeta\in(0,1/2)\), it is
		sufficient that
		\begin{equation}
		\varepsilon_W
		\le
		\dfrac{c_{k,L}\varepsilon_B}
		{\kappa^{O(L)}r^{7/2}}
		\label{eq:span-simplified-prefix-condition}
		\end{equation}
		and
		\begin{equation}
		m
		\ge
		C_{k,L}\kappa^{O(L)}
		\dfrac{r^6}{\varepsilon_B^2}
		\log\dfrac{2r}{\zeta}.
		\label{eq:span-simplified-sample-size}
		\end{equation}
		Under these conditions, one can construct an orthonormal
		matrix
		\(\widehat{\mathbf B}_\ell\) such that, with probability
		at least
		\(1-\zeta\), after orthogonal alignment,
		\[
		\|\widehat{\mathbf B}_\ell-\mathbf
		B_\ell\|_2\le\varepsilon_B.
		\]
		The total number of additional real function-value
		queries is at most
		\begin{equation}
		\widetilde O_{k,L}\!\left(
r^7\kappa^{O(L)}\varepsilon_B^{-2}
\right).
		\label{eq:span-simplified-query}
		\end{equation}
	\end{proposition}
	
\begin{proof}
		Throughout this proof, \( C_L\) and \(
		c_L\) denote
		generic positive quantities that may change from line to
		line and satisfy
		\[
		1\le C_L\le C_{k,L}\kappa^{O(L)},
		\qquad
		c_{k,L}\kappa^{-O(L)}\le c_L\le1.
		\]
		This convention records every use of \(\kappa_0\) and
		\(\mu^{-1}\)
		without suppressing their polynomial dependence.
        If \(r_\ell=r_{\ell-1}\), then \(\mathbf{W}_\ell\)
        is square and invertible by
        Assumption~\ref{ass:conditioning}, so
        \(
        \operatorname{col}(\mathbf{W}_\ell)=\mathbb
        R^{r_{\ell-1}}.
        \)
        Hence no subspace estimation is needed: we may
        simply take
        \(\widehat{\mathbf{B}}_\ell=\mathbf{I}_{r_{\ell-1}}\),
        and the asserted basis error is zero
        after orthogonal alignment. Thus, in the remainder
        of the proof, we may
        assume \(r_\ell<r_{\ell-1}\).
		Condition on the previously recovered matrices
		\(\widehat{\mathbf{W}}_1,\ldots,\widehat{\mathbf{W}}_{\ell-1}\)
		satisfying the assumed error
		bound, with the first-layer signs chosen as in the
		statement. All
		probabilities below are with respect to the fresh
		samples used in this
		stage.
		
		Let
		\[
		\mathcal I=[1,2],
		\qquad
		\mathbf{Z}_1,\ldots,\mathbf{Z}_m
		\overset{\mathrm{i.i.d.}}{\sim}
		\operatorname{Unif}(\mathcal I^{r}),
		\]
		and write
		\(
		\mathbf{T}_i:=\mathbf{Z}_i^{\odot 1/k}.
		\)
		Define
		\[
		\widehat{\mathbf{R}}_1
		:=
		\widehat{\mathbf{W}}_1
		(\widehat{\mathbf{W}}_1^\top\widehat{\mathbf{W}}_1)^{-1},
		\qquad
		\mathbf{x}_i:=\widehat{\mathbf{R}}_1\mathbf{T}_i.
		\]
		By Lemma~\ref{lem:weight-singular-values} and Weyl's
		inequality,
		\(\widehat{\mathbf{W}}_1\) has full column rank and
		\(
		\|\widehat{\mathbf{R}}_1\|_2\le  C_L
		\)
		whenever \(\varepsilon_W\) is sufficiently small. Hence
		\(
		\|\mathbf{x}_i\|_2
		\le
		 C_L\sqrt r.
		\)
		Moreover,
		\[
		\widehat{\mathbf{W}}_1^\top \mathbf{x}_i=\mathbf{T}_i,
		\qquad
		\sigma(\widehat{\mathbf{W}}_1^\top
		\mathbf{x}_i)=\mathbf{Z}_i.
		\]
		Thus the query points are obtained by inverting only the
		first layer.
		
		We first establish uniform regularity of the true prefix
		on these query
		points. Set
		\(
		\mathbf{A}:=\mathbf{W}_1^\top\widehat{\mathbf{R}}_1.
		\)
		Since
		\(
		\widehat{\mathbf{W}}_1^\top\widehat{\mathbf{R}}_1=\mathbf{I}_{r},
		\)
		we have
		\[
		\mathbf{A}-\mathbf{I}_{r}
		=
		(\mathbf{W}_1-\widehat{\mathbf{W}}_1)^\top\widehat{\mathbf{R}}_1,
		\]
		and therefore
		\(
		\|\mathbf{A}-\mathbf{I}_{r}\|_2
		\le
		 C_L\varepsilon_W.
		\)
		For \(\mathbf{T}\in[1,2^{1/k}]^{r}\),
		\(
		\|(\mathbf{A}-\mathbf{I}_{r})\mathbf{T}\|_\infty
		\le
		 C_L\sqrt r\,\varepsilon_W.
		\)
		Since the assumed upper bound on \(\varepsilon_W\)
		implies
		\(\sqrt r\,\varepsilon_W\le  c_L\), after
		decreasing
		\( c_L\) if necessary,
		\(
		\dfrac12
		\le
		(\mathbf{A}\mathbf{T})_s
		\le
		3,
		1\le s\le r.
		\)
		For \(\mathbf{T}\in[1,2^{1/k}]^{r}\), define the true
		forward maps
		\[
		\mathbf{H}_1(\mathbf{T}):=\mathbf{A}\mathbf{T},
		\qquad
		\boldsymbol{\Phi}_2(\mathbf{T}):=\mathbf{H}_1(\mathbf{T})^{\odot
		k},
		\]
		and recursively, for \(2\le j\le \ell-1\),
		\(\mathbf{H}_j(\mathbf{T}):=\mathbf{W}_j^\top\boldsymbol{\Phi}_j(\mathbf{T}),
		\boldsymbol{\Phi}_{j+1}(\mathbf{T}):=\mathbf{H}_j(\mathbf{T})^{\odot
		k}.
		\)
		Thus
		\(\boldsymbol{\Phi}_\ell(\mathbf{T})=\boldsymbol{\phi}_\ell(\widehat{\mathbf{R}}_1\mathbf{T}).
		\)
		Since every \(\mathbf{W}_j\), \(j\ge2\), is nonnegative
		and satisfies
		\(
		\mathbf{W}_j^\top\mathbf1=\mathbf1,
		\)
		each coordinate of \(\mathbf{W}_j^\top \mathbf{v}\) is a
		convex combination of the
		coordinates of \(\mathbf{v}\). Therefore, by induction,
		there exist constants
		\(0< c_L< C_L<\infty\) such that
		\(
		 c_L
		\le
		(\mathbf{H}_j(\mathbf{T}))_s
		\le
		 C_L
		\)
		for every \(1\le j\le \ell-1\), every admissible
		coordinate \(s\), and
		every \(\mathbf{T}\in[1,2^{1/k}]^{r}\).		Consequently, if
		\[
		\mathbf{D}_j(\mathbf{T})
		:=
		\operatorname{diag}
		\left(
		\mathbf{H}_j(\mathbf{T})^{\odot(k-1)}
		\right),
		\qquad
		1\le j\le\ell-1,
		\]
		then
		\(
		 c_L \mathbf{I}
		\preceq
		\mathbf{D}_j(\mathbf{T})
		\preceq
		 C_L \mathbf{I}.
		\)
		Furthermore,
		\[
		D\boldsymbol{\Phi}_\ell(\mathbf{T})
		=
		k^{\ell-1}
		\mathbf{D}_{\ell-1}(\mathbf{T})\mathbf{W}_{\ell-1}^\top
		\cdots
		\mathbf{D}_2(\mathbf{T})\mathbf{W}_2^\top
		\mathbf{D}_1(\mathbf{T})\mathbf{A}.
		\]
		By Lemma~\ref{lem:weight-singular-values},
		\(
		\sigma_{\min}(\mathbf W_j)
		\ge
		\dfrac1{\kappa_0}
		\sqrt{\dfrac{r_j}{r_{j-1}}},
		\|\mathbf W_j\|_2\le\kappa_0,
		2\le j\le\ell-1,
		\)
		and \(\mathbf A\) has all its singular values in a fixed
		neighborhood of
		one. For compatible full-row-rank matrices
		\(\mathbf P\in\mathbb R^{m\times n}\) and
		\(\mathbf Q\in\mathbb R^{n\times p}\),
		\(
		\sigma_m(\mathbf P\mathbf Q)
		\ge
		\sigma_m(\mathbf P)\sigma_n(\mathbf Q),
		\)
		since
		\[
		\mathbf P\mathbf Q\mathbf Q^\top\mathbf P^\top
		\succeq
		\sigma_n(\mathbf Q)^2\mathbf P\mathbf P^\top.
		\]
		Applying this inequality successively to the preceding
		product, and using
		the telescoping product of aspect ratios, gives
		\[
		\sigma_{r_{\ell-1}}
		\left(D\boldsymbol\Phi_\ell(\mathbf T)\right)
		\ge
		 c_L
		\prod_{j=2}^{\ell-1}
		\sqrt{\dfrac{r_j}{r_{j-1}}}
		=
		 c_L\sqrt{\dfrac{r_{\ell-1}}r}.
		\]
		
		Define
		\(
		\mathbf U_\ell(\mathbf T)
		:=
		\mathbf W_\ell^\top\boldsymbol\Phi_\ell(\mathbf T).
		\)
		Then
		\(
		D\mathbf U_\ell(\mathbf T)
		=
		\mathbf W_\ell^\top D\boldsymbol\Phi_\ell(\mathbf T),
		\)
		and another application of the same product inequality
		yields
		\[
		\sigma_{r_\ell}
		\left(D\mathbf U_\ell(\mathbf T)\right)
		\ge
		 c_L
		\sqrt{\dfrac{r_\ell}{r_{\ell-1}}}
		\sqrt{\dfrac{r_{\ell-1}}r}
		=
		 c_L\sqrt{\dfrac{r_\ell}r}.
		\]
		
We also need uniform positivity of the hidden states
appearing in the
		suffix network. For \(\ell\le j\le L-2\), define
		recursively
		\[
		\mathbf{U}_{j+1}(\mathbf{T})
		:=
		\mathbf{W}_{j+1}^\top \mathbf{U}_j(\mathbf{T})^{\odot
		k}.
		\]
		The bound on \(\boldsymbol{\Phi}_\ell(\mathbf{T})\),
		together with the fact that
		\(\mathbf{W}_\ell^\top\mathbf1=\mathbf1\), first gives
		\(
		 c_L
		\le
		(\mathbf{U}_\ell(\mathbf{T}))_s
		\le
		 C_L.
		\)
		If the same bounds hold for
		\(\mathbf{U}_j(\mathbf{T})\), then each coordinate of
		\(\mathbf{U}_{j+1}(\mathbf{T})\) is a convex combination
		of the coordinates of
		\(\mathbf{U}_j(\mathbf{T})^{\odot k}\). Hence, after
		adjusting the constants depending only
		on \(k,L\),
		\(
		 c_L
		\le
		(\mathbf{U}_j(\mathbf{T}))_s
		\le
		 C_L,
		\qquad
		\ell\le j\le L-1,
		\)
		uniformly in \(\mathbf{T}\in[1,2^{1/k}]^{r}\).
		
		We require a uniform lower bound for the Hessian of the
		suffix network.
		Recall that
		\(
		F_\ell(\mathbf y)
		=
		h^{(\ell)}(\mathbf W_\ell^\top\mathbf y).
		\)
		We claim that, for every \(\ell\le j\le L-1\),
		\(
		\nabla h^{(j)}(\mathbf U_j(\mathbf T))
		\ge
		\dfrac{ c_L}{r_j}\mathbf1
		\)
		coordinatewise and
		\(
		\nabla^2h^{(j)}(\mathbf U_j(\mathbf T))
		\succeq
		\dfrac{ c_L}{r_j}\mathbf I
		\)
		uniformly in \(\mathbf T\). For \(j=L-1\),
		\(
		h^{(L-1)}(\mathbf u)
		=
		\mathbf a^\top\mathbf u^{\odot k},
		\)
		so
		\(
		\nabla h^{(L-1)}(\mathbf u)
		=
		k\,\mathbf a\odot\mathbf u^{\odot(k-1)},
		\)
		and
		\(
		\nabla^2h^{(L-1)}(\mathbf u)
		=
		k(k-1)\operatorname{diag}
		\left(\mathbf a\odot\mathbf u^{\odot(k-2)}\right).
		\)
		The coordinates of \(\mathbf U_{L-1}(\mathbf T)\) are
		uniformly bounded
		below, and Assumption~\ref{ass:conditioning} at
		\(\ell=L\) gives
		\(a_i\ge\mu/r_{L-1}\), proving the claim at \(j=L-1\).

		Suppose it holds at \(j+1\). Write
		\[
		\mathbf u:=\mathbf U_j(\mathbf T),
		\qquad
		\mathbf v:=\mathbf U_{j+1}(\mathbf T)
		=
		\mathbf W_{j+1}^\top\mathbf u^{\odot k},
		\qquad
		\mathbf D_{\mathbf u}
		:=
		\operatorname{diag}(\mathbf u^{\odot(k-1)}).
		\]
		By the chain rule,
		\[
		\nabla h^{(j)}(\mathbf u)
		=
		k\mathbf D_{\mathbf u}\mathbf W_{j+1}
		\nabla h^{(j+1)}(\mathbf v).
		\]
		Since
		\(
		\mathbf W_{j+1}\mathbf1_{r_{j+1}}
		\ge
		\mu\dfrac{r_{j+1}}{r_j}\mathbf1_{r_j},
		\)
		the induction hypothesis gives
		\(
		\mathbf W_{j+1}\nabla h^{(j+1)}(\mathbf v)
		\ge
		\dfrac{ c_L}{r_j}\mathbf1.
		\)
		Together with the uniform positive lower bound on
		\(\mathbf u\), this
		proves the gradient claim. A second differentiation
		gives
		\[
		\begin{aligned}
		\nabla^2h^{(j)}(\mathbf u)
		=
		k^2\mathbf D_{\mathbf u}\mathbf W_{j+1}
		\nabla^2h^{(j+1)}(\mathbf v)
		\mathbf W_{j+1}^\top\mathbf D_{\mathbf u}+
		k(k-1)\operatorname{diag}\!\left(
		\mathbf u^{\odot(k-2)}\odot
		\mathbf W_{j+1}\nabla h^{(j+1)}(\mathbf v)
		\right).
		\end{aligned}
		\]
		The first term is positive semidefinite, while the
		second is bounded below
		by \( c_L r_j^{-1}\mathbf I\), proving the
		claim.

		Define
		\(
		\mathbf q_\ell(\mathbf T)
		:=
		\nabla h^{(\ell)}(\mathbf U_\ell(\mathbf T))
		\in\mathbb R^{r_\ell}.
		\)
		By the chain rule,
		\(
		D\mathbf q_\ell(\mathbf T)
		=
		\nabla^2h^{(\ell)}(\mathbf U_\ell(\mathbf T))
		D\mathbf U_\ell(\mathbf T).
		\)
		The preceding Hessian and Jacobian lower bounds imply
		\[
		\sigma_{r_\ell}\!\left(D\mathbf q_\ell(\mathbf T)\right)
		\ge
		\dfrac{ c_L}{r_\ell}
		\sqrt{\dfrac{r_\ell}r}
		=
		\dfrac{ c_L}{\sqrt{r r_\ell}},
		\]
		uniformly over \(\mathbf T\in[1,2^{1/k}]^r\). This pointwise differential lower bound yields a
population
		second-moment lower bound. Let
		\(\mathbf{T}=\mathbf{Z}^{\odot1/k}\), where
		\(
		\mathbf{Z}\sim\operatorname{Unif}([1,2]^{r}).
		\)
		The coordinates of \(\mathbf{T}\) are independent and
		have density
		\(
		\rho(t)=kt^{k-1},
		1\le t\le2^{1/k}.
		\)
		In particular, \(\rho\) is bounded above and below by
		positive constants
		depending only on \(k\). Let \(D_{k,L}\) be an upper bound, depending only on
		\(k,L\), for the
		degree of each coordinate of
		\(\mathbf{q}_\ell(\mathbf{T})\) as a polynomial in
		\(\mathbf{T}\).
		For example, one may take
		\(
		D_{k,L}=k^{L-1}.
		\)
		Set
		\(
		I_k:=[1,2^{1/k}],
		h_k:=2^{1/k}-1,
		\)
		and define the explicit constant
		\[
		C_{\mathrm{inv}}(k,L)
		:=
		\dfrac{
		2^{(k-1)/k}
		D_{k,L}(D_{k,L}+1)^2(D_{k,L}+2)
		}{
		(2^{1/k}-1)^2
		}.
		\]
		With the concrete choice \(D_{k,L}=k^{L-1}\), this
		becomes
		\[
		C_{\mathrm{inv}}(k,L)
		=
		\dfrac{
		2^{(k-1)/k}
		k^{L-1}(k^{L-1}+1)^2(k^{L-1}+2)
		}{
		(2^{1/k}-1)^2
		}.
		\]
		We claim that every univariate polynomial \(p\) of
		degree at most
		\(D_{k,L}\) satisfies
		\[
		\int_{I_k}
		|p'(t)|^2\rho(t)\,dt
		\le
		C_{\mathrm{inv}}(k,L)
		\int_{I_k}
		\left|
		p(t)-\int_{I_k}p(s)\rho(s)\,ds
		\right|^2
		\rho(t)\,dt.
		\]
		To prove this, let
		\[
		q(t)
		:=
		p(t)-\int_{I_k}p(s)\rho(s)\,ds.
		\]
		Then \(q'=p'\) and \(\deg q\le D_{k,L}\). Make the
		affine change of
		variables
		\[
		t
		=
		1+\dfrac{h_k}{2}(x+1),
		\qquad
		Q(x):=q(t),
		\qquad
		-1\le x\le1.
		\]
		Write \(D:=D_{k,L}\), let \(P_n\) denote the \(n\)-th
		Legendre
		polynomial, and set
		\[
		\varphi_n(x)
		:=
		\sqrt{\dfrac{2n+1}{2}}\,P_n(x),
		\qquad
		0\le n\le D.
		\]
		Then \(\{\varphi_0,\ldots,\varphi_D\}\) is an
		orthonormal basis for the
		polynomials of degree at most \(D\) in \(L^2([-1,1])\).
		Write
		\(
		Q=\sum_{n=0}^D a_n\varphi_n.
		\)
		The Legendre identity
		\[
		P_n'(x)-P_{n-2}'(x)
		=
		(2n-1)P_{n-1}(x),
		\qquad n\ge2,
		\]
		together with \(P_1'=P_0\), gives by iteration
		\[
		P_n'
		=
		\sum_{\substack{0\le j\le n-1\\ n-j\ {\rm odd}}}
		(2j+1)P_j.
		\]
		Using the orthogonality relation
		\(
		\int_{-1}^1P_j(x)^2\,dx
		=
		\dfrac{2}{2j+1},
		\)
		we obtain
		\[
		\|P_n'\|_{L^2(-1,1)}^2
		=
		n(n+1),
		\qquad
		\|\varphi_n'\|_{L^2(-1,1)}^2
		=
		\dfrac{(2n+1)n(n+1)}{2}.
		\]
		Hence, by the triangle inequality followed by
		Cauchy--Schwarz,
		\[
		\begin{aligned}
		\|Q'\|_{L^2(-1,1)}^2
		&\le
		\left(
		\sum_{n=0}^D
		|a_n|
		\|\varphi_n'\|_{L^2(-1,1)}
		\right)^2
		\\
		&\le
		\left(\sum_{n=0}^D a_n^2\right)
		\left(
		\sum_{n=0}^D
		\|\varphi_n'\|_{L^2(-1,1)}^2
		\right)
		\\
		&=
		\dfrac{D(D+1)^2(D+2)}{4}
		\|Q\|_{L^2(-1,1)}^2.
		\end{aligned}
		\]
		Rescaling from \([-1,1]\) back to \(I_k\) therefore
		gives
		\[
		\int_{I_k}|q'(t)|^2\,dt
		\le
		\dfrac{
		D(D+1)^2(D+2)
		}{
		h_k^2
		}
		\int_{I_k}|q(t)|^2\,dt.
		\]
		Finally, on \(I_k\),
		\(
		k
		\le
		\rho(t)=kt^{k-1}
		\le
		k\,2^{(k-1)/k}.
		\)
		Consequently,
		\[
		\begin{aligned}
		\int_{I_k}|q'(t)|^2\rho(t)\,dt
		&\le
		k\,2^{(k-1)/k}
		\int_{I_k}|q'(t)|^2\,dt
		\\
		&\le
		C_{\mathrm{inv}}(k,L)
		k\int_{I_k}|q(t)|^2\,dt
		\\
		&\le
		C_{\mathrm{inv}}(k,L)
		\int_{I_k}|q(t)|^2\rho(t)\,dt,
		\end{aligned}
		\]
		which proves the claimed one-dimensional inverse
		inequality.
		
		Fix \(\mathbf v\in\mathbb R^{r_\ell}\) with
		\(\|\mathbf v\|_2=1\), and set
		\(
		p_{\mathbf v}(\mathbf T):=\mathbf v^\top\mathbf
		q_\ell(\mathbf T).
		\)
		Then
		\[
		\nabla p_{\mathbf v}(\mathbf T)
		=
		D\mathbf q_\ell(\mathbf T)^\top\mathbf v,
		\qquad
		\|\nabla p_{\mathbf v}(\mathbf T)\|_2
		\ge
		\dfrac{ c_L}{\sqrt{r r_\ell}}.
		\]
		For each \(1\le s\le r\), condition on all coordinates
		of \(\mathbf T\)
		except \(T_s\) and apply the preceding one-dimensional
		inequality to the
		polynomial \(T_s\mapsto p_{\mathbf v}(\mathbf T)\).
		Since the conditional
		law of \(T_s\) is still \(\rho(t)\,dt\),
		\[
		\mathbb E[(\partial_s p_{\mathbf v}(\mathbf T))^2]
		\le
		C_{\mathrm{inv}}(k,L)
		\operatorname{Var}(p_{\mathbf v}(\mathbf T)).
		\]
		Summing over \(s\) gives
		\[
		\mathbb E\|\nabla p_{\mathbf v}(\mathbf T)\|_2^2
		\le
		C_{\mathrm{inv}}(k,L)r
		\operatorname{Var}(p_{\mathbf v}(\mathbf T)).
		\]
		On the other hand,
		\[
		\mathbb E\|\nabla p_{\mathbf v}(\mathbf T)\|_2^2
		\ge
		\dfrac{ c_L}{r r_\ell}.
		\]
		Therefore, after absorbing the constant
		\(C_{\mathrm{inv}}(k,L)\) into \( c_L\),
		\(
		\operatorname{Var}(p_{\mathbf v}(\mathbf T))
		\ge
		\dfrac{ c_L}{r^2r_\ell}.
		\)
		Since \(\mathbf v\) was arbitrary,
		\(
		\operatorname{Cov}(\mathbf q_\ell(\mathbf T))
		\succeq
		\dfrac{ c_L}{r^2r_\ell}
		\mathbf I_{r_\ell}.
		\)
		Consequently,
		\(
		\mathbb E[
		\mathbf q_\ell(\mathbf T)\mathbf q_\ell(\mathbf T)^\top]
		\succeq
		\dfrac{ c_L}{r^2r_\ell}
		\mathbf I_{r_\ell}.
		\)
		For
		\[
		\mathbf g_\ell(\mathbf T)
		:=
		\nabla F_\ell(\boldsymbol\Phi_\ell(\mathbf T))
		=
		\mathbf W_\ell\mathbf q_\ell(\mathbf T),
		\]
		define the population moment
		\(
		\mathbf M_\ell
		:=
		\mathbb E[
		\mathbf g_\ell(\mathbf T)\mathbf g_\ell(\mathbf
		T)^\top].
		\)
		Then
		\(
		\mathbf M_\ell
		=
		\mathbf W_\ell
		\mathbb E[
		\mathbf q_\ell(\mathbf T)\mathbf q_\ell(\mathbf T)^\top]
		\mathbf W_\ell^\top,
		\)
		and therefore
		\(
		\operatorname{col}(\mathbf M_\ell)
		=
		\operatorname{col}(\mathbf W_\ell),
		\)
		while Lemma~\ref{lem:weight-singular-values} gives
		\[
		\lambda_{r_\ell}(\mathbf M_\ell)
		\ge
		\dfrac{ c_L}{r^2r_\ell}
		\dfrac{r_\ell}{r_{\ell-1}}
		=
		\dfrac{ c_L}{r^2r_{\ell-1}},
		\qquad
		\lambda_{r_\ell+1}(\mathbf M_\ell)=0.
		\]
		
For the recovered-gradient error, fix a query point
		\(\mathbf x=\widehat{\mathbf R}_1\mathbf T\), let
		\(
		\mathbf J_\ell(\mathbf x):=D\boldsymbol\phi_\ell(\mathbf
		x),
		\)
		and let \(\widehat{\mathbf J}_\ell(\mathbf x)\) denote
		the Jacobian obtained
		by replacing the prefix matrices by their recovered
		versions. We prove
		simultaneously by induction that
		\begin{equation}
		\max_{1\le j<\ell}
		\|\widehat{\mathbf h}_j(\mathbf x)-\mathbf h_j(\mathbf
		x)\|_2
		\le
		 C_L\sqrt r\,\varepsilon_W
		\label{eq:intermediate-hidden-state-error}
		\end{equation}
		and that every coordinate of each true and estimated
		hidden state up to
		layer \(\ell-1\) lies in a fixed interval
		\([- C_L, C_L]\). At the first layer,
		\(
		\widehat{\mathbf h}_1(\mathbf x)=\mathbf T,
		\mathbf h_1(\mathbf x)=\mathbf A\mathbf T,
		\)
		so
		\[
		\|\widehat{\mathbf h}_1(\mathbf x)-\mathbf h_1(\mathbf
		x)\|_2
		\le
		\varepsilon_W\|\mathbf x\|_2
		\le
		 C_L\sqrt r\,\varepsilon_W.
		\]
		Suppose the claims hold through layer \(j-1\). The map
		\(t\mapsto t^k\) is Lipschitz on the fixed interval
		containing the true
		and estimated coordinates, and
		\(\|\sigma(\widehat{\mathbf h}_{j-1})\|_2\le
		C_L\sqrt r\).
		Using \(\|\mathbf W_j\|_2\le\kappa_0\) gives
		\[
		\|\widehat{\mathbf h}_j(\mathbf x)-\mathbf h_j(\mathbf
		x)\|_2
		\le
		 C_L\sqrt r\,\varepsilon_W.
		\]
		The prefix condition makes \(\sqrt r\,\varepsilon_W\)
		sufficiently
		small, which also preserves the fixed coordinate
		interval and closes the
		induction.

		For \(1\le j<\ell\), write
\[
\mathbf D_j(\mathbf x)
:=
\operatorname{diag}
\left(
\mathbf h_j(\mathbf x)^{\odot(k-1)}
\right),
\qquad
\widehat{\mathbf D}_j(\mathbf x)
:=
\operatorname{diag}
\left(
\widehat{\mathbf h}_j(\mathbf x)^{\odot(k-1)}
\right).
\]
The true and recovered prefix Jacobians are
\[
\mathbf J_\ell(\mathbf x)
=
k^{\ell-1}
\mathbf D_{\ell-1}\mathbf W_{\ell-1}^\top
\cdots
\mathbf D_1\mathbf W_1^\top
\]
and
\[
\widehat{\mathbf J}_\ell(\mathbf x)
=
k^{\ell-1}
\widehat{\mathbf D}_{\ell-1}\widehat{\mathbf W}_{\ell-1}^\top
\cdots
\widehat{\mathbf D}_1\widehat{\mathbf W}_1^\top.
\]
On the fixed coordinate interval established above, the derivative of
\(t\mapsto t^{k-1}\) is uniformly bounded. Hence
\[
\|\widehat{\mathbf D}_j-\mathbf D_j\|_2
\le
 C_L
\|\widehat{\mathbf h}_j-\mathbf h_j\|_\infty
\le
 C_L\sqrt r\,\varepsilon_W.
\]
For compatible matrix products,
\[
\mathbf A_m\cdots\mathbf A_1
-
\mathbf B_m\cdots\mathbf B_1
=
\sum_{s=1}^m
\mathbf A_m\cdots\mathbf A_{s+1}
(\mathbf A_s-\mathbf B_s)
\mathbf B_{s-1}\cdots\mathbf B_1.
\]
All remaining factors in the two Jacobian products have operator norm
at most \( C_L\). Applying the preceding identity gives
\[
\|\widehat{\mathbf J}_\ell(\mathbf x)
-
\mathbf J_\ell(\mathbf x)\|_2
\le
 C_L\sqrt r\,\varepsilon_W.
\]

Since
\(
\boldsymbol\Phi_\ell(\mathbf T)
=
\boldsymbol\phi_\ell(\widehat{\mathbf R}_1\mathbf T),
\mathbf x=\widehat{\mathbf R}_1\mathbf T,
\)
the chain rule gives
\(
D\boldsymbol\Phi_\ell(\mathbf T)
=
\mathbf J_\ell(\mathbf x)\widehat{\mathbf R}_1.
\)
Using \(\|\widehat{\mathbf R}_1\|_2\le C_L\) and the lower
bound already obtained for
\(D\boldsymbol\Phi_\ell(\mathbf T)\),
\[
\begin{aligned}
\sigma_{r_{\ell-1}}(\mathbf J_\ell(\mathbf x))
\ge
\dfrac{
\sigma_{r_{\ell-1}}(D\boldsymbol\Phi_\ell(\mathbf T))
}{
\|\widehat{\mathbf R}_1\|_2
}
\ge
 c_L
\sqrt{\dfrac{r_{\ell-1}}r}.
\end{aligned}
\]
The prefix condition makes
\(
\|\widehat{\mathbf J}_\ell-\mathbf J_\ell\|_2
\le
\dfrac12
\sigma_{r_{\ell-1}}(\mathbf J_\ell).
\)
Weyl's inequality therefore yields
\(
\sigma_{r_{\ell-1}}(\widehat{\mathbf J}_\ell(\mathbf x))
\ge
 c_L
\sqrt{\dfrac{r_{\ell-1}}r},
\)
after adjusting the constants \( c_L, C_L\). Recovering the full gradient
\(\nabla f^\star(\mathbf x)\in\mathbb R^d\)
is unnecessary. For
\(1\le j\le r_{\ell-1}\), define
\(
\mathbf v_j(\mathbf x)
:=
\widehat{\mathbf J}_\ell(\mathbf x)^\top
\mathbf e_j
\in\mathbb R^d.
\)
By Subsection~\ref{subsec:trick},
the directional derivative
\(
b_j(\mathbf x)
:=
D_{\mathbf v_j(\mathbf x)}
f^\star(\mathbf x)
\)
can be recovered exactly using \(O_{k,L}(1)\)
real function-value queries. Moreover,
\[
b_j(\mathbf x)
=
\mathbf e_j^\top
\widehat{\mathbf J}_\ell(\mathbf x)
\nabla f^\star(\mathbf x).
\]
Thus, with
\[
\mathbf b(\mathbf x)
:=
(b_1(\mathbf x),\ldots,
b_{r_{\ell-1}}(\mathbf x))^\top,
\]
we recover exactly
\[
\mathbf b(\mathbf x)
=
\widehat{\mathbf J}_\ell(\mathbf x)
\nabla f^\star(\mathbf x)
\]
using \(O_{k,L}(r_{\ell-1})\) real
function-value queries. Define
\[
\widehat{\mathbf g}_\ell(\mathbf x)
:=
(\widehat{\mathbf J}_\ell
\widehat{\mathbf J}_\ell^\top)^{-1}
\mathbf b(\mathbf x).
\]
Equivalently,
\[
\widehat{\mathbf g}_\ell(\mathbf x)
=
(\widehat{\mathbf J}_\ell
\widehat{\mathbf J}_\ell^\top)^{-1}
\widehat{\mathbf J}_\ell
\nabla f^\star(\mathbf x).
\]
		As in the original chain-rule calculation,
		\[
		\widehat{\mathbf g}_\ell-\mathbf g_\ell
		=
		(\widehat{\mathbf J}_\ell^\top)^\dagger
		(\mathbf J_\ell-\widehat{\mathbf J}_\ell)^\top
		\mathbf g_\ell.
		\]
		Hence
		\[
		\|\widehat{\mathbf g}_\ell(\mathbf x)-\mathbf
		g_\ell(\mathbf x)\|_2
		\le
		 C_L
		\dfrac r{\sqrt{r_{\ell-1}}}
		\varepsilon_W
		\|\mathbf g_\ell(\mathbf x)\|_2.
		\]

		It remains to bound the gradient envelope without an
		output-weight ratio.
		At the final suffix layer, the fixed coordinate bounds
		and
		\(\|\mathbf a\|_1=1\) give
		\[
		\|\nabla h^{(L-1)}(\mathbf U_{L-1})\|_1\le C_L.
		\]
		For every nonnegative vector \(\mathbf v\), column
		normalization gives
		\[
		\|\mathbf W_j\mathbf v\|_1
		=
		\mathbf1^\top\mathbf W_j\mathbf v
		=
		(\mathbf W_j^\top\mathbf1)^\top\mathbf v
		=
		\|\mathbf v\|_1.
		\]
		The diagonal derivative factors are uniformly bounded,
		so backward
		induction yields
		\[
		\|\mathbf q_\ell(\mathbf T)\|_2
		\le
		\|\mathbf q_\ell(\mathbf T)\|_1
		\le
		 C_L,
		\qquad
		\|\mathbf g_\ell(\mathbf T)\|_2\le C_L.
		\]
		It follows that, uniformly in \(i\),
		\[
		\left\|
		\widehat{\mathbf g}_\ell(\mathbf T_i)
		\widehat{\mathbf g}_\ell(\mathbf T_i)^\top
		-
		\mathbf g_\ell(\mathbf T_i)\mathbf g_\ell(\mathbf
		T_i)^\top
		\right\|_2
		\le
		 C_L
		\dfrac r{\sqrt{r_{\ell-1}}}\varepsilon_W.
		\]
		
		Define
		\(
		\mathbf{M}_{\ell,m}
		:=
		\dfrac1m
		\sum_{i=1}^m
		\mathbf{g}_\ell(\mathbf{T}_i)\mathbf{g}_\ell(\mathbf{T}_i)^\top
		\)
		and
		\(\widehat{\mathbf{M}}_{\ell,m}
		:=
		\dfrac1m
		\sum_{i=1}^m
		\widehat{\mathbf{g}}_\ell(\mathbf{T}_i)
		\widehat{\mathbf{g}}_\ell(\mathbf{T}_i)^\top.
		\)
		Then deterministically,
		\[
		\|\widehat{\mathbf M}_{\ell,m}-\mathbf M_{\ell,m}\|_2
		\le
		 C_L
		\dfrac r{\sqrt{r_{\ell-1}}}\varepsilon_W.
		\]
		Furthermore,
		\[
		\|\mathbf g_\ell(\mathbf T_i)\mathbf g_\ell(\mathbf
		T_i)^\top\|_2
		\le
		 C_L.
		\]
		Matrix Bernstein's inequality therefore implies that,
		with probability at
		least \(1-\zeta\),
		\[
		\|\mathbf M_{\ell,m}-\mathbf M_\ell\|_2
		\le
		 C_L
		\left(
		\sqrt{\dfrac{\log(r/\zeta)}m}
		+
		\dfrac{\log(r/\zeta)}m
		\right).
		\]
		The assumed sample size makes this at most
		\(
		\dfrac{ c_L\varepsilon_B}{r^2r_{\ell-1}}.
		\)
		Similarly, because \(r_{\ell-1}\le r\), the prefix
		condition gives
		\(
		 C_L
		\dfrac r{\sqrt{r_{\ell-1}}}\varepsilon_W
		\le
		\dfrac{ c_L\varepsilon_B}{r^2r_{\ell-1}}.
		\)
		Consequently,
		\(
		\|\widehat{\mathbf M}_{\ell,m}-\mathbf M_\ell\|_2
		\le
		\dfrac{ c_L\varepsilon_B}{r^2r_{\ell-1}}.
		\)
		Since
		\(
		\lambda_{r_\ell}(\mathbf M_\ell)
		\ge
		\dfrac{ c_L}{r^2r_{\ell-1}},
		\lambda_{r_\ell+1}(\mathbf M_\ell)=0,
		\)
		Davis--Kahan's theorem yields
		\(
		\left\|\widehat{\mathbf{B}}_\ell\widehat{\mathbf{B}}_\ell^\top
		-
		\mathbf{B}_\ell \mathbf{B}_\ell^\top
		\right\|_2
		\le
		 C_L\varepsilon_B,
		\)
		where \(\widehat{\mathbf{B}}_\ell\) consists of the
		eigenvectors corresponding to
		the \(r_\ell\) largest eigenvalues of
		\(\widehat{\mathbf{M}}_{\ell,m}\).
		For two orthonormal bases of subspaces with projector
		distance at most
		$\eta<1$, an orthogonal alignment $\mathbf{R}$ exists
		with
		\[
		\|\widehat{\mathbf{B}}_\ell
		\mathbf{R}-\mathbf{B}_\ell\|_2
		\le
		\sqrt2\,
		\|\widehat{\mathbf{B}}_\ell\widehat{\mathbf{B}}_\ell^\top-\mathbf{B}_\ell
		\mathbf{B}_\ell^\top\|_2.
		\]
		After decreasing the constant in the assumptions and
		absorbing $\mathbf{R}$ into the
		choice of recovered basis, this gives
		\(
		\|\widehat{\mathbf{B}}_\ell-\mathbf{B}_\ell\|_2
		\le
		\varepsilon_B.
		\)
		Finally, for each sample \(\mathbf x_i\), the vector
\(
\widehat{\mathbf J}_\ell(\mathbf x_i)
\nabla f^\star(\mathbf x_i)
\in\mathbb R^{r_{\ell-1}}
\)
is recovered coordinatewise from directional derivatives.
By Subsection~\ref{subsec:trick}, this uses
\(O_{k,L}(r_{\ell-1})\le O_{k,L}(r)\)
real function-value queries per sample.
Together with
\eqref{eq:span-simplified-sample-size}, the total number
of additional queries is therefore
\[
\widetilde O_{k,L}\!\left(
r^7\kappa^{O(L)}\varepsilon_B^{-2}
\right).
\]
	\end{proof}
	
\subsection{Stable Recovery Within the Recovered Subspace}

Given an accurate basis for the intermediate-layer column space, we next use the simulated suffix oracle, convex-body sampling, and two-matrix diagonalization to recover the individual hidden directions and hence the full weight matrix.

	\begin{proposition}[Stable recovery by convex-body sampling]
		\label{prop:stable-convex-body-recovery}
		Fix \(2\le\ell\le L-2\), \(d\ge2\), and
		\(\varepsilon,\zeta\in(0,1/2)\). Suppose that, after
		consistent
		relabeling,
		\(
		\max_{1\le j<\ell}
		\|\widehat{\mathbf W}_j-\mathbf W_j\|_2
		\le\varepsilon_W,
		\|\widehat{\mathbf B}_\ell-\mathbf B_\ell\|_2
		\le\varepsilon_B,
		\)
		where the second inequality is after orthogonal
		alignment. Set
		\(
		\alpha:=1+\dfrac1L,
		\mathcal L(\varepsilon,\gamma,\zeta)
		:=
		1+
		\log\left(e+
		\dfrac{\gamma r^{13/2}}
		{\varepsilon\zeta}
		\right).
		\)
		It is sufficient that
		\begin{equation}
		\varepsilon_B
		\le
		c_{k,L}\dfrac{\varepsilon}{\kappa^2r}
		\label{eq:cb-simplified-basis-condition}
		\end{equation}
		and
		\begin{equation}
		\varepsilon_W
		\le
		\dfrac{
		c_{k,L}(\varepsilon/\gamma)^\alpha
		}{
		\kappa^{O(L)}
		r^{\,\ell/2+12+10/L}
		\mathcal L(\varepsilon,\gamma,\zeta)^\alpha
		}.
		\label{eq:cb-simplified-prefix-condition}
		\end{equation}
		Then
		\begin{equation}
		\widetilde O_{k,L}\!\left(
		\dfrac{\gamma^2r^{14}\kappa^{O(L)}}{\varepsilon^2}
		\right)
		\label{eq:cb-simplified-query}
		\end{equation}
		additional real function-value queries suffice to
		recover
		\(\widehat{\mathbf W}_\ell\) satisfying, with
		probability at least
		\(1-\zeta\),
		\(
		\|\widehat{\mathbf W}_\ell-\mathbf
		W_\ell\|_2\le\varepsilon
		\)
		after a suitable relabeling of the recovered columns.
	\end{proposition}
	
	\begin{proof}
		By the preceding approximate-prefix and simulated-suffix
		lemmas, the
		available suffix oracle satisfies
		\[
		|\widehat G_\ell(\mathbf z)-G_\ell(\mathbf z)|
		\le
		\xi_\ell G_\ell(\mathbf z),
		\qquad
		\xi_\ell
		\le
		C_{k,L}\kappa^{O(\ell)}r^{\ell/2}\varepsilon_W.
		\]
		Fix an absolute integer \(A_0\ge20\), large enough to
		dominate the fixed
		conditioning exponents in
		Lemmas~\ref{lem:population-moment-scales} and
		\ref{lem:noisy-two-matrix}, and set
		\(
		\eta
		:=
		c_{k,L}
		\dfrac{\varepsilon}{\gamma\kappa^{A_0L}r^4},
		\nu
		:=
		c_{k,L}
		\dfrac{\varepsilon}{\gamma\kappa^{A_0L+8}r^6}.
		\)
		Combining the assumed bound on \(\varepsilon_W\) with
		the preceding
		bound on \(\xi_\ell\) gives
		\[
		\xi_\ell
		\le
		\dfrac{
		c_{k,L}\nu^{1+1/L}
		}{
		\kappa^{O(L)}
		r^{6+4/L}
		\left(
		\log\dfrac{C_{k,L}\sqrt r}{\nu}
		\right)^{1+1/L}
		}.
		\]
		Since
		\(
		r_\ell\le r,
		q_\ell=O_{k,L}(1),
		\omega_\ell\le C_{k,L}\kappa^4\sqrt r,
		\)
		this is sufficient for Lemma~\ref{lem:noisy-hit-and-run}
		with
		\(\delta=1/L\) and target total-variation error \(\nu\).
		It also implies
		\[
		\xi_\ell
		\le
		c_{k,L}\kappa^{-8}\eta r^{-5/2}.
		\]
		Thus Lemma~\ref{lem:empirical-two-moments} yields, from
		\[
		\begin{aligned}
		m
		=
		O_{k,L}\!\left(
		\kappa^8r^2\eta^{-2}
		\log\dfrac{2r}{\zeta}
		\right)=
		O_{k,L}\!\left(
		\dfrac{\gamma^2r^{10}\kappa^{2A_0L+8}}
		{\varepsilon^2}
		\log\dfrac{2r}{\zeta}
		\right)
		\end{aligned}
		\]
		independent endpoints, moment matrices satisfying
		\[
		\|\widehat{\boldsymbol\Sigma}_{x,\ell}
		-\boldsymbol\Sigma_{x,\ell}\|_2
		\le
		\eta\|\boldsymbol\Sigma_{x,\ell}\|_2,
		\qquad \|\widehat{\boldsymbol\Sigma}_{g,\ell}
		-\boldsymbol\Sigma_{g,\ell}\|_2
		\le
		\eta\|\boldsymbol\Sigma_{g,\ell}\|_2
		\]
		with probability at least \(1-\zeta\). By Lemma~\ref{lem:population-moment-scales},
		\(
		\operatorname{cond}(\mathbf V_\ell)\le9\kappa^2,
		\rho_\ell\le C_{k,L}\kappa^8r^3.
		\)
		Assumption~\ref{ass:spectral-gap} supplies the inverse
		relative gap
		\(\gamma\). Lemma~\ref{lem:noisy-two-matrix} therefore
		implies that the
		normalized eigenvectors of
		\(
		\widehat{\boldsymbol\Sigma}_{g,\ell}
		\widehat{\boldsymbol\Sigma}_{x,\ell}
		\)
		can be labeled so that
		\[
		\min_{s_i\in\{\pm1\}}
		\|\widehat{\mathbf v}_i-s_i\mathbf v_i\|_2
		\le
		C_{k,L}\gamma\kappa^{12}r^3\eta
		\le
		c_{k,L}\dfrac{\varepsilon}{\kappa^4r}.
		\]

		For each recovered direction define
		\(
		\widehat{\mathbf u}_i
		:=
		\widehat{\mathbf B}_\ell\widehat{\mathbf v}_i,
		\widehat{\mathbf w}_{\ell,i}
		:=\dfrac{|\widehat{\mathbf u}_i|}
		{\mathbf1^\top|\widehat{\mathbf u}_i|},
	   \)
		where the absolute value is coordinatewise. Let
		\(
		\mathbf w_i:=\mathbf w_{\ell,i},
		\mathbf u_i^\star:=\dfrac{\mathbf w_i}{\|\mathbf
		w_i\|_2}.
		\)
		Since \(\mathbf W_\ell \ge \mathbf0\), the vector
		\(\mathbf u_i^\star\) is coordinatewise nonnegative. Write
		\(\mathbf P_\ell\) and \(\widehat{\mathbf P}_\ell\) for
		the projectors
		onto the true and recovered column spaces. Because
		\(
		\widehat{\mathbf B}_\ell\mathbf v_i
		=
		\dfrac{\widehat{\mathbf P}_\ell\mathbf w_i}
		{\|\widehat{\mathbf P}_\ell\mathbf w_i\|_2},
		\)
		and
		\(\|\widehat{\mathbf P}_\ell-\mathbf P_\ell\|_2
		\le C\varepsilon_B\), normalization of the projected
		vector gives
		\[
		\min_{s_i\in\{\pm1\}}
		\|s_i\widehat{\mathbf u}_i-\mathbf u_i^\star\|_2
		\le
		C_{k,L}\varepsilon_B
		+
		c_{k,L}\dfrac{\varepsilon}{\kappa^4r}.
		\]
		The coordinatewise inequality
		\(
		\bigl||a|-b\bigr|\le|a-b|
		\)
		for \(b\ge0\) yields the same bound after taking
		absolute values.

		The normalization \(\mathbf1^\top\mathbf w_i=1\) gives
		\(
		\mathbf1^\top\mathbf u_i^\star
		=
		\dfrac1{\|\mathbf w_i\|_2}
		\ge1.
		\)
		Hence, for the constants in the hypotheses sufficiently
		small,
		\[
		\|\widehat{\mathbf w}_{\ell,i}-\mathbf w_i\|_2
		\le
		C_{k,L}\sqrt{r_{\ell-1}}
		\left(
		\varepsilon_B+
		\dfrac{\varepsilon}{\kappa^4r}
		\right).
		\]
		Consequently,
		\[
		\begin{aligned}
		\|\widehat{\mathbf W}_\ell-\mathbf W_\ell\|_2
		\le
		\|\widehat{\mathbf W}_\ell-\mathbf W_\ell\|_F
		\le
		C_{k,L}\sqrt{r_\ell r_{\ell-1}}
		\left(
		\varepsilon_B+
		\dfrac{\varepsilon}{\kappa^4r}
		\right)
		\le
		C_{k,L}r
		\left(
		\varepsilon_B+
		\dfrac{\varepsilon}{\kappa^4r}
		\right)
		\le
		\varepsilon,
		\end{aligned}
		\]
		after decreasing \(c_{k,L}\) if necessary.

		For the query count, each independent endpoint requires,
		by
		Lemma~\ref{lem:noisy-hit-and-run},
		\(
		\widetilde O_{k,L}(\kappa^8r^4)
		\)
		evaluations of the simulated suffix oracle. The gradient
		estimator in
		Lemma~\ref{lem:stable-gradient-from-suffix} requires
		only
		\(O_{k,L}(r)\) further evaluations per endpoint. Each
		simulated suffix
		evaluation is implemented using \(O_{k,L}(1)\) real
		evaluations of
		\(f^\star\). Multiplying by the above value of \(m\)
		gives
		\[
		\widetilde O_{k,L}\!\left(
		\dfrac{\gamma^2r^{14}\kappa^{2A_0L+16}}
		{\varepsilon^2}
		\log\dfrac1\zeta
		\right)
		\]
		additional real function-value queries.
	\end{proof}

\section{Recovery of the Final Hidden Layer and Output Weights}
\label{sec:appendix-final-layer}
This section analyzes the final recovery stage, where a two-Hessian decomposition identifies the last hidden layer and the output weights.
	\begin{proposition}[Stable recovery of the final hidden layer and output weights]
		\label{prop:stable-final-layer}
		Suppose that after the consistent relabeling
		inherited from the
		previous stages,
		\[
		\max_{1\le j\le L-2}
		\|\widehat{\mathbf W}_j-\mathbf W_j\|_2
		\le\varepsilon_W.
		\]
		There exists \(c_{k,L}>0\), independent of the
		conditioning parameters,
		such that, for every \(\varepsilon,\zeta\in(0,1/2)\), it
		is sufficient that
		\begin{equation}
		\varepsilon_W
		\le
		c_{k,L}
		\dfrac{\zeta\varepsilon}
		{r^{L/2+6}\kappa^{O(L)}}.
		\label{eq:final-layer-simplified-prefix-condition}
		\end{equation}
		Then, using \(O_{k,L}(r^2)\) additional real
		function-value queries, one
		can construct \(\widehat{\mathbf W}_{L-1}\) and
		\(\widehat{\mathbf a}\) such that, with probability at
		least
		\(1-\zeta\), after a suitable relabeling,
		\[
		\|\widehat{\mathbf W}_{L-1}-\mathbf W_{L-1}\|_2
		\le\varepsilon,
		\qquad
		\|\widehat{\mathbf a}-\mathbf a\|_1
		\le\varepsilon.
		\]
	\end{proposition}
	
	\begin{proof}
		Throughout this proof, \( C_L\) and \(
		c_L\) are
		generic quantities satisfying
		\[
		1\le C_L\le C_{k,L}\kappa^{O(L)},
		\qquad
		c_{k,L}\kappa^{-O(L)}\le c_L\le1.
		\]
		Write
		\[
		n:=r_{L-2},
		\qquad
		s:=r_{L-1},
		\qquad
		\mathbf W=(\mathbf w_1,\ldots,\mathbf w_s),
		\]
		and consider the final suffix
		\(
		F(\mathbf y)
		:=
		F_{L-1}(\mathbf y)
		=
		\sum_{i=1}^s a_i(\mathbf w_i^\top\mathbf y)^k,
		\mathbf y\in\mathbb R^n.
		\)
		Lemma~\ref{lem:output-weight-scale} gives
		\[
		\dfrac\mu s\le a_{\min}\le\dfrac1s\le a_{\max}\le1.
		\]

		We first construct accurate Hessians of \(F\) at two
		points using the
		approximately recovered prefix. In the notation of
		Lemma~\ref{lem:approximate-prefix-inversion}, let
		\(
		\vartheta_k:=\dfrac{k-1}{2k},
		P_{L-2}
		=
		\sqrt{\dfrac r n}
		\prod_{j=1}^{L-2}r_j^{\vartheta_k}.
		\)
		For \(\mathbf y\) in a fixed bounded positive box, the
		simulated suffix
		oracle has the form
		\(
		\widehat F(\mathbf y)
		=
		\operatorname{Re}F_{\mathbb C}(\mathbf y+\mathbf
		e(\mathbf y)),
		\)
		where
		\[
		\|\mathbf e(\mathbf y)\|_2
		\le
		 C_L P_{L-2}\varepsilon_W\|\mathbf y\|_2.
		\]
		On the box considered below, \(\|\mathbf y\|_2\le C\sqrt
		n\), and
		therefore
		\[
		\begin{aligned}
		P_{L-2}\|\mathbf y\|_2
		\le
		C\sqrt r\prod_{j=1}^{L-2}r_j^{\vartheta_k}
		\le
		Cr^{1/2+\vartheta_k(L-2)}
		\le
		Cr^{(L-1)/2}.
		\end{aligned}
		\]
		Thus
		\[
		\|\mathbf e(\mathbf y)\|_2
		\le
		 C_L r^{(L-1)/2}\varepsilon_W.
		\]
		Since \(\mathbf W\ge \mathbf0\), each column has unit
		\(\ell_1\)-norm, and
		all points lie in a fixed positive box, every
		\(\mathbf w_i^\top\mathbf y\) remains in a fixed
		positive interval. The
		identity
		\[
		(z+\Delta)^k-z^k
		=
		\Delta\sum_{j=0}^{k-1}(z+\Delta)^{k-1-j}z^j
		\]
		and \(\sum_i a_i=1\) give
		\(
		|\widehat F(\mathbf y)-F(\mathbf y)|
		\le
		 C_L\|\mathbf e(\mathbf y)\|_2
		\le
		 C_L r^{(L-1)/2}\varepsilon_W.
		\)
		All complex evaluations involved in the simulated suffix
		oracle are
		implemented exactly by real function-value queries.

		Set \(\mathbf y^{(2)}:=\mathbf1_n\). Let
		\(\mathbf u\sim\operatorname{Unif}(S^{n-1})\), choose a
		sufficiently
		small constant \(\tau=\tau_{k,L}>0\), and define
		\[
		\mathbf y^{(1)}:=\mathbf1_n+\tau\mathbf u.
		\]
		Choose fixed interpolation nodes sufficiently close to
		zero so that all
		points used to evaluate directional second derivatives
		around
		\(\mathbf y^{(1)}\) and \(\mathbf y^{(2)}\) remain in
		the same positive
		box. For \(\mathbf y\in\{\mathbf y^{(1)},\mathbf
		y^{(2)}\}\) and
		\(\mathbf v\in\mathbb R^n\), the function
		\(t\mapsto F(\mathbf y+t\mathbf v)\) has degree at most
		\(k\), so its
		second derivative at zero is a fixed linear combination
		of \(k+1\)
		function values. Replacing these values by simulated
		suffix values gives an
		estimator of \(\mathbf v^\top\nabla^2F(\mathbf y)\mathbf
		v\) with error at
		most \( C_L r^{(L-1)/2}\varepsilon_W\).

		Applying this to
		\[
		\mathbf v=\mathbf e_p,
		\qquad
		\mathbf v=\mathbf e_p+\mathbf e_q,
		\qquad
		1\le p<q\le n,
		\]
		and using polarization, we construct symmetric matrices
		\(
		\widehat{\mathbf H}_j,
		\mathbf H_j:=\nabla^2F(\mathbf y^{(j)}),
		j=1,2,
		\)
		satisfying
		\[
		\|\widehat{\mathbf H}_j-\mathbf H_j\|_2
		\le
		\delta_H,
		\qquad
		\delta_H
		:=
		 C_L r^{(L+1)/2}\varepsilon_W.
		\]
		The construction uses \(O_k(n^2)\le O_k(r^2)\) simulated
		suffix
		evaluations, and hence \(O_{k,L}(r^2)\) real
		function-value queries.

		At \(\mathbf y^{(2)}=\mathbf1_n\), column normalization
		gives
		\(\mathbf w_i^\top\mathbf y^{(2)}=1\), and therefore
		\[
		\mathbf H_2
		=
		k(k-1)\mathbf W\operatorname{diag}(\mathbf a)\mathbf
		W^\top.
		\]
		If \(s=n\), then \(\mathbf W\) is square and invertible,
		so no subspace
		estimation is needed and we set \(\widehat{\mathbf
		B}:=\mathbf I_n\). If
		\(s<n\), Lemmas~\ref{lem:weight-singular-values} and
		\ref{lem:output-weight-scale} give
		\[
		\begin{aligned}
		\lambda_s(\mathbf H_2)
		\ge
		k(k-1)a_{\min}\sigma_{\min}(\mathbf W)^2\ge
		 c_L
		\dfrac1s\dfrac s n
		=
		\dfrac{ c_L}{n},
		\qquad
		\lambda_{s+1}(\mathbf H_2)=0.
		\end{aligned}
		\]
		Let \(\widehat{\mathbf B}\in\mathbb R^{n\times s}\)
		contain the
		eigenvectors corresponding to the \(s\) largest
		eigenvalues of
		\(\widehat{\mathbf H}_2\). Davis--Kahan's theorem yields
		\[
		\|\mathbf P_{\widehat{\mathbf B}}-\mathbf P_{\mathbf
		W}\|_2
		\le
		 C_L n\delta_H
		\le
		 C_L r\delta_H.
		\]
		Thus the same projector bound holds in either case. Choose an
orthonormal basis
\[
\mathbf B
\in
\underset{
\substack{
\widetilde{\mathbf B}^{\top}\widetilde{\mathbf B}=\mathbf I_s\\
\operatorname{col}(\widetilde{\mathbf B})
=\operatorname{col}(\mathbf W)
}
}{\arg\min}
\,
\|\widehat{\mathbf B}-\widetilde{\mathbf B}\|_2.
\]
By the standard relation between principal angles and optimally
chosen orthonormal bases,
\[
\|\widehat{\mathbf B}-\mathbf B\|_2
\le
\sqrt{2}\,
\|\mathbf P_{\widehat{\mathbf B}}
-
\mathbf P_{\mathbf W}\|_2
\le
C_L n\delta_H.
\]
In the square case \(s=n\), this convention simply corresponds to
taking
\(
\mathbf B=\widehat{\mathbf B}=\mathbf I_n.
\)
Set
\(
\varepsilon_{B,\mathrm{fin}}
:=
\|\widehat{\mathbf B}-\mathbf B\|_2.
\)
Since
\(
\delta_H=C_Lr^{(L+1)/2}\varepsilon_W,
n\le r,
\)
we have
\(
\varepsilon_{B,\mathrm{fin}}
\le
C_Lr^{(L+3)/2}\varepsilon_W.
\)
Under
\eqref{eq:final-layer-simplified-prefix-condition},
after decreasing \(c_{k,L}\), and enlarging if necessary the
exponent hidden in \(\kappa^{O(L)}\), it follows that
\(
\varepsilon_{B,\mathrm{fin}}\le\frac12.
\)
Define
\(
\mathbf C:=\widehat{\mathbf B}^{\top}\mathbf W
\in\mathbb R^{s\times s}.
\)
Lemma~\ref{lem:reduced-weight-conditioning}, applied to the above
choice of \(\mathbf B\) and \(\widehat{\mathbf B}\), gives
\[
\sigma_{\min}(\mathbf C)
\ge
(1-\varepsilon_{B,\mathrm{fin}})
\sigma_{\min}(\mathbf W)
\ge
\frac{1}{2\kappa_0}\sqrt{\frac{s}{n}},
\]
and
\[
\sigma_{\max}(\mathbf C)
\le
(1+\varepsilon_{B,\mathrm{fin}})
\sigma_{\max}(\mathbf W)
\le
\frac32\kappa_0.
\]
Moreover,
\[
\operatorname{cond}(\mathbf C)
\le
\frac{1+\varepsilon_{B,\mathrm{fin}}}
{1-\varepsilon_{B,\mathrm{fin}}}
\operatorname{cond}(\mathbf W)
\le
3\kappa_0.
\]
Thus, in the generic \(c_L,C_L\) notation used in this proof,
\[
c_L\sqrt{\frac{s}{n}}
\le
\sigma_{\min}(\mathbf C)
\le
\sigma_{\max}(\mathbf C)
\le
C_L,
\qquad
\operatorname{cond}(\mathbf C)\le C_L.
\]

		Define
		\(
		\mathbf M_j:=\widehat{\mathbf B}^\top\mathbf
		H_j\widehat{\mathbf B},
		\widehat{\mathbf M}_j
		:=
		\widehat{\mathbf B}^\top\widehat{\mathbf
		H}_j\widehat{\mathbf B}.
		\)
		Then
		\(\|\widehat{\mathbf M}_j-\mathbf M_j\|_2\le\delta_H\),
		and
		\[
		\mathbf M_2
		=
		k(k-1)\mathbf C\operatorname{diag}(\mathbf a)\mathbf
		C^\top,
		\]
		while
		\[
		\mathbf M_1
		=
		k(k-1)\mathbf C
		\operatorname{diag}\!\left(
		a_i(1+\tau\mathbf w_i^\top\mathbf u)^{k-2}
		\right)
		\mathbf C^\top.
		\]
		Consequently,
		\[
		\mathbf A:=\mathbf M_1\mathbf M_2^{-1}
		=
		\mathbf
		C\operatorname{diag}(\lambda_1,\ldots,\lambda_s)\mathbf
		C^{-1},
		\]
		where
		\(
		\lambda_i
		:=
		(1+\tau\mathbf w_i^\top\mathbf u)^{k-2}.
		\)
		For the spectral separation,
		Lemma~\ref{lem:weight-singular-values} gives
		\(
		\|\mathbf w_i-\mathbf w_j\|_2
		\ge
		 c_L\sqrt{\dfrac s n},
		i\neq j.
		\)
		The spherical small-ball bound implies, for every
		\(i\neq j\),
		\[
		\mathbb P\!\left(
		|(\mathbf w_i-\mathbf w_j)^\top\mathbf u|
		\le
		t\dfrac{\|\mathbf w_i-\mathbf w_j\|_2}{\sqrt n}
		\right)
		\le Ct.
		\]
		Choosing \(t=c\zeta/s^2\) and applying a union bound
		over all pairs gives,
		with probability at least \(1-\zeta\),
		\[
		\min_{i\neq j}
		|(\mathbf w_i-\mathbf w_j)^\top\mathbf u|
		\ge
		 c_L
		\dfrac{\zeta}{s^{3/2}n}
		\ge
		 c_L\dfrac{\zeta}{r^{5/2}}.
		\]
		By the choice of \(\tau\),
		\(1/2\le1+\tau\mathbf w_i^\top\mathbf u\le3/2\). The
		mean-value theorem
		therefore gives, on the same event,
		\[
		\Delta
		:=
		\min_{i\neq j}|\lambda_i-\lambda_j|
		\ge
		 c_L\dfrac{\zeta}{r^{5/2}}.
		\]

		Moreover,
		\[
		\lambda_{\min}(\mathbf M_2)\ge\dfrac{ c_L}{n},
		\qquad
		\|\mathbf A\|_2\le C_L.
		\]
		The prefix condition makes \(\delta_H\le
		c_L/n\), so
		\(\widehat{\mathbf M}_2\) is invertible and
		\(\|\widehat{\mathbf M}_2^{-1}\|_2\le C_L n\).
		Let
		\(\mathbf E_j:=\widehat{\mathbf M}_j-\mathbf M_j\).
		Using
		\(\mathbf A\mathbf M_2=\mathbf M_1\), the perturbation
		admits the
		relative factorization
		\[
		\begin{aligned}
		\mathbf E_{\mathbf A}:=
		\widehat{\mathbf M}_1\widehat{\mathbf M}_2^{-1}
		-\mathbf M_1\mathbf M_2^{-1}
		=
		(\mathbf E_1-\mathbf A\mathbf E_2)
		\widehat{\mathbf M}_2^{-1}.
		\end{aligned}
		\]
		Hence
		\[
		\|\mathbf E_{\mathbf A}\|_2
		\le
		 C_L n\delta_H
		\le
		 C_L r\delta_H.
		\]
		Let \(\mathbf V_{\mathbf C}\) denote the matrix whose
		columns are the
		normalized columns of \(\mathbf C\). The conditioning
		bounds give
		\(\operatorname{cond}(\mathbf V_{\mathbf C})\le
		C_L\). Applying
		the simple-spectrum eigenvector argument used in
		Lemma~\ref{lem:noisy-two-matrix} shows that the
		eigenvectors of
		\(\widehat{\mathbf M}_1\widehat{\mathbf M}_2^{-1}\) may
		be labeled so
		that, after choosing signs,
		\[
		\left\|
		\widehat{\mathbf v}_i-
		\dfrac{\mathbf c_i}{\|\mathbf c_i\|_2}
		\right\|_2
		\le
		 C_L
		\dfrac{r\delta_H}{\Delta}
		\le
		 C_L
		\dfrac{r^{L/2+4}}\zeta\varepsilon_W.
		\]

		Set
		\(
		\widehat{\mathbf u}_i:=\widehat{\mathbf
		B}\widehat{\mathbf v}_i,
		\mathbf u_i^\star:=\dfrac{\mathbf w_i}{\|\mathbf
		w_i\|_2}.
		\)
		Combining the eigenvector and subspace estimates yields
		\[
		\min_{\epsilon_i\in\{\pm1\}}
		\|\epsilon_i\widehat{\mathbf u}_i-\mathbf u_i^\star\|_2
		\le
		 C_L
		\dfrac{r^{L/2+4}}\zeta\varepsilon_W.
		\]
		Since \(\mathbf w_i>\mathbf0\), taking coordinatewise
		absolute values
		preserves this bound. Define
		\(
		\widehat{\mathbf w}_i
		:=
		\dfrac{|\widehat{\mathbf u}_i|}
		{\mathbf1^\top|\widehat{\mathbf u}_i|}.
		\)
		Because
		\(\mathbf1^\top\mathbf u_i^\star=1/\|\mathbf
		w_i\|_2\ge1\),
		\(
		\|\widehat{\mathbf w}_i-\mathbf w_i\|_2
		\le
		 C_L
		\dfrac{r^{L/2+9/2}}\zeta\varepsilon_W.
		\)
		Consequently,
		\(
		E_W
		:=
		\|\widehat{\mathbf W}-\mathbf W\|_2
		\le
		\|\widehat{\mathbf W}-\mathbf W\|_F
		\le
		 C_L
		\dfrac{r^{L/2+5}}\zeta\varepsilon_W
		\le
		c_{k,L}\dfrac\varepsilon r
		\le\varepsilon.
		\)
		The stronger \(\varepsilon/r\) bound is used below.
		Enlarging the exponent
		hidden in \(\kappa^{O(L)}\), it also ensures
		\(
		E_W
		\le
		\dfrac12\sigma_{\min}(\mathbf W),
		\)
		because
		\(\sigma_{\min}(\mathbf W)\ge\kappa_0^{-1}\sqrt{s/n}\ge
		(\kappa_0\sqrt r)^{-1}\). Hence \(\widehat{\mathbf W}\)
		has full column
		rank. To recover \(\mathbf a\), use
		\[
		\mathbf H_2
		=
		k(k-1)\mathbf W\operatorname{diag}(\mathbf a)\mathbf
		W^\top.
		\]
		Define
		\(
		\widehat{\mathbf a}
		:=
		\dfrac1{k(k-1)}
		\operatorname{diag}\!\left(
		\widehat{\mathbf W}^\dagger
		\widehat{\mathbf H}_2
		\widehat{\mathbf W}^{\dagger\top}
		\right).
		\)
		Let
		\(
		\mathbf T_W:=\widehat{\mathbf W}^\dagger\mathbf W.
		\)
		Since \(\widehat{\mathbf W}^\dagger\widehat{\mathbf
		W}=\mathbf I_s\),
		\[
		\mathbf T_W-\mathbf I_s
		=
		-\widehat{\mathbf W}^\dagger
		(\widehat{\mathbf W}-\mathbf W),
		\]
		and therefore
		\[
		\|\mathbf T_W-\mathbf I_s\|_2
		\le
		 C_L\sqrt{\dfrac n s}\,E_W.
		\]
		Using \(a_{\max}\le1\),
		\[
		\|\mathbf T_W\operatorname{diag}(\mathbf a)\mathbf
		T_W^\top
		-\operatorname{diag}(\mathbf a)\|_2
		\le
		 C_L\sqrt{\dfrac n s}\,E_W.
		\]
		The Hessian perturbation contributes
		\[
		\left\|
		\widehat{\mathbf W}^\dagger
		(\widehat{\mathbf H}_2-\mathbf H_2)
		\widehat{\mathbf W}^{\dagger\top}
		\right\|_2
		\le
		 C_L\dfrac n s\delta_H.
		\]
		Taking the \(\ell_1\)-norm of the diagonal and using
		\(\sqrt{ns}\le r\) and \(n\le r\), we obtain
		\[
		\|\widehat{\mathbf a}-\mathbf a\|_1
		\le
		 C_L r(E_W+\delta_H)
		\le
		\varepsilon,
		\]
		where the final inequality follows from the prefix
		condition and the bounds
		for \(E_W\) and \(\delta_H\). This completes the proof.
	\end{proof}

\section{Proof of the Main Theorem}
\label{sec:simplified-conditioning-bookkeeping}

The preceding propositions isolate four sources of
complexity. The input dimension \(d\) in the query
complexity appears only in
the initial recovery of the first-layer feature subspace,
which costs \(O_{k,L}(dr)\) queries. Once this subspace has
been identified, all subsequent gradient information is
accessed through hidden-dimensional directional derivatives,
and all sampling, moment estimation, and prefix-inversion
costs are governed by the maximal hidden width \(r\).
The uniform inverse relative spectral gap contributes powers
of \(\gamma\), while composing \(O(L)\) conditioned stages
produces the factor \(\kappa^{O(L^2)}\).

\begin{proof}[Proof of Theorem~\ref{thm:main}]
	Set
	\(
	\alpha:=1+\dfrac1L,
	p_\star:=\alpha^{L-3}.
	\)
	Since
	\(
	p_\star
	=
	\left(1+1/L\right)^{L-3}
	<
	\left(1+1/L\right)^L
	<
	e,
	\)
	it is enough to prove the sharper query bound with
	accuracy exponent
	\(2p_\star\). Let
	\(
	N:=2L-4,
	\zeta_0:=\dfrac{0.01}{N}.
	\)
	The number \(N\) upper-bounds the randomized recovery
	stages: there is one
	first-layer stage, two stages for every \(2\le\ell\le
	L-2\), and one final
	stage. We will show
	\begin{equation}
	Q_{\mathrm{tot}}
	\le
	\widetilde O_{k,L}\!\left(
	dr+
	r^{L^2+O(L)}
	\gamma^{O(L)}
	\kappa^{O(L^2)}
	\delta^{-2p_\star}
	\right).
	\label{eq:main-pstar-complexity}
	\end{equation}

	\paragraph{Accuracy Schedule.}
	For \(2\le\ell\le L-2\), define
	\(
	\beta_\ell
	:=
	\dfrac\ell2+12+\dfrac{10}{L}.
	\)
	Choose an absolute constant \(A\ge1\) sufficiently large
	that
	\(\kappa^{AL}\) dominates every factor denoted by
	\(\kappa^{O(L)}\) in
	Propositions~\ref{prop:stable-column-space},
	\ref{prop:stable-convex-body-recovery}, and
	\ref{prop:stable-final-layer}. Define the terminal intermediate-layer accuracy by
	\[
	\tau_{L-2}
	:=
	c_{k,L}
	\dfrac{\zeta_0\delta}
	{r^{L/2+6}\kappa^{AL}}.
	\]
	For \(\ell=L-2,L-3,\ldots,2\), define recursively
	\begin{equation}
	\tau_{\ell-1}
	:=
	\dfrac{
	c_{k,L}\tau_\ell^\alpha
	}{
	\gamma^\alpha
	\kappa^{AL}
	r^{\beta_\ell}
	\left[
	1+
	\log\!\left(
	 e+
	\dfrac{\gamma r^{13/2}}
	{\zeta_0\tau_\ell}
	\right)
	\right]^\alpha
	}.
	\label{eq:simplified-layer-accuracy-recurrence}
	\end{equation}
	After decreasing \(c_{k,L}\), if necessary,
	\(
	0<\tau_1\le\tau_2\le\cdots\le\tau_{L-2}<\dfrac12.
	\)
	Indeed, \(\tau_\ell^\alpha\le\tau_\ell\), \(\gamma\ge1\),
	\(\beta_\ell>9/2\), and the logarithmic factor is at
	least one. The
	terminal value is also at most \(1/2\) after the same
	choice of constant.

	\paragraph{First-Layer Recovery.}
	Apply Proposition~\ref{prop:first-layer-recovery} with
	target accuracy
	\(\tau_1\) and failure probability \(\zeta_0\). Its query
	count is
	\[
	\widetilde O_{k,L}\!\left(
	dr+
	\gamma^2r^{13}\kappa^{O(1)}\tau_1^{-2}
	\right).
	\]
	With probability at least \(1-\zeta_0\), the result
	satisfies, after a
	permutation of the columns and a choice of their inherent
	signs,
	\(
	\|\widehat{\mathbf W}_1-\mathbf W_1\|_2\le\tau_1.
	\)
	\paragraph{Intermediate-Layer Recovery.}
	Fix \(2\le\ell\le L-2\) and suppose inductively that,
	after propagating the
	relabelings selected in the preceding stages,
	\(
	\|\widehat{\mathbf W}_j-\mathbf W_j\|_2\le\tau_j,
	1\le j<\ell.
	\)
	Since the sequence \((\tau_j)\) is nondecreasing,
	\[
	\max_{1\le j<\ell}
	\|\widehat{\mathbf W}_j-\mathbf W_j\|_2
	\le
	\tau_{\ell-1}.
	\]
	Set
	\(
	\varepsilon_{B,\ell}
	:=
	c_{k,L}\dfrac{\tau_\ell}{\kappa^2r}.
	\)
	The recurrence, together with
	\(\tau_\ell^\alpha\le\tau_\ell\) and
	\(\beta_\ell>9/2\), gives
	\[
	\tau_{\ell-1}
	\le
	c_{k,L}
	\dfrac{\tau_\ell}{\kappa^{O(L)}r^{9/2}}
	=
	c_{k,L}
	\dfrac{\varepsilon_{B,\ell}}
	{\kappa^{O(L)}r^{7/2}}.
	\]
	Thus Proposition~\ref{prop:stable-column-space} applies
	with target basis
	accuracy \(\varepsilon_{B,\ell}\). With probability at
	least
	\(1-\zeta_0\), it returns an orthonormal matrix
	\(\widehat{\mathbf B}_\ell\) such that, after orthogonal
	alignment,
	\(
	\|\widehat{\mathbf B}_\ell-\mathbf B_\ell\|_2
	\le
	\varepsilon_{B,\ell}.
	\)
	The corresponding query count is
	\[
	\begin{aligned}
	\widetilde O_{k,L}\!\left(
	r^7\kappa^{O(L)}
	\varepsilon_{B,\ell}^{-2}
	\right)
	&=
	\widetilde O_{k,L}\!\left(
	r^9\kappa^{O(L)}
	\tau_\ell^{-2}
	\right).
	\end{aligned}
	\]

	The basis condition in
	Proposition~\ref{prop:stable-convex-body-recovery} is
	satisfied by
	construction. Moreover, recurrence
	\eqref{eq:simplified-layer-accuracy-recurrence} is
	exactly the prefix
	condition in that proposition with target accuracy
	\(\tau_\ell\), failure
	probability \(\zeta_0\), and the uniform inverse-gap
	parameter \(\gamma\).
	Hence the convex-body stage applies, with query count
	\[
	\widetilde O_{k,L}\!\left(
	\gamma^2r^{14}\kappa^{O(L)}\tau_\ell^{-2}
	\right).
	\]
	With probability at least \(1-\zeta_0\), it returns,
	after a suitable
	column relabeling,
	\(
	\|\widehat{\mathbf W}_\ell-\mathbf
	W_\ell\|_2\le\tau_\ell.
	\)
	This closes the induction for every \(2\le\ell\le L-2\).

	\paragraph{Final Hidden Layer and Output Weights Recovery.}
	After recovery through layer \(L-2\),
	\[
	\max_{1\le j\le L-2}
	\|\widehat{\mathbf W}_j-\mathbf W_j\|_2
	\le
	\tau_{L-2}
	=
	c_{k,L}
	\dfrac{\zeta_0\delta}
	{r^{L/2+6}\kappa^{AL}}.
	\]
	By the choice of \(A\), this is the sufficient condition
	in
	Proposition~\ref{prop:stable-final-layer}. That
	proposition uses
	\(O_{k,L}(r^2)\) additional real function-value queries
	and returns, with
	probability at least \(1-\zeta_0\),
	\[
	\|\widehat{\mathbf W}_{L-1}-\mathbf W_{L-1}\|_2\le\delta,
	\qquad
	\|\widehat{\mathbf a}-\mathbf a\|_1\le\delta.
	\]
	Because \(\tau_\ell\le\tau_{L-2}\le\delta\) for
	\(1\le\ell\le L-2\), the intersection of the successful
	events satisfies
	all parameter-recovery guarantees in the theorem. The
	permutations selected
	at successive stages are propagated through adjacent
	weight matrices,
	producing one consistent hidden-unit relabeling. The
	first-layer column signs
	may be selected arbitrarily because \(k\) is even.

	\paragraph{Solving the Accuracy Recurrence.}
	Up to factors suppressed by \(\widetilde O_{k,L}\), write
	\begin{equation}
	\tau_\ell^{-1}
	\le
	\widetilde O_{k,L}\!\left(
	r^{s_\ell}
	\gamma^{g_\ell}
	\kappa^{t_\ell}
	\delta^{-p_\ell}
	\right).
	\label{eq:accuracy-exponent-ansatz}
	\end{equation}
	The terminal values are
	\(
	p_{L-2}=1,
	s_{L-2}=\dfrac L2+6,
	g_{L-2}=0,
	t_{L-2}=AL.
	\)
	The recurrence gives
	\(
	p_{\ell-1}=\alpha p_\ell,
	s_{\ell-1}=\alpha s_\ell+\beta_\ell,
	g_{\ell-1}=\alpha g_\ell+\alpha,
	t_{\ell-1}=\alpha t_\ell+AL.
	\)
	Therefore
	\(
	p_1=\alpha^{L-3}=p_\star.
	\)
	Unrolling the width recurrence yields
	\begin{equation}
	s_1
	=
	p_\star\left(\dfrac L2+6\right)
	+
	\sum_{\ell=2}^{L-2}\alpha^{\ell-2}\beta_\ell.
	\label{eq:s1-expanded}
	\end{equation}
	Set
	\(
	c_L:=12+10/L.
	\)
	Then \(\beta_\ell=\ell/2+c_L\). The geometric identities
	\(
	\sum_{\ell=2}^{L-2}\alpha^{\ell-2}
	=
	L(p_\star-1),
	\)
	and
	\(
	\sum_{\ell=2}^{L-2}\ell\alpha^{\ell-2}
	=
	L^2-L-2Lp_\star
	\)
	give
	\[
	\begin{aligned}
	s_1
	=
	p_\star\left(\dfrac L2+6\right)
	+
	\dfrac12\left(L^2-L-2Lp_\star\right)
	+
	c_L L(p_\star-1)
	=
	\dfrac{L^2}{2}+O(L).
	\end{aligned}
	\]
	Similarly,
	\[
	g_1
	=
	\sum_{j=1}^{L-3}\alpha^j
	=
	(L+1)(p_\star-1)
	=
	O(L),
	\]
	and
	\(
	t_1
	=
	ALp_\star+AL^2(p_\star-1)
	=
	O(L^2).
	\)
	The logarithmic factors generated by the recurrence are
	polynomial in
	\(\log r\), \(\log\gamma\), \(\log(1/\delta)\), and
	\(\log\kappa\). The first three are suppressed by
	\(\widetilde O_{k,L}\), while \(\log\kappa\) is absorbed
	into the displayed
	factor \(\kappa^{O(L^2)}\), since \(\kappa\ge2\).
	Substituting these bounds
	into \eqref{eq:accuracy-exponent-ansatz} gives
	\begin{equation}
	\tau_1^{-1}
	\le
	\widetilde O_{k,L}\!\left(
	r^{L^2/2+O(L)}
	\gamma^{g_1}
	\kappa^{O(L^2)}
	\delta^{-p_\star}
	\right),
	\label{eq:simplified-tau1-bound}
	\end{equation}
	and consequently
	\[
	\tau_1^{-2}
	\le
	\widetilde O_{k,L}\!\left(
	r^{L^2+O(L)}
	\gamma^{2g_1}
	\kappa^{O(L^2)}
	\delta^{-2p_\star}
	\right).
	\]

	\paragraph{Total Complexity.}
	For every \(1\le\ell\le L-2\), \(\tau_1\le\tau_\ell\).
	The preceding stage bounds are
	\[
	\widetilde O_{k,L}\!\left(
	dr+
	\gamma^2r^{13}\kappa^{O(1)}\tau_1^{-2}
	\right)
	\]
	for the first layer,
	\[
	\widetilde O_{k,L}\!\left(
	r^9\kappa^{O(L)}\tau_1^{-2}
	+
	\gamma^2r^{14}\kappa^{O(L)}\tau_1^{-2}
	\right)
	\]
	for each intermediate layer, and \(O_{k,L}(r^2)\) for the
	final stage. Since \(r\ge1\), \(\gamma\ge1\),
	\(\tau_1<1\), and there are only \(O(L)\) stages, all
	hidden-dimensional contributions are dominated, up to
	\(\widetilde O_{k,L}\)-factors, by the convex-body term.
	Hence
	\[
	Q_{\mathrm{tot}}
	\le
	\widetilde O_{k,L}\!\left(
	dr+
	\gamma^2r^{14}\kappa^{O(L)}\tau_1^{-2}
	\right).
	\]
	Using \eqref{eq:simplified-tau1-bound}, absorbing the
	fixed width overhead into \(r^{O(L)}\), and observing that
	\(2g_1+2=O(L)\), we obtain
	\[
	Q_{\mathrm{tot}}
	=
	\widetilde O_{k,L}\!\left(
	dr+
	r^{L^2+O(L)}
	\gamma^{O(L)}
	\kappa^{O(L^2)}
	\delta^{-2p_\star}
	\right).
	\]
	This is \eqref{eq:main-pstar-complexity}. Since
	\(0<\delta<1\) and \(p_\star<\mathrm e\),
	\(\delta^{-2p_\star}\le\delta^{-2\mathrm e}\), proving
	the query bound in Theorem~\ref{thm:main}.

	\paragraph{Success Probability.}
	Each stage uses fresh randomness. Conditioning
	successively on all preceding
	successful stages, the relevant recovery proposition
	fails with probability at
	most \(\zeta_0\). Hence the union bound gives
	\[
	\mathbb P(\text{some recovery stage fails})
	\le
	N\zeta_0
	=
	0.01.
	\]
	Thus all parameter-recovery guarantees hold
	simultaneously with probability
	at least \(0.99\).
\end{proof}

\begin{lemma}[Gaussian stability under parameter perturbations]
	\label{lem:gaussian-parameter-stability}
	Let \(f^\star\) satisfy the normalization and positivity
	assumptions of
	Theorem~\ref{thm:main}, and let \(\widehat f\) have the
	same architecture.
	Suppose that, after a consistent relabeling and a
	suitable choice of the
	first-layer column signs,
	\[
	\max_{1\le\ell\le L-1}
	\|\widehat{\mathbf{W}}_\ell-\mathbf{W}_\ell\|_2
	\le\delta,
	\qquad
	\|\widehat{\mathbf{a}}-\mathbf{a}\|_1
	\le\delta,
	\]
	where \(0<\delta<1/2\). Assume that the recovered hidden
	weights are
	normalized as in the recovery algorithm: the columns of
	\(\widehat{\mathbf{W}}_1\) have unit Euclidean norm,
	while, for
	\(2\le\ell\le L-1\), the columns of
	\(\widehat{\mathbf{W}}_\ell\) are
	nonnegative and have unit \(\ell_1\)-norm. Then
	\[
	\mathbb E_{\mathbf{X}\sim\mathcal
	N(\mathbf{0},\mathbf{I}_d)}
	\left[
	\bigl(\widehat f(\mathbf{X})-f^\star(\mathbf{X})\bigr)^2
	\right]
	\le
	C_{k,L}r\delta^2.
	\]
\end{lemma}
	
\begin{proof}
	Write
	\(
	\mathbf{H}_1=\mathbf{W}_1^\top \mathbf{X},\widehat{\mathbf{H}}_1=\widehat{\mathbf{W}}_1^\top
	\mathbf{X},
	\)
	and, for \(2\le\ell\le L-1\),
	\(
	\mathbf{H}_\ell
	=
	\mathbf{W}_\ell^\top \mathbf{H}_{\ell-1}^{\odot k},
	\widehat{\mathbf{H}}_\ell
	=
	\widehat{\mathbf{W}}_\ell^\top
	\widehat{\mathbf{H}}_{\ell-1}^{\odot k}.
	\)
	We first record uniform moment bounds. For \(p\ge1\),
	define
	\[
	M_{\ell,p}
	:=
	\max_j
	\max
	\left\{
	\|(\mathbf{H}_\ell)_j\|_{L_p},
	\|(\widehat{\mathbf{H}}_\ell)_j\|_{L_p}
	\right\}.
	\]
	Because every column of \(\mathbf{W}_1\) and
	\(\widehat{\mathbf{W}}_1\) has unit Euclidean
	norm,
	\(
	(\mathbf{H}_1)_j,\,
	(\widehat{\mathbf{H}}_1)_j
	\sim
	\mathcal N(0,1),
	\)
	and hence
	\(
	M_{1,p}\le C_p.
	\)
	For \(\ell\ge2\), every column of both
	\(\mathbf{W}_\ell\) and
	\(\widehat{\mathbf{W}}_\ell\) is a probability vector.
	Therefore Minkowski's
	inequality gives
	\(
	M_{\ell,p}
	\le
	M_{\ell-1,kp}^k.
	\)
	Since \(k\) and \(L\) are fixed,
	\(
	M_{\ell,p}
	\le
	C_{k,L,p},
	1\le\ell\le L-1,
	\)
	for every fixed \(p<\infty\).
		
	Next define
	\(
	D_{\ell,p}
	:=
	\max_j
	\|(\widehat{\mathbf{H}}_\ell)_j-(\mathbf{H}_\ell)_j\|_{L_p}.
	\)
	For the first layer,
	\(
	(\widehat{\mathbf{H}}_1-\mathbf{H}_1)_j
	=
	(\widehat{\mathbf{w}}_{1,j}-\mathbf{w}_{1,j})^\top
	\mathbf{X}.
	\)
	Hence
	\[
	D_{1,p}
	\le
	C_p
	\max_j
	\|\widehat{\mathbf{w}}_{1,j}-\mathbf{w}_{1,j}\|_2
	\le
	C_p\delta.
	\]
		
	Fix \(2\le\ell\le L-1\). For every \(j\),
	\[
	\begin{aligned}
	(\widehat{\mathbf{H}}_\ell-\mathbf{H}_\ell)_j
	=
	(\widehat{\mathbf{w}}_{\ell,j}-\mathbf{w}_{\ell,j})^\top
	\widehat{\mathbf{H}}_{\ell-1}^{\odot k}
	+
	\mathbf{w}_{\ell,j}^\top
	\left(
	\widehat{\mathbf{H}}_{\ell-1}^{\odot k}
	-
	\mathbf{H}_{\ell-1}^{\odot k}
	\right).
	\end{aligned}
	\]
	The first term satisfies
	\[
	\begin{aligned}
	\left\|
	(\widehat{\mathbf{w}}_{\ell,j}-\mathbf{w}_{\ell,j})^\top
	\widehat{\mathbf{H}}_{\ell-1}^{\odot k}
	\right\|_{L_p}
	\le
	\|\widehat{\mathbf{w}}_{\ell,j}-\mathbf{w}_{\ell,j}\|_2
	\left\|
	\|\widehat{\mathbf{H}}_{\ell-1}^{\odot k}\|_2
	\right\|_{L_p}
	\le
	C_{k,L,p}\sqrt r\,\delta.
	\end{aligned}
	\]
	For the second term, using
	\(
	\mathbf{w}_{\ell,j}\ge\mathbf{0},
	\mathbf1^\top \mathbf{w}_{\ell,j}=1,
	\)
	we obtain
	\[
	\left\|
	\mathbf{w}_{\ell,j}^\top
	\left(
	\widehat{\mathbf{H}}_{\ell-1}^{\odot k}
	-
	\mathbf{H}_{\ell-1}^{\odot k}
	\right)
	\right\|_{L_p}
	\le
	\max_i
	\left\|
	(\widehat H_{\ell-1,i})^k
	-
	(H_{\ell-1,i})^k
	\right\|_{L_p}.
	\]
	Using
	\(
	|u^k-v^k|
	\le
	k|u-v|
	\left(
	|u|^{k-1}+|v|^{k-1}
	\right)
	\)
	and Hölder's inequality, together with the preceding
	uniform moment
	bounds, gives
	\[
	\left\|
	(\widehat H_{\ell-1,i})^k
	-
	(H_{\ell-1,i})^k
	\right\|_{L_p}
	\le
	C_{k,L,p}
	D_{\ell-1,2p}.
	\]
	Consequently,
	\(
	D_{\ell,p}
	\le
	C_{k,L,p}
	\left(
	\sqrt r\,\delta+D_{\ell-1,2p}
	\right).
	\)
	Iterating this inequality through the finitely many
	layers yields
	\[
	D_{\ell,p}
	\le
	C_{k,L,p}\sqrt r\,\delta,
	\qquad
	1\le\ell\le L-1.
	\]
		
	Finally,
	\[
	f^\star(\mathbf{X})
	=
	\mathbf{a}^\top \mathbf{H}_{L-1}^{\odot k},
	\qquad
	\widehat f(\mathbf{X})
	=
	\widehat{\mathbf{a}}^\top
	\widehat{\mathbf{H}}_{L-1}^{\odot k},
	\]
	and therefore
	\[
	\begin{aligned}
	\widehat f(\mathbf{X})-f^\star(\mathbf{X})
	=
	(\widehat{\mathbf{a}}-\mathbf{a})^\top
	\widehat{\mathbf{H}}_{L-1}^{\odot k}
	+
	\mathbf{a}^\top
	\left(
	\widehat{\mathbf{H}}_{L-1}^{\odot k}
	-
	\mathbf{H}_{L-1}^{\odot k}
	\right).
	\end{aligned}
	\]
	The first term satisfies
	\[
	\left\|
	(\widehat{\mathbf{a}}-\mathbf{a})^\top
	\widehat{\mathbf{H}}_{L-1}^{\odot k}
	\right\|_{L_2}
	\le
	\|\widehat{\mathbf{a}}-\mathbf{a}\|_1
	\max_i
	\|(\widehat H_{L-1,i})^k\|_{L_2}
	\le
	C_{k,L}\delta.
	\]
	Since \(\mathbf{a}\ge\mathbf{0}\) and \(\mathbf1^\top
	\mathbf{a}=1\),
	\[
	\begin{aligned}
	\left\|
	\mathbf{a}^\top
	\left(
	\widehat{\mathbf{H}}_{L-1}^{\odot k}
	-
	\mathbf{H}_{L-1}^{\odot k}
	\right)
	\right\|_{L_2}
	\le
	\max_i
	\left\|
	(\widehat H_{L-1,i})^k
	-
	(H_{L-1,i})^k
	\right\|_{L_2}
	\le
	C_{k,L}D_{L-1,4}
	\le
	C_{k,L}\sqrt r\,\delta.
	\end{aligned}
	\]
	Hence
	\[
	\|\widehat f-f^\star\|_{L_2(\mathcal
	N(\mathbf{0},\mathbf{I}_d))}
	\le
	C_{k,L}\sqrt r\,\delta,
	\]
	which proves
	\[
	\mathbb E_{\mathbf{X}\sim\mathcal
	N(\mathbf{0},\mathbf{I}_d)}
	\left[
	\bigl(\widehat f(\mathbf{X})-f^\star(\mathbf{X})\bigr)^2
	\right]
	\le
	C_{k,L}r\delta^2.
	\]
\end{proof}

\begin{proof}[Proof of Corollary~\ref{cor:gaussian_generalization_error}]
	Fix \(0<\varepsilon<1/2\), and choose
	\[
	\delta
	:=
	c_{k,L}\sqrt{\dfrac\varepsilon r},
	\]
	where \(c_{k,L}>0\) is sufficiently small that
	\(C_{k,L}r\delta^2\le\varepsilon\). After decreasing the
	constant once more
	if necessary, \(0<\delta<1/2\). Apply
	Theorem~\ref{thm:main} with this
	value of \(\delta\). With probability at least \(0.99\),
	the returned
	parameters satisfy the hypotheses of
	Lemma~\ref{lem:gaussian-parameter-stability}, and
	therefore
	\[
	\mathbb E_{\mathbf X\sim\mathcal N(\mathbf0,\mathbf I_d)}
	\left[(\widehat f(\mathbf X)-f^\star(\mathbf X))^2\right]
	\le
	C_{k,L}r\delta^2
	\le
	\varepsilon.
	\]

	By Theorem~\ref{thm:main}, the number of real
	function-value queries is
	\[
	\begin{aligned}
	\widetilde O_{k,L}\!\left(
	dr+
	r^{L^2+O(L)}
	\gamma^{O(L)}
	\kappa^{O(L^2)}
	\delta^{-2\mathrm e}
	\right)
	&=
	\widetilde O_{k,L}\!\left(
	dr+
	r^{L^2+O(L)+\mathrm e}
	\gamma^{O(L)}
	\kappa^{O(L^2)}
	\varepsilon^{-\mathrm e}
	\right)\\
	&=
	\widetilde O_{k,L}\!\left(
	dr+
	r^{L^2+O(L)}
	\gamma^{O(L)}
	\kappa^{O(L^2)}
	\varepsilon^{-\mathrm e}
	\right),
	\end{aligned}
	\]
	because the fixed factor \(r^{\mathrm e}\) is absorbed
	into \(r^{O(L)}\) for \(L\ge3\). This proves the
	corollary.
\end{proof}

\begin{proof}[Proof of Corollary~\ref{cor:relu}]
Fix \(0<\varepsilon<\dfrac12\) and run our recovery algorithm with
\(\delta=c_{k,L}\sqrt{\dfrac{\varepsilon}{r}}\), where
\(c_{k,L}>0\) is sufficiently small. On the successful recovery
event, which has probability at least \(0.99\),
Lemma~\ref{lem:gaussian-parameter-stability} gives
\(\mathbb E[(\widehat f(\mathbf X)-f^\star(\mathbf X))^2]
\le\dfrac{\varepsilon}{4}\).
Fix any returned parameters on this event. All expectations below
are over an independent
\(\mathbf X\sim\mathcal N(\mathbf0,\mathbf I_d)\).
The recovered first-layer columns have unit Euclidean norm, and
the higher hidden-layer columns are nonnegative and sum to one.
Moreover,
\(\|\widehat{\mathbf a}\|_1
\le\|\mathbf a\|_1+\delta\le2\);
no positivity assumption on \(\widehat{\mathbf a}\) is needed.

\paragraph{A Bounded Piecewise-Linear Activation.}
Write \(q:=k^{L-1}\), and let \(R\ge1\) and an even integer
\(M\ge2\) be chosen below. Let \(P_M\) be the linear interpolant
of \(t^k\) at the knots \(t_j=-1+\dfrac{2j}{M}\),
\(0\le j\le M\), extended by the constant \(1\) outside
\([-1,1]\). Since \(k\) is even, this extension is continuous
and \(0\le P_M(t)\le1\) on \(\mathbb R\).
The linear interpolation remainder, using
\(\sup_{|t|\le1}|(t^k)''|\le k(k-1)\), gives
\(\sup_{|t|\le1}|P_M(t)-t^k|\le C_kM^{-2}\).
If \(b_j\) is the jump in the slope of \(P_M\) at \(t_j\), then
\begin{equation}
    P_M(t)=1+\sum_{j=0}^{M}b_j\operatorname{ReLU}(t-t_j).
    \label{eq:relu-spline-representation}
\end{equation}
Indeed, both sides equal \(1\) to the left of all knots and have
the same slope on every interval.

Define \(f_{\mathrm{ReLU}}\) by replacing each power activation
in \(\widehat f\!\left(\dfrac{\mathbf x}{R}\right)\) with \(P_M\)
and multiplying the output by \(R^q\).
Without these replacements, homogeneity would give
\(R^q\widehat f\!\left(\dfrac{\mathbf x}{R}\right)
=\widehat f(\mathbf x)\).
All parameters of this construction are determined by the recovered
network and the interpolation knots.

\paragraph{Approximation On Bounded Features.}
Set \(Z_j:=\widehat{\mathbf w}_{1,j}^{\top}\mathbf X\) and
\(V:=\max_{1\le j\le r}|Z_j|\).
On \(\{V\le R\}\), the first preactivations of the rescaled
polynomial network lie in \([-1,1]\). Every higher hidden linear
map is a convex combination and is therefore nonexpansive in the
maximum norm. Thus both the original and modified activation
outputs remain in \([0,1]\) at every layer on this event.
Since \(t\mapsto t^k\) is \(k\)-Lipschitz on \([-1,1]\),
the maximum coordinatewise activation errors satisfy
\(e_1\le C_kM^{-2}\) and
\(e_\ell\le C_kM^{-2}+ke_{\ell-1}\) for
\(2\le\ell\le L-1\). Iteration and
\(\|\widehat{\mathbf a}\|_1\le2\) yield
\begin{equation}
    |f_{\mathrm{ReLU}}(\mathbf X)-\widehat f(\mathbf X)|
    \le C_{k,L}R^qM^{-2}
    \qquad\text{on }\{V\le R\}.
    \label{eq:relu-bounded-feature-error}
\end{equation}

\paragraph{Gaussian Tails and the Error Budget.}
The same column normalization gives
\(|\widehat f(\mathbf X)|\le2V^q\) everywhere, while the bounded
range of \(P_M\) gives
\(|f_{\mathrm{ReLU}}(\mathbf X)|\le2R^q\).
Hence the squared approximation error is at most \(16V^{2q}\)
on \(\{V>R\}\).
Each \(Z_j\) is standard Gaussian, although the coordinates need
not be independent. For a standard Gaussian variable \(Z\),
integrating its density and using
\(t^{2q}e^{-t^2/4}\le C_q\) gives
\(\mathbb E\left[|Z|^{2q}\mathbf1_{\{|Z|>R\}}\right]
\le C_qe^{-R^2/4}\) for \(R\ge1\).
Since
\(V^{2q}\mathbf1_{\{V>R\}}
\le\sum_{j=1}^r|Z_j|^{2q}\mathbf1_{\{|Z_j|>R\}}\),
combining this tail estimate with
\eqref{eq:relu-bounded-feature-error} gives
\begin{equation}
    \mathbb E\left[
        \bigl(f_{\mathrm{ReLU}}(\mathbf X)
        -\widehat f(\mathbf X)\bigr)^2
    \right]
    \le C_{k,L}\left(
        \dfrac{R^{2q}}{M^4}
        +r e^{-R^2/4}
    \right).
    \label{eq:relu-gaussian-approximation}
\end{equation}
Choose
\(R=\sqrt{C_{k,L}\log\!\left(\dfrac{2r}{\varepsilon}\right)}\)
and
\(M=2\left\lceil
C_{k,L}R^{q/2}\varepsilon^{-1/4}
\right\rceil\), with sufficiently large constants.
The two contributions in \eqref{eq:relu-gaussian-approximation}
are then at most \(\dfrac{\varepsilon}{8}\) each.
Consequently,
\begin{align*}
    \mathbb E\left[\bigl(f_{\mathrm{ReLU}}(\mathbf X)
        -f^\star(\mathbf X)\bigr)^2
    \right]
    \le 2\mathbb E\left[
        \bigl(f_{\mathrm{ReLU}}(\mathbf X)
        -\widehat f(\mathbf X)\bigr)^2
    \right]
    +2\mathbb E\left[\bigl(\widehat f(\mathbf X)-f^\star(\mathbf X)\bigr)^2
    \right]
    \le\varepsilon.
\end{align*}

\paragraph{Depth, Width, and Complexity.}
By \eqref{eq:relu-spline-representation}, each scalar replacement
uses \(M+1\) parallel ReLU units followed by an affine readout.
This readout is absorbed into the next layer's affine map, or into
the output layer at the final stage. The constant extensions are
already included in the same representation, so no additional
clipping layer is needed. The input and output rescalings are
absorbed into the first and final affine maps.
Thus the resulting network has exactly \(L-1\) hidden ReLU layers
and an affine output layer, with biases and widths bounded by
\begin{equation}
    \begin{aligned}
        r_\ell^{\mathrm{ReLU}}
        \le(M+1)r_\ell=O_{k,L}\!\left(r_\ell\varepsilon^{-1/4}
            \left[
                \log\left(\dfrac{2r}{\varepsilon}\right)
            \right]^{q/4}
        \right)
        =\widetilde O_{k,L}\left(r_\ell\varepsilon^{-1/4}
        \right).
    \end{aligned}
    \label{eq:relu-width-bound}
\end{equation}
For fixed \(k,L\), these parameters can be constructed in polynomial
time. No additional target queries or randomness are used, so the
success probability remains at least \(0.99\).
Finally, substituting
\(\delta=c_{k,L}\sqrt{\dfrac{\varepsilon}{r}}\)
into Theorem~\ref{thm:main} gives
\[
    \widetilde O_{k,L}\!\left(
        dr+r^{L^2+O(L)}
        \gamma^{O(L)}
        \kappa^{O(L^2)}
        \varepsilon^{-\mathrm e}
    \right)
\]
queries, with the additional factor \(r^{\mathrm e}\) absorbed
into \(r^{O(L)}\). This proves the corollary.
\end{proof}

\section{A Weight-Based Sufficient Condition for the Assumptions}

\label{sec:weight-based-assumptions}

We construct a class of networks to demonstrate that the assumptions in
Section~\ref{sec:problem-prelim} are not vacuous. Throughout this section, the hidden layers have a common width
\(r\), so that \(d=r_0\ge r_1=\cdots=r_{L-1}=r\ge3\) and \(r_L=1\). 

\begin{proposition}[Ordered diagonal separation]
\label{prop:weight-based-assumptions}
Fix an even integer \(k\ge4\) and a depth \(L\ge3\). There exist
constants \(c_0,c_1\in(0,1]\), depending only on \(k,L\), with the
following property. Let
\(\mathbf W_1\in\mathbb R^{d\times r}\) satisfy
\(\mathbf W_1^\top\mathbf W_1=\mathbf I_r\), and set
\(\mathbf W_L=\mathbf a=\dfrac{1}{r}\mathbf1_r\).
For each \(2\le m\le L-1\), choose the off-diagonal entries of
\(\mathbf W_m\in\mathbb R^{r\times r}\) to satisfy
\((\mathbf W_m)_{ij}=(\mathbf W_m)_{ji}>0\) for \(i\ne j\), and set
\((\mathbf W_m)_{ii}=1-\sum_{j\ne i}(\mathbf W_m)_{ij}\).
Suppose that, in these same coordinate labels at all hidden layers,
there are \(\varepsilon,\Delta>0\) such that, for every
\(2\le m\le L-1\),
\begin{equation}
    \max_{1\le i\le r}
    \bigl(1-(\mathbf W_m)_{ii}\bigr)\le\varepsilon,
    \qquad
    (\mathbf W_m)_{ii}-(\mathbf W_m)_{i+1,i+1}\ge\Delta
    \quad (1\le i<r),
    \label{eq:wc-diagonal-conditions}
\end{equation}
and
\begin{equation}
    r^2\varepsilon^2\le c_0\Delta.
    \label{eq:wc-small-coupling}
\end{equation}
Then Assumptions~\ref{ass:input-convexity},~\ref{ass:normalization},~\ref{ass:conditioning} and~\ref{ass:spectral-gap} all hold,
with the choices
\begin{equation}
    \kappa_0=2,
    \qquad
    \mu=1,
    \qquad
    \gamma=\dfrac{1}{c_1\Delta}.
    \label{eq:wc-assumption-parameters}
\end{equation}
In particular, every hidden weight matrix has full column rank, and
\begin{equation}
    \dfrac{\min_{i\ne j}|\theta_{\ell,i}-\theta_{\ell,j}|}
    {\max_i\theta_{\ell,i}}
    \ge c_1\Delta,
    \qquad 1\le\ell\le L-2.
    \label{eq:wc-relative-gap}
\end{equation}
\end{proposition}

The next three subsections prove the proposition. All constants denoted
by \(C\) below depend only on \(k,L\), may increase from line to line,
and are independent of \(d,r,\varepsilon,\Delta\), and the particular
weights. The constant \(c_0\) is chosen sufficiently small at the end.

\subsection{Normalization and Nonnegativity}
\label{subsec:wc-normalization}

The identity \(\mathbf W_1^\top\mathbf W_1=\mathbf I_r\) gives
\(\|\mathbf w_{1,j}\|_2=1\) for every \(j\). The definition of the
higher-layer diagonal entries gives \(\mathbf W_m\mathbf1_r=\mathbf1_r\).
Their symmetry therefore also gives
\(\mathbf1_r^\top\mathbf W_m=\mathbf1_r^\top\), so every column is
normalized as required.

To check positivity of the diagonal entries, telescope the ordered
gaps in any one higher hidden layer. Since every off-diagonal entry is
positive, \(0<1-(\mathbf W_m)_{ii}\le\varepsilon\), and hence
\((r-1)\Delta\le\varepsilon\).
Together with~\eqref{eq:wc-small-coupling}, this implies
\begin{equation}
    \varepsilon\le\dfrac{c_0}{r^2(r-1)}
    \le\dfrac{1}{18}<\dfrac{1}{4}.
    \label{eq:wc-implied-smallness}
\end{equation}
Consequently \((\mathbf W_m)_{ii}\ge1-\varepsilon>0\), and all the
higher hidden matrices are strictly positive and doubly stochastic.
Finally, \(a_i=\dfrac{1}{r}>0\) and
\(\mathbf1_r^\top\mathbf a=1\), which verifies the output-layer
requirements as well. Notice also that \(0<\Delta\le\varepsilon<1\);
thus the value of \(\gamma\) in~\eqref{eq:wc-assumption-parameters}
is at least one.

\subsection{Conditioning and Relative Row Balance}
\label{subsec:wc-conditioning}

All singular values of \(\mathbf W_1\) equal one. For
\(2\le m\le L-1\), symmetry and the definition of the diagonal give
\(\|\mathbf W_m-\mathbf I_r\|_1
=\|\mathbf W_m-\mathbf I_r\|_\infty\le2\varepsilon\), and therefore
\(\|\mathbf W_m-\mathbf I_r\|_2\le2\varepsilon\).
Moreover, double stochasticity gives \(\|\mathbf W_m\|_2\le1\), while
\(\mathbf W_m\mathbf1_r=\mathbf1_r\) gives the reverse inequality.
It follows that \(\sigma_{\max}(\mathbf W_m)=1\) and
\(\sigma_{\min}(\mathbf W_m)\ge1-2\varepsilon>0\). Thus
\(\operatorname{cond}(\mathbf W_m)
\le\dfrac{1}{1-2\varepsilon}\le2\).
This proves full column rank and the conditioning bound with
\(\kappa_0=2\).

For the row-balance condition, the higher hidden layers satisfy
\((\mathbf W_m\mathbf1_r)_i=1
=\dfrac{r_m}{r_{m-1}}\) for every \(i\).
At the output layer,
\((\mathbf W_L\mathbf1_1)_i=a_i
=\dfrac{1}{r}=\dfrac{r_L}{r_{L-1}}\).
Hence the relative row-balance lower bound holds with equality at
every relevant layer when \(\mu=1\).

\subsection{Uniform Relative Spectral Gap}
\label{subsec:wc-spectral-gap}

\begin{proof}[Proof of the spectral-gap assertion]
Fix \(1\le\ell\le L-2\), and write \(q=k^{L-\ell}\).
For \(m=\ell+1,\ldots,L-1\), consider the interpolation
\(\mathbf W_m(t)=\mathbf I_r+
\dfrac{t}{\varepsilon}(\mathbf W_m-\mathbf I_r)\),
\(0\le t\le1\).
Every such matrix is nonnegative and doubly stochastic: its diagonal
entries are at least \(1-t\), and its off-diagonal entries are
nonnegative. Let \(h_t^{(\ell)}\) be the corresponding suffix,
with the output weights kept equal to \(\dfrac{1}{r}\mathbf1_r\),
and set \(G_t=rh_t^{(\ell)}\).
Then \(G_0(\mathbf u)=Q(\mathbf u):=\sum_{j=1}^r u_j^q\), and
\(G_\varepsilon=rh^{(\ell)}\).
Write \(\theta_{\ell,i}(t)\) for the spectral signature defined from
\(h_t^{(\ell)}\) exactly as in Assumption~\ref{ass:spectral-gap}.
In particular, \(\theta_{\ell,i}(\varepsilon)=\theta_{\ell,i}\).
At \(t=0\), permutation symmetry makes these signatures equal;
denote their common positive value by \(\theta_\ell^{(0)}\).

\paragraph{Exponential-Integral Representation.}
A positive coordinate rescaling does not change a spectral signature.
Indeed, replacing \(g(\mathbf u)\) by
\(g(s_1u_1,\ldots,s_ru_r)\), with all \(s_j>0\), multiplies its
\(i\)-th position moment by \(s_i^{-2}\) and its \(i\)-th gradient
moment by \(s_i^2\), by the corresponding change of variables in the
uniform distribution. Multiplying a degree-\(q\) homogeneous
polynomial by a positive scalar is a special case of this rescaling.
Consequently, \(G_t\) and \(h_t^{(\ell)}\) have the same signatures.

For a positive-definite homogeneous polynomial \(g\) of degree \(q\)
and a homogeneous polynomial \(p\), integration over the sublevel
sets of \(g\) gives
\begin{equation}
    \int_{\mathbb R^r}p(\mathbf u)e^{-g(\mathbf u)}\,d\mathbf u
    =
    \Gamma\!\left(1+\dfrac{r+\deg p}{q}\right)
    \int_{\{g\le1\}}p(\mathbf u)\,d\mathbf u.
    \label{eq:wc-homogeneous-integral}
\end{equation}
Indeed, write
\(e^{-g(\mathbf u)}=\int_{g(\mathbf u)}^\infty e^{-v}\,dv\)
and interchange the integrals. In the inner integral over
\(\{g\le v\}\), rescale
\(\mathbf u=v^{\dfrac{1}{q}}\mathbf w\).
The remaining scalar integral is the displayed gamma factor.
Absolute integrability justifies the same argument when \(p\) is
not nonnegative.

Let \(Z_t=\int_{\mathbb R^r}e^{-G_t(\mathbf u)}\,d\mathbf u\), and
write \(\mathbb E_t\) and \(\operatorname{Cov}_t\) for expectation
and covariance under the density \(Z_t^{-1}e^{-G_t}\).
Define \(A_i(t)=\mathbb E_t[X_i^2]\) and
\(B_i(t)=\mathbb E_t[\partial_{ii}G_t(\mathbf X)]\).
Integration by parts gives
\(\mathbb E_t[(\partial_iG_t(\mathbf X))^2]=B_i(t)\).
All these integrals are well defined for \(t\) near zero; a uniform
coercive bound is established below.
At \(t=0\), the density is a product density, with one-dimensional
moments
\(m_q(v):=\mathbb E_0[|X_i|^v]
=\dfrac{\Gamma\!\left(\dfrac{v+1}{q}\right)}
{\Gamma\!\left(\dfrac{1}{q}\right)}\) for \(v\ge0\).
Thus \(A_i(0)=m_q(2)\) and
\(B_i(0)=q(q-1)m_q(q-2)\).
Applying~\eqref{eq:wc-homogeneous-integral} to
\(p=1\), \(p=u_i^2\), and \(p=(\partial_i g)^2\), the gamma factors
cancel between \(t\) and zero, leaving
\begin{equation}
    \dfrac{\theta_{\ell,i}(t)}{\theta_\ell^{(0)}}
    =
    \dfrac{A_i(t)B_i(t)}{m_q(2)\,q(q-1)m_q(q-2)}.
    \label{eq:wc-moment-product}
\end{equation}

\paragraph{First-Order Response To the Weights.}
A dot denotes differentiation with respect to \(t\).
Differentiating the normalized exponential integrals at zero gives
\begin{equation}
    \begin{aligned}
        A_i'(0)
        &=-\operatorname{Cov}_0(X_i^2,\dot G_0(\mathbf X)),\\
        B_i'(0)
        &=\mathbb E_0[\partial_{ii}\dot G_0(\mathbf X)]
        -q(q-1)\operatorname{Cov}_0
        (X_i^{q-2},\dot G_0(\mathbf X)).
    \end{aligned}
    \label{eq:wc-moment-derivatives}
\end{equation}
By~\eqref{eq:wc-moment-product}, the normalized derivative of the
signature is
\(\dfrac{A_i'(0)}{m_q(2)}+
\dfrac{B_i'(0)}{q(q-1)m_q(q-2)}\), which is linear in \(\dot G_0\).
If a term in \(\dot G_0\) is \(u_i^s u_j^{q-s}\), with \(i\ne j\)
and even \(s,q-s\ge4\), its contribution to this derivative is
\begin{equation}
    \chi_q(s)
    :=m_q(q-s)\left[
        2m_q(s)-\dfrac{m_q(s+2)}{m_q(2)}
        -\dfrac{(s-1)(q-s-1)m_q(s-2)}
        {q(q-1)m_q(q-2)}
    \right].
    \label{eq:wc-single-edge-response}
\end{equation}
Indeed, insert the monomial into~\eqref{eq:wc-moment-derivatives},
use independence, and simplify with
\(m_q(v+q)=\dfrac{v+1}{q}m_q(v)\).
The contribution to coordinate \(j\) is \(\chi_q(q-s)\), and the
contributions to all other coordinates are zero. A pure-power term
\(u_j^q\) contributes zero to every coordinate, since changing its
coefficient in \(Q\) is a coordinate rescaling.

We next check the sign in~\eqref{eq:wc-single-edge-response}.
The gamma function is log-convex, by H\"older's inequality applied
to its defining integral. Therefore
\(x\mapsto\dfrac{\Gamma(x+a)}{\Gamma(x)}\) is nondecreasing for
fixed \(a>0\). Using this fact with \(a=\dfrac{2}{q}\), together
with the gamma recurrence, gives
\begin{equation}
    \begin{aligned}
        \dfrac{m_q(s+2)}{m_q(2)m_q(s)}
        ={\dfrac{3(s+1)}{s+3}}
        \dfrac{
            \Gamma\!\left(1+\dfrac{s+3}{q}\right)
            \Gamma\!\left(1+\dfrac{1}{q}\right)
        }{
            \Gamma\!\left(1+\dfrac{s+1}{q}\right)
            \Gamma\!\left(1+\dfrac{3}{q}\right)
        }
        \ge\dfrac{3(s+1)}{s+3}>2
        \qquad (s\ge4).
    \end{aligned}
    \label{eq:wc-response-sign}
\end{equation}
The first two terms in the bracket
in~\eqref{eq:wc-single-edge-response} therefore have negative sum,
and its last term is also negative. Hence \(\chi_q(s)<0\)
whenever \(4\le s\le q-4\) and \(s\) is even.

For \(\ell+1\le m\le L-1\), write \(s_m=k^{m-\ell}\).
Differentiating the suffix recursion at the identity matrices yields
\begin{equation}
    \dot G_0(\mathbf u)
    =\dfrac{1}{\varepsilon}
    \sum_{m=\ell+1}^{L-1}\dfrac{q}{s_m}
    \left[
        \sum_{i\ne j}(\mathbf W_m)_{ij}u_i^{s_m}u_j^{q-s_m}
        -\sum_i\bigl(1-(\mathbf W_m)_{ii}\bigr)u_i^q
    \right].
    \label{eq:wc-network-first-variation}
\end{equation}
For completeness, at the identity matrices the input to
\(\mathbf W_m^\top\) has coordinates \(u_i^{s_m}\); the remaining
network is the sum of the \(\dfrac{q}{s_m}\)-th powers of its
coordinates. Differentiating one matrix entry gives precisely the
corresponding term in~\eqref{eq:wc-network-first-variation}.

Symmetry makes both the off-diagonal row sum and the off-diagonal
column sum at coordinate \(i\) equal to \(1-(\mathbf W_m)_{ii}\).
Thus the single-edge calculation gives
\begin{equation}
    \left.
    \dfrac{d}{dt}\dfrac{\theta_{\ell,i}(t)}{\theta_\ell^{(0)}}
    \right|_{t=0}
    =-\dfrac{1}{\varepsilon}
    \sum_{m=\ell+1}^{L-1}
    b_{\ell m}\bigl(1-(\mathbf W_m)_{ii}\bigr),
    \qquad
    b_{\ell m}:=-\dfrac{q}{s_m}
    \bigl(\chi_q(s_m)+\chi_q(q-s_m)\bigr)>0.
    \label{eq:wc-first-order-diagonals}
\end{equation}
Here \(s_m\ge k\ge4\) and \(q-s_m\ge4\), so the strict sign follows
from~\eqref{eq:wc-response-sign}. The finitely many coefficients
\(b_{\ell m}\) depend only on \(k,L,\ell,m\), not on the width or
the weights.

\paragraph{A Uniform Second-Order Remainder.}
We prove that, for some \(\delta>0\) depending only on \(k,L\),
\begin{equation}
    \left|
        \dfrac{d^2}{dt^2}
        \dfrac{\theta_{\ell,i}(t)}{\theta_\ell^{(0)}}
    \right|
    \le Cr^2,
    \qquad 0\le t\le\dfrac{\delta}{r}.
    \label{eq:wc-second-derivative}
\end{equation}
This estimate does not use any spectral separation.

We first bound the polynomial and its derivatives. Every nonnegative
doubly stochastic matrix is a contraction in each \(\ell_p\) norm,
\(1\le p<\infty\): Jensen's inequality gives
\(\sum_j|(\mathbf W^\top\mathbf v)_j|^p
\le\sum_{i,j}W_{ij}|v_i|^p=\sum_i|v_i|^p\).
Applying this observation through the suffix, or inducting directly
through its recursion, gives \(0\le G_\tau(\mathbf u)\le Q(\mathbf u)\)
for \(0\le\tau\le1\). At the all-ones vector every hidden coordinate
remains one, so
\(G_\tau(\mathbf1_r)=r\) and
\(\partial_iG_\tau(\mathbf1_r)=q\).

The degree of \(G_t\) as a polynomial in \(t\) is at most
\(\sum_{m=\ell+1}^{L-1}\dfrac{q}{s_m}
=\dfrac{q-k}{k-1}\), which depends only on \(k,L\).
Fixed-node univariate interpolation on \([0,1]\) therefore expresses
each of \(\dot G_t\) and \(\ddot G_t\) as a linear combination of
finitely many polynomials \(G_\tau\), with the sum of the absolute
coefficients at most \(C\), uniformly for \(t\in[0,1]\).
Consequently,
\begin{equation}
    |\dot G_t(\mathbf u)|+|\ddot G_t(\mathbf u)|\le CQ(\mathbf u),
    \qquad
    (1-Ct)Q(\mathbf u)\le G_t(\mathbf u)\le Q(\mathbf u).
    \label{eq:wc-polynomial-bounds}
\end{equation}
The lower bound follows by integrating the bound on \(\dot G_t\)
from zero.

We also need local derivative bounds without a factor of \(r\).
Each \(G_\tau\) has nonnegative coefficients and is homogeneous of
degree \(q\). Its second derivative \(\partial_{ii}G_\tau\) thus
has nonnegative coefficients whose sum is
\(\partial_{ii}G_\tau(\mathbf1_r)
\le(q-1)\partial_iG_\tau(\mathbf1_r)=q(q-1)\).
Applying the same interpolation formula after differentiating in
\(u_i\) twice shows that, for \(j=0,1,2\), the sum of the absolute
coefficients of \(\partial_{ii}\partial_t^jG_t\) is at most \(C\).
All these polynomials have degree \(q-2\).

Choose \(\delta\) small enough that \(Ct\le\dfrac{1}{2}\) when
\(0\le t\le\dfrac{\delta}{r}\).
By~\eqref{eq:wc-polynomial-bounds} and homogeneity,
\(1\le\dfrac{Z_t}{Z_0}\le(1-Ct)^{-\dfrac{r}{q}}\le C\).
For a polynomial \(p\) of bounded degree whose sum of absolute
coefficients is bounded independently of \(r\), we then have
\begin{equation}
    \mathbb E_t\bigl[|p(\mathbf X)|Q(\mathbf X)^j\bigr]
    \le Cr^j,
    \qquad j=0,1,2.
    \label{eq:wc-weighted-moment-bound}
\end{equation}
Here the bound is uniform for the polynomials just considered.
To verify it explicitly, take a monomial with multi-index
\(\alpha\), write
\(|\mathbf u|^\alpha=\prod_{i=1}^r|u_i|^{\alpha_i}\), and use
\begin{equation}
    \dfrac{1}{Z_0}
    \int_{\mathbb R^r}
    |\mathbf u|^\alpha Q(\mathbf u)^j e^{-aQ(\mathbf u)}\,d\mathbf u
    =
    a^{-\dfrac{r+|\alpha|}{q}-j}
    \dfrac{\Gamma\!\left(\dfrac{r+|\alpha|}{q}+j\right)}
    {\Gamma\!\left(\dfrac{r+|\alpha|}{q}\right)}
    \prod_{i=1}^r m_q(\alpha_i).
    \label{eq:wc-monomial-moment}
\end{equation}
For \(j=0\) this is a product of one-dimensional integrals;
differentiating in \(a\) gives the identities for \(j=1,2\).
With \(a=1-Ct\), the power of \(a\) is uniformly bounded because
\(rt\le\delta\), the gamma ratio is \(O_{k,L}(r^j)\), and the last
product is bounded because \(|\alpha|\) is bounded. Summing over
monomials proves~\eqref{eq:wc-weighted-moment-bound}.

To differentiate the moments, use the following identities, with all
integrals taken over \(\mathbb R^r\):
\begin{equation}
    \begin{aligned}
        \dfrac{d}{dt}\int p_t e^{-G_t}\,d\mathbf u
        &=\int(\dot p_t-p_t\dot G_t)e^{-G_t}\,d\mathbf u,\\
        \dfrac{d^2}{dt^2}\int p_t e^{-G_t}\,d\mathbf u
        &=\int\left[
            \ddot p_t-2\dot p_t\dot G_t
            +p_t\bigl((\dot G_t)^2-\ddot G_t\bigr)
        \right]e^{-G_t}\,d\mathbf u.
    \end{aligned}
    \label{eq:wc-integral-derivatives}
\end{equation}
The coercive lower bound in~\eqref{eq:wc-polynomial-bounds} and the
polynomial envelopes above justify differentiation under the
integral sign, including the derivatives used earlier at zero.
Apply~\eqref{eq:wc-integral-derivatives} to
\(p_t=1\), \(p_t=u_i^2\), and \(p_t=\partial_{ii}G_t\).
For each of these unnormalized integrals, the absolute value of its
\(j\)-th derivative, divided by \(Z_t\), is at most \(Cr^j\) for
\(j=0,1,2\), by~\eqref{eq:wc-weighted-moment-bound}.
The quotient rule, using
\(\dfrac{|Z_t'|}{Z_t}\le Cr\) and
\(\dfrac{|Z_t''|}{Z_t}\le Cr^2\), therefore gives
\(|A_i^{(j)}(t)|+|B_i^{(j)}(t)|\le Cr^j\) for \(j=0,1,2\).
The denominator in~\eqref{eq:wc-moment-product} is a positive constant
depending only on \(q\). The product rule now
proves~\eqref{eq:wc-second-derivative}.

By~\eqref{eq:wc-implied-smallness}, decreasing \(c_0\) ensures
\(\varepsilon\le\dfrac{\delta}{r}\). Taylor's theorem and
\eqref{eq:wc-first-order-diagonals} consequently yield the uniform
expansion
\begin{equation}
    \dfrac{\theta_{\ell,i}}{\theta_\ell^{(0)}}
    =1-\sum_{m=\ell+1}^{L-1}
    b_{\ell m}\bigl(1-(\mathbf W_m)_{ii}\bigr)
    +O_{k,L}(r^2\varepsilon^2).
    \label{eq:wc-signature-expansion}
\end{equation}

\paragraph{From Ordered Diagonals to A Relative Gap.}
Since there are only finitely many positive coefficients
\(b_{\ell m}\), there is a constant \(b>0\), depending only on
\(k,L\), such that
\(\sum_{m=\ell+1}^{L-1}b_{\ell m}\ge b\) for every
\(1\le\ell\le L-2\).
For \(i<j\), the common ordering gives
\((\mathbf W_m)_{ii}-(\mathbf W_m)_{jj}\ge(j-i)\Delta\ge\Delta\)
in every higher layer. Subtracting the two instances
of~\eqref{eq:wc-signature-expansion} therefore gives
\begin{equation}
    \dfrac{\theta_{\ell,i}-\theta_{\ell,j}}{\theta_\ell^{(0)}}
    \ge b\Delta-Cr^2\varepsilon^2
    \ge(b-Cc_0)\Delta
    \ge\dfrac{b\Delta}{2},
    \label{eq:wc-absolute-gap}
\end{equation}
where the last inequality follows by choosing \(c_0\) sufficiently
small. The sign of the first-order term
in~\eqref{eq:wc-signature-expansion} also gives
\(\max_i\theta_{\ell,i}\le
(1+Cr^2\varepsilon^2)\theta_\ell^{(0)}\le2\theta_\ell^{(0)}\),
after the same choice of \(c_0\). Dividing the last two bounds proves
\eqref{eq:wc-relative-gap} with
\(c_1=\min\{1,\dfrac{b}{4}\}\), uniformly over all
\(1\le\ell\le L-2\). This completes the proof of
Proposition~\ref{prop:weight-based-assumptions}.
\end{proof}

\subsection{An Explicit Nonempty Family}
\label{subsec:wc-explicit-family}

The conditions hold for every \(r\ge3\), not just for a particular
small width. Choose any \(\mathbf W_1\) with orthonormal columns,
set \(\mathbf a=\dfrac{1}{r}\mathbf1_r\), and, for
\(2\le m\le L-1\), take
\begin{equation}
    (\mathbf W_m)_{ij}
    =
    \begin{cases}
        \dfrac{t(i+j)}{r^2}, & i\ne j,\\[2mm]
        1-\dfrac{t}{r^2}
        \left((r-2)i+\dfrac{r(r+1)}{2}\right), & i=j,
    \end{cases}
    \qquad
    0<t\le\dfrac{c_0}{3r^3}.
    \label{eq:wc-explicit-weights}
\end{equation}
These matrices are symmetric, and the diagonal entries are precisely
one minus the corresponding off-diagonal row sums.
Their largest off-diagonal row sum is
\(\varepsilon=\dfrac{3(r-1)t}{2r}\), and their consecutive diagonal
gaps equal \(\Delta=\dfrac{(r-2)t}{r^2}\).
Moreover,
\begin{equation}
    \dfrac{r^2\varepsilon^2}{\Delta}
    =\dfrac{9t r^2(r-1)^2}{4(r-2)}
    \le3tr^3\le c_0,
    \label{eq:wc-explicit-smallness}
\end{equation}
where the first inequality uses
\(3(r-1)^2\le4r(r-2)\) for \(r\ge3\).
Thus Proposition~\ref{prop:weight-based-assumptions} applies.
In particular, setting \(t=\dfrac{c_0}{3r^3}\) gives
\(\Delta=\dfrac{c_0(r-2)}{3r^5}\) and
\(\gamma=\dfrac{3r^5}{c_0c_1(r-2)}=O_{k,L}(r^4)\),
while \(\kappa_0=2\) and \(\mu=1\).

The first-layer orthonormality is used only for its normalization and
conditioning, not for the spectral-gap argument. It can be replaced
by unit-norm columns and any prescribed bound on
\(\operatorname{cond}(\mathbf W_1)\); the conclusion then uses
\(\kappa_0=\max\{2,\operatorname{cond}(\mathbf W_1)\}\).
The restriction to \(r\ge3\) is natural for this construction:
a symmetric doubly stochastic \(2\times2\) matrix necessarily has
equal diagonal entries. Finally, the condition is local:
\eqref{eq:wc-implied-smallness} requires coupling of order at most
\(r^{-3}\) for fixed \(k,L\). No spectral-gap claim is made here for
symmetric matrices outside this weak-coupling regime.

\section{Proof of Challenges}
\label{app:Proof_of_Challenge}
\subsection{Local Minima and Fourth-Order Saddle Barriers}
\label{sec:direct-fitting-barriers}

\begin{proof}[Proof of Proposition~\ref{prop:fourth-order-barrier}]
We give two explicit realizable instances. In both cases, let
$f^\star(\mathbf{x})
=
\frac12(\mathbf{e}_1^\top\mathbf{x})^4
+
\frac12(\mathbf{e}_3^\top\mathbf{x})^4$.
Thus the teacher has
$\mathbf{W}_1^\star=[\mathbf{e}_1,\mathbf{e}_3]$ and
$\mathbf{a}^\star=(1/2,1/2)$.
Its first-layer columns are orthonormal, and its output weights are positive,
normalized, and row balanced, so all assumptions applicable to the two-layer
case are satisfied.

\paragraph{A Non-Global Local Minimum.}
Consider the student point
$\bar{\mathbf{w}}_1=\mathbf{e}_2$,
$\bar{\mathbf{w}}_2=\mathbf{e}_3$, and
$\bar{\mathbf{a}}=(1/2,1/2)$.
We work in the normalized parameterization and use the local coordinates
$\boldsymbol{\xi}=(p,q,s,t,u)$ defined by
$\mathbf{w}_1=p\mathbf{e}_1+\sqrt{1-p^2-q^2}\mathbf{e}_2+q\mathbf{e}_3$,
$\mathbf{w}_2=s\mathbf{e}_1+t\mathbf{e}_2+\sqrt{1-s^2-t^2}\mathbf{e}_3$,
$a_1=1/2+u$, and $a_2=1/2-u$.
The point $\bar{\boldsymbol{\theta}}$ corresponds to
$\boldsymbol{\xi}=\mathbf{0}$.

Let $c=2^{1/4}$ and introduce the four samples
$\mathcal G=
\{\mathbf{e}_2,-\mathbf{e}_2,
c\mathbf{e}_1+\mathbf{e}_2,
-c\mathbf{e}_1+\mathbf{e}_2\}$.
At $\boldsymbol{\xi}=\mathbf{0}$, their residuals are respectively
$1/2,1/2,-1/2,-1/2$.
For a point $\mathbf{x}=(x_1,x_2,0)\in\mathcal G$, the first variation of its
residual is
$D r_{\mathbf{x}}(\mathbf{0})[\boldsymbol{\xi}]
=
2x_2^3x_1p+x_2^4u$.
Hence
$\sum_{\mathbf{x}\in\mathcal G}
r_{\mathbf{x}}D r_{\mathbf{x}}(\mathbf{0})=0$,
so these four samples contribute zero gradient at
$\bar{\boldsymbol{\theta}}$.

We next add exact-fit samples to make this stationary point locally stable.
Let
$$\mathcal A=
\{(1,1,0),(-1,1,0),(1,1,1),(-1,1,1),(1,1,2)\}$$.
For every $\mathbf{x}\in\mathcal A$, one has
$(\mathbf{e}_1^\top\mathbf{x})^4
=
(\mathbf{e}_2^\top\mathbf{x})^4$,
and therefore the teacher and the student agree exactly at
$\bar{\boldsymbol{\theta}}$.
Writing an anchor point as $(\varepsilon,1,z)$ with
$\varepsilon\in\{\pm1\}$, its residual Jacobian at
$\boldsymbol{\xi}=\mathbf{0}$ is
$(2\varepsilon,\,2z,\,2\varepsilon z^3,\,2z^3,\,1-z^4)$.
For the five points above, the resulting Jacobian matrix is
$$\mathbf{J}=
\begin{pmatrix}
2&0&0&0&1\\
-2&0&0&0&1\\
2&2&2&2&0\\
-2&2&-2&2&0\\
2&4&16&16&-15
\end{pmatrix},$$
whose determinant is $384$.
Thus $\mathbf{J}$ is invertible and
$\mathbf{J}^\top\mathbf{J}\succ0$.

Form an empirical data multiset by taking the four points in $\mathcal G$ once
and each point in $\mathcal A$ exactly $M$ times.
Let $\mathbf{H}_{\mathcal G}$ denote the Hessian contribution of the four
points in $\mathcal G$ at $\boldsymbol{\xi}=\mathbf{0}$.
Since all anchor residuals vanish, their Hessian contribution is precisely
$\mathbf{J}^\top\mathbf{J}$ per copy, up to the common positive normalization
of the empirical loss. Therefore the total Hessian is proportional to
$\mathbf{H}_{\mathcal G}+M\mathbf{J}^\top\mathbf{J}$.
Choosing
$M>\|\mathbf{H}_{\mathcal G}\|_{\mathrm{op}}/
\lambda_{\min}(\mathbf{J}^\top\mathbf{J})$
makes this matrix positive definite.
Hence $\bar{\boldsymbol{\theta}}$ is a strict local minimum in the normalized
parameterization.

It is not global: its four residuals on $\mathcal G$ are nonzero, so its
empirical loss is strictly positive, whereas the teacher parameters achieve
zero loss on every sample.
This proves the existence of a non-global local minimum.

\paragraph{A Fourth-Order Saddle.}
For the second instance, use only the two samples
$\mathbf{x}_1=\mathbf{e}_1$ and $\mathbf{x}_2=\mathbf{e}_3$ and consider the
student point
$\widetilde{\mathbf{w}}_1=\mathbf{e}_1$,
$\widetilde{\mathbf{w}}_2=\mathbf{e}_2$, and
$\widetilde{\mathbf{a}}=(1/2,1/2)$.
The two residuals are $r_1=0$ and $r_2=-1/2$.
Since both student preactivations vanish at $\mathbf{x}_2$, the quartic
activation gives
$D r_2=D^2 r_2=D^3 r_2=0$
at this point.
For the empirical squared loss
$\mathcal L=\frac14(r_1^2+r_2^2)$, it follows that
$\nabla\mathcal L(\widetilde{\boldsymbol{\theta}})=0$ and
$\nabla^2\mathcal L(\widetilde{\boldsymbol{\theta}})
=
\frac12\nabla r_1\nabla r_1^\top\succeq0$.

Moreover, if a direction lies in the nullspace of this Hessian, then the
first-order variation of $r_1$ vanishes, so $r_1^2$ changes only at fourth
order, while $r_2=-1/2+O(\|\boldsymbol{\delta}\|^4)$.
Hence no descent term of order one, two, or three can occur along a Hessian-flat
direction.

Finally, perturb only the second hidden neuron along the feasible normalized
path
$\mathbf{w}_2(t)=\sqrt{1-t^2}\mathbf{e}_2+t\mathbf{e}_3$.
The first sample remains fitted exactly, while on the second sample the student
output is $\frac12t^4$.
Consequently,
$\mathcal L(\widetilde{\boldsymbol{\theta}}(t))
=
\frac14(-\frac12+\frac12t^4)^2
=
\frac1{16}-\frac18t^4+\frac1{16}t^8$.
Thus the loss decreases for all sufficiently small nonzero $t$, and its first
descent term is quartic.
Since the Hessian also has positive directions, this point is a genuine
fourth-order saddle.
\end{proof}
\subsection{Sample-Complexity Barrier for the NTK}
\label{subsec:ntk-gaussian-lower-bound}

We prove Proposition~\ref{prop:ntk-gaussian-lower-bound} by applying
the spherical kernel lower bound of
\citet[Theorem~3]{ghorbani2020linearized}.
The Gaussian-to-sphere reduction and the verification of a
nonvanishing highest-degree component for our admissible targets
are supplied below.

\paragraph{The Spherical Result.}
Let \(\nu_d=\operatorname{Unif}(S^{d-1}(\sqrt d))\).
Write \(P_{\le m}\) for the orthogonal projection in
\(L^2(\nu_d)\) onto restrictions of polynomials of total degree
at most \(m\), and set \(P_{>m}=I-P_{\le m}\).
For a fixed rotationally invariant kernel
\[
    K_d(\mathbf u,\mathbf v)
    =\psi_d\!\left(\frac{\langle\mathbf u,\mathbf v\rangle}{d}\right),
    \qquad
    \mathbf u,\mathbf v\in S^{d-1}(\sqrt d),
\]
and independent \(\mathbf U_i\sim\nu_d\), define
\begin{equation}
    R_{K,d}(f;\mathbf U_{1:n})
    :=\inf_{\boldsymbol\beta\in\mathbb R^n}
    \left\|
        f-\sum_{i=1}^n\beta_iK_d(\cdot,\mathbf U_i)
    \right\|_{L^2(\nu_d)}^2.
    \label{eq:ntk-spherical-oracle-risk}
\end{equation}
Theorem~3 of \citet{ghorbani2020linearized} states that, for a fixed
integer \(m\ge0\), if
\[
    n\le d^{m+1-\delta_d},
    \qquad \delta_d>0,
    \qquad \delta_d^2\log d\longrightarrow\infty,
\]
then the following holds for any deterministic target sequence
\(f_d\in L^2(\nu_d)\) and any fixed kernel sequence whose spherical
sections are square integrable. For every fixed \(\epsilon>0\),
with probability tending to one,
\begin{equation}
    \begin{aligned}
    \Big|&R_{K,d}(f_d;\mathbf U_{1:n})
       -R_{K,d}(P_{\le m}f_d;\mathbf U_{1:n})
       -\|P_{>m}f_d\|_{L^2(\nu_d)}^2\Big|\\
    &\le
    \epsilon\|f_d\|_{L^2(\nu_d)}
             \|P_{>m}f_d\|_{L^2(\nu_d)}.
    \end{aligned}
    \label{eq:ntk-gmmm-risk-decomposition}
\end{equation}
The kernel and target are fixed before the sample is drawn.
The risk in~\eqref{eq:ntk-spherical-oracle-risk} minimizes over
all coefficients, rather than over coefficients obtained by a
particular training algorithm. No assumption that all low-degree
kernel eigenvalues are positive is required for this lower bound.

\begin{proof}[Proof of Proposition~\ref{prop:ntk-gaussian-lower-bound}]
Throughout the proof, \(k,L\), and hence \(q=k^{L-1}\), are fixed.
We abbreviate \(K_d^{\mathrm{NTK}}\) by \(K_d\).

\paragraph{Step 1: NTK Symmetry, Homogeneity, and Integrability.}
For a finite-width bias-free network with the specified activation,
let \(F_{\boldsymbol\vartheta}\) denote its output as a function
of all trainable parameters \(\boldsymbol\vartheta\).
Induction through its \(L-1\) activations gives
\[
    F_{\boldsymbol\vartheta}(s\mathbf x)
    =s^qF_{\boldsymbol\vartheta}(\mathbf x).
\]
Differentiating in the parameters shows that every tangent feature
has the same degree of input homogeneity:
\[
    \nabla_{\boldsymbol\vartheta}
    F_{\boldsymbol\vartheta}(s\mathbf x)
    =s^q\nabla_{\boldsymbol\vartheta}
    F_{\boldsymbol\vartheta}(\mathbf x).
\]
The finite-width tangent kernel, and therefore its deterministic
infinite-width limit, consequently satisfy
\begin{equation}
    K_d(s\mathbf x,t\mathbf z)
    =s^qt^qK_d(\mathbf x,\mathbf z),
    \qquad s,t>0.
    \label{eq:ntk-separate-homogeneity}
\end{equation}
An orthogonal change of input coordinates can be absorbed into
an orthogonal change of the first-layer Gaussian weights.
This preserves their initialization law and the Euclidean inner
product of parameter gradients. The finite-width kernel is thus
rotation invariant in distribution, and its deterministic limit
satisfies
\[
    K_d(\mathbf O\mathbf x,\mathbf O\mathbf z)
    =K_d(\mathbf x,\mathbf z),
    \qquad \mathbf O\in O(d).
\]
This is also the multilayer NTK invariance described in
\citet[Section~3]{ghorbani2020linearized}.
Its restriction to \(S^{d-1}(\sqrt d)\) therefore has the form
used in~\eqref{eq:ntk-spherical-oracle-risk}.
The diagonal values of this fixed-depth monomial NTK are finite,
since its covariance and derivative-covariance recursions involve
only finite Gaussian polynomial moments.
Positive semidefiniteness and rotational invariance give
\[
    |K_d(\mathbf u,\mathbf v)|^2
    \le K_d(\mathbf u,\mathbf u)K_d(\mathbf v,\mathbf v)
    =\psi_d(1)^2<\infty
\]
on the sphere. Thus the square-integrability hypothesis
of~\eqref{eq:ntk-gmmm-risk-decomposition} holds.

\paragraph{Step 2: An Exact Gaussian-to-Sphere Reduction.}
For the Gaussian training inputs, write
\[
    R_i:=\frac{\|\mathbf X_i\|_2}{\sqrt d},
    \qquad
    \mathbf U_i:=\frac{\sqrt d\,\mathbf X_i}{\|\mathbf X_i\|_2},
    \qquad \mathbf X_i=R_i\mathbf U_i.
\]
The vectors \(\mathbf U_i\) are independent with law \(\nu_d\),
and are independent of the radii \(R_i\), which are strictly
positive almost surely. By~\eqref{eq:ntk-separate-homogeneity},
\[
    K_d(\cdot,\mathbf X_i)=R_i^qK_d(\cdot,\mathbf U_i).
\]
Consequently, the Gaussian samples and their normalized directions
generate exactly the same kernel span:
\begin{equation}
    \operatorname{span}\{K_d(\cdot,\mathbf X_i):1\le i\le n\}
    =\operatorname{span}\{K_d(\cdot,\mathbf U_i):1\le i\le n\}.
    \label{eq:ntk-equal-kernel-spans}
\end{equation}
Every function in this span is \(q\)-homogeneous.
An independent test input has the decomposition
\(\mathbf X=R\mathbf U\), with \(\mathbf U\sim\nu_d\)
independent of \(R\). For each integer \(j\ge0\), set
\begin{equation}
    A_{d,j}:=\mathbb E[R^{2j}]
    =\frac{2^j\Gamma(d/2+j)}{d^j\Gamma(d/2)}
    =\prod_{a=0}^{j-1}\left(1+\frac{2a}{d}\right),
    \label{eq:ntk-radial-moments}
\end{equation}
where the empty product is one.
For any two \(q\)-homogeneous functions \(v,f\) whose restrictions
to the sphere are square integrable, independence and homogeneity
give the exact identity
\begin{equation}
    \mathbb E_{\mathbf X\sim\mathcal N(\mathbf0,\mathbf I_d)}
    [(v(\mathbf X)-f(\mathbf X))^2]
    =A_{d,q}\|v-f\|_{L^2(\nu_d)}^2.
    \label{eq:ntk-gaussian-spherical-risk}
\end{equation}
In particular, \(A_{d,q}\ge1\) and
\(A_{d,q}=1+O_q(d^{-1})\).
Combining~\eqref{eq:ntk-equal-kernel-spans}
and~\eqref{eq:ntk-gaussian-spherical-risk} yields, pathwise,
\begin{equation}
    \begin{aligned}
    &\inf_{\boldsymbol\alpha\in\mathbb R^n}
    \mathbb E_{\mathbf X\sim\mathcal N(\mathbf0,\mathbf I_d)}
    \left[
      \left(f_d^\star(\mathbf X)
       -\sum_{i=1}^n\alpha_iK_d(\mathbf X,\mathbf X_i)\right)^2
    \right]\\
    &\hspace{15mm}
    =A_{d,q}R_{K,d}(f_d^\star;\mathbf U_{1:n}).
    \end{aligned}
    \label{eq:ntk-gaussian-spherical-oracle-identity}
\end{equation}
Because the infimum permits every coefficient vector, this
identity supplies a lower bound even when the learned coefficients
use all sample radii, labels, and additional randomness.
It does not assert equality of the Gaussian and spherical
empirical ridge-regression objectives.

\paragraph{Step 3: An Admissible Constant-Width Target Family.}
Let \(c_0,c_1\) be the constants in
Proposition~\ref{prop:weight-based-assumptions}, and fix
\(t=c_0/162\), independently of \(d\).
For every \(d\ge3\), choose
\[
    \mathbf W_1=(\mathbf e_1,\mathbf e_2,\mathbf e_3),
    \qquad \mathbf a=\frac13\mathbf1_3,
    \qquad \mathbf W_j=\mathbf M_t\quad(2\le j\le L-1),
\]
where
\begin{equation}
    \mathbf M_t=
    \begin{pmatrix}
      1-7t/9 & t/3 & 4t/9\\
      t/3 & 1-8t/9 & 5t/9\\
      4t/9 & 5t/9 & 1-t
    \end{pmatrix}.
    \label{eq:ntk-admissible-weights}
\end{equation}
The maximum off-diagonal row sum is \(t\), and the consecutive
diagonal gaps equal \(t/9\). Since
\[
    3^2t^2\le c_0(t/9),
\]
Proposition~\ref{prop:weight-based-assumptions} applies and verifies
all three assumptions in the statement, with
\[
    \kappa_0=2,\qquad \mu=1,\qquad
    \gamma=\frac9{c_1t}.
\]
These parameters depend only on \(k,L\), not on \(d\).
The resulting target has the form
\begin{equation}
    f_d^\star(\mathbf x)=g(x_1,x_2,x_3),
    \label{eq:ntk-fixed-dimensional-target}
\end{equation}
where \(g\) is a fixed homogeneous polynomial of degree \(q\).
In particular, \(g\) itself does not change with \(d\).

Write
\(g(\mathbf u)=\sum_{|\boldsymbol\beta|=q}
 b_{\boldsymbol\beta}\mathbf u^{\boldsymbol\beta}\),
and set \(b_\star:=b_{(q,0,0)}=g(1,0,0)\).
Keeping only the path that uses the first neuron at every layer
gives
\begin{equation}
    b_\star\ge
    \frac13\left(1-\frac{7t}{9}\right)^{s_{k,L}}>0,
    \qquad s_{k,L}:=\sum_{j=1}^{L-2}k^j.
    \label{eq:ntk-positive-leading-coefficient}
\end{equation}
Column normalization, positivity, and the even activation also give
\[
    0\le g(\mathbf u)\le
    \max_{1\le j\le3}|u_j|^q.
\]
Hence, for independent standard Gaussian \(G_1,G_2,G_3\),
\[
    \mathbb E[g(G_1,G_2,G_3)^2]
    \le3\mathbb E[|G_1|^{2q}]=3(2q-1)!!.
\]
Using~\eqref{eq:ntk-gaussian-spherical-risk} with \(v=0\) yields
\begin{equation}
    \sup_{d\ge3}\|f_d^\star\|_{L^2(\nu_d)}\le M,
    \qquad M:=\sqrt{3(2q-1)!!}.
    \label{eq:ntk-target-norm-bound}
\end{equation}

\paragraph{Step 4: Nonvanishing Degree-\(q\) Spherical-Harmonic Energy.}
Let \(P_q=P_{\le q}-P_{\le q-1}\).
Its range is the space \(V_{d,q}\) of restrictions of homogeneous
harmonic polynomials of degree \(q\), with dimension
\[
    B(d,q)=\binom{d+q-1}{q}-\binom{d+q-3}{q-2}.
\]
Use the Gegenbauer normalization of
\citet[Section~5.2]{ghorbani2020linearized}, and define
\[
    Z_{q,d}(\mathbf u)
    :=\sqrt{B(d,q)}\,
      Q_q^{(d)}(\sqrt d\,u_1).
\]
Then \(Z_{q,d}\in V_{d,q}\) and
\(\|Z_{q,d}\|_{L^2(\nu_d)}=1\), by Eq.~(32) of that paper.
Consequently,
\begin{equation}
    \|P_qf_d^\star\|_{L^2(\nu_d)}^2
    \ge
    |\langle f_d^\star,Z_{q,d}\rangle_{L^2(\nu_d)}|^2.
    \label{eq:ntk-harmonic-test-bound}
\end{equation}
Eq.~(43) of \citet{ghorbani2020linearized} gives, for fixed \(q\),
the following coefficientwise convergence of polynomials:
\begin{equation}
    \sqrt{B(d,q)}Q_q^{(d)}(\sqrt d\,s)
    \longrightarrow\frac{\operatorname{He}_q(s)}{\sqrt{q!}}.
    \label{eq:ntk-gegenbauer-hermite-limit}
\end{equation}
Here \(\operatorname{He}_q\) is the probabilists' Hermite polynomial,
normalized by
\(\mathbb E[\operatorname{He}_q(G_1)^2]=q!\).

For completeness, convergence of all joint moments needed below
follows from an exact radial identity. For
\(\boldsymbol\beta\in\mathbb Z_{\ge0}^3\), writing
\(j=|\boldsymbol\beta|\), the decomposition
\(\mathbf X=R\mathbf U\) and~\eqref{eq:ntk-radial-moments} imply
\begin{equation}
    \mathbb E_{\nu_d}\left[
       U_1^{2\beta_1}U_2^{2\beta_2}U_3^{2\beta_3}\right]
    =\frac{\prod_{a=1}^3(2\beta_a-1)!!}{A_{d,j}}
    \longrightarrow\prod_{a=1}^3(2\beta_a-1)!!.
    \label{eq:ntk-joint-moment-limit}
\end{equation}
We use \((-1)!!=1\); moments with an odd exponent in any
coordinate vanish by reflection symmetry.
Since \(g\) and the polynomials
in~\eqref{eq:ntk-gegenbauer-hermite-limit} have fixed degree,
coefficientwise convergence and~\eqref{eq:ntk-joint-moment-limit}
yield
\begin{equation}
    \langle f_d^\star,Z_{q,d}\rangle_{L^2(\nu_d)}
    \longrightarrow
    \frac{\mathbb E[
       g(G_1,G_2,G_3)\operatorname{He}_q(G_1)]}{\sqrt{q!}}.
    \label{eq:ntk-harmonic-correlation-limit}
\end{equation}
Hermite orthogonality gives
\[
    \mathbb E[G_1^j\operatorname{He}_q(G_1)]=0
    \quad(0\le j<q),
    \qquad
    \mathbb E[G_1^q\operatorname{He}_q(G_1)]=q!.
\]
Every monomial of \(g\) has total degree \(q\).
Thus only \(b_\star u_1^q\) contributes to the expectation
in~\eqref{eq:ntk-harmonic-correlation-limit}; independence of the
Gaussian coordinates gives
\[
    \mathbb E[g(G_1,G_2,G_3)\operatorname{He}_q(G_1)]
    =b_\star q!.
\]
Combining this with~\eqref{eq:ntk-harmonic-test-bound} proves
\begin{equation}
    \liminf_{d\to\infty}
    \|P_qf_d^\star\|_{L^2(\nu_d)}^2
    \ge q!b_\star^2>0.
    \label{eq:ntk-positive-top-energy}
\end{equation}
In particular, with
\[
    E_0:=\frac{q!b_\star^2}{2}>0,
\]
we have
\(\|P_qf_d^\star\|_{L^2(\nu_d)}^2\ge E_0\)
for all sufficiently large \(d\).

\paragraph{Step 5: Apply the Spherical Theorem and Conclude.}
Fix \(\eta\in(0,q)\), and suppose \(n\le d^{q-\eta}\).
Apply~\eqref{eq:ntk-gmmm-risk-decomposition} with
\[
    m=q-1,\qquad \delta_d=\eta.
\]
The integer \(m\) is fixed, and
\(\delta_d^2\log d=\eta^2\log d\to\infty\).
Since \(f_d^\star\) has degree \(q\),
\(P_{>q-1}f_d^\star=P_qf_d^\star\).
The nonnegativity of
\(R_{K,d}(P_{\le q-1}f_d^\star;\mathbf U_{1:n})\) therefore gives,
with probability tending to one,
\[
    R_{K,d}(f_d^\star;\mathbf U_{1:n})
    \ge E_d-\epsilon\|f_d^\star\|_{L^2(\nu_d)}\sqrt{E_d},
    \qquad
    E_d:=\|P_qf_d^\star\|_{L^2(\nu_d)}^2.
\]
Choose the fixed theorem tolerance
\(\epsilon=\sqrt{E_0}/(2M)\).
Equations~\eqref{eq:ntk-target-norm-bound}
and~\eqref{eq:ntk-positive-top-energy} imply
\[
    \epsilon\|f_d^\star\|_{L^2(\nu_d)}\sqrt{E_d}
    \le\frac{\sqrt{E_0}\sqrt{E_d}}2
    \le\frac{E_d}2.
\]
Consequently,
\begin{equation}
    \mathbb P\!\left(
       R_{K,d}(f_d^\star;\mathbf U_{1:n})\ge\frac{E_0}{2}
    \right)\longrightarrow1.
    \label{eq:ntk-spherical-positive-risk}
\end{equation}
The exact identity~\eqref{eq:ntk-gaussian-spherical-oracle-identity},
together with \(A_{d,q}\ge1\), transfers this bound to the best
possible predictor in the Gaussian sample-centered kernel span.
Set
\[
    \varepsilon_\star:=\min\left\{\frac14,\frac{E_0}{4}\right\}.
\]
Then every predictor stated above has Gaussian risk at least
\(2\varepsilon_\star\) on the same event.
The event depends only on the sample directions and holds for
all coefficient vectors simultaneously, so additional algorithmic
randomness does not affect the lower bound.

In particular, the probability of achieving risk at most
\(\varepsilon_\star\) tends to zero whenever
\(n\le d^{q-\eta}\).
For a learner with success probability at least \(0.99\) for all
sufficiently large \(d\), even a subsequence satisfying this
sample-size inequality is impossible. It follows that its
worst-case sample complexity is
\(\Omega_{k,L,\eta}(d^{q-\eta})\) for every fixed \(\eta\in(0,q)\).
Equivalently, its necessary sample-size exponent satisfies
\[
    \liminf_{d\to\infty}\frac{\log n(d)}{\log d}\ge q.
\]
\end{proof}

\paragraph{Scope of the Conclusion.}
The hard targets are chosen before the data are sampled and satisfy
all of the manuscript's assumptions; no additional highest-degree
energy assumption is imposed on the entire target class.
The lower bound concerns fixed, sample-centered infinite-width NTK
regression, including zero-initialized kernel gradient descent.
It does not rule out passive feature-learning algorithms,
data-dependent kernels, or active queries.
The use of Theorem~3 also requires \(k,L\) to remain fixed as
\(d\to\infty\); the exponent statement alone is not a claim of an
\(\Omega(d^q)\) lower bound with no exponent slack.

\section{Experiment Details}
\label{app:activation_experiments}


We present a numerical illustration of parameter recovery on an
$8\rightarrow3\rightarrow3\rightarrow1$ polynomial network
with activation $\sigma(t)=t^k$ for $k\in\{4,6,8\}$.
All configurations use the same target weights and numerical settings, with only the activation degree changed.

\begin{figure}[H]
    \centering
    \includegraphics[width=\linewidth]{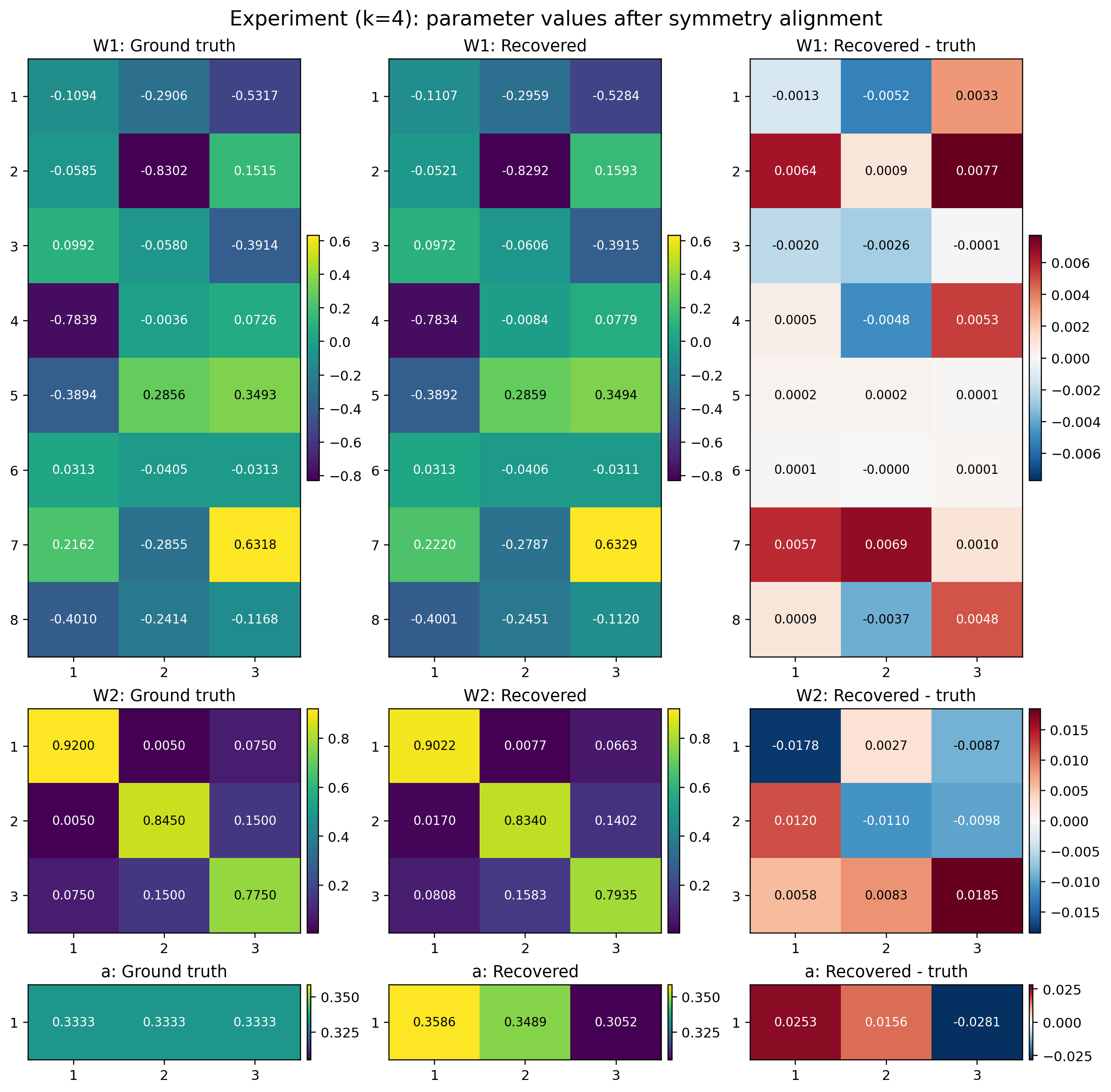}
    \caption{
    Parameter recovery for the $8\!\to\!3\!\to\!3\!\to\!1$ network
    with fourth-power activation ($k=4$). 
    }
    \label{fig:k4_weight_recovery}
\end{figure}

\begin{figure}[t]
    \centering
    \includegraphics[width=\linewidth]{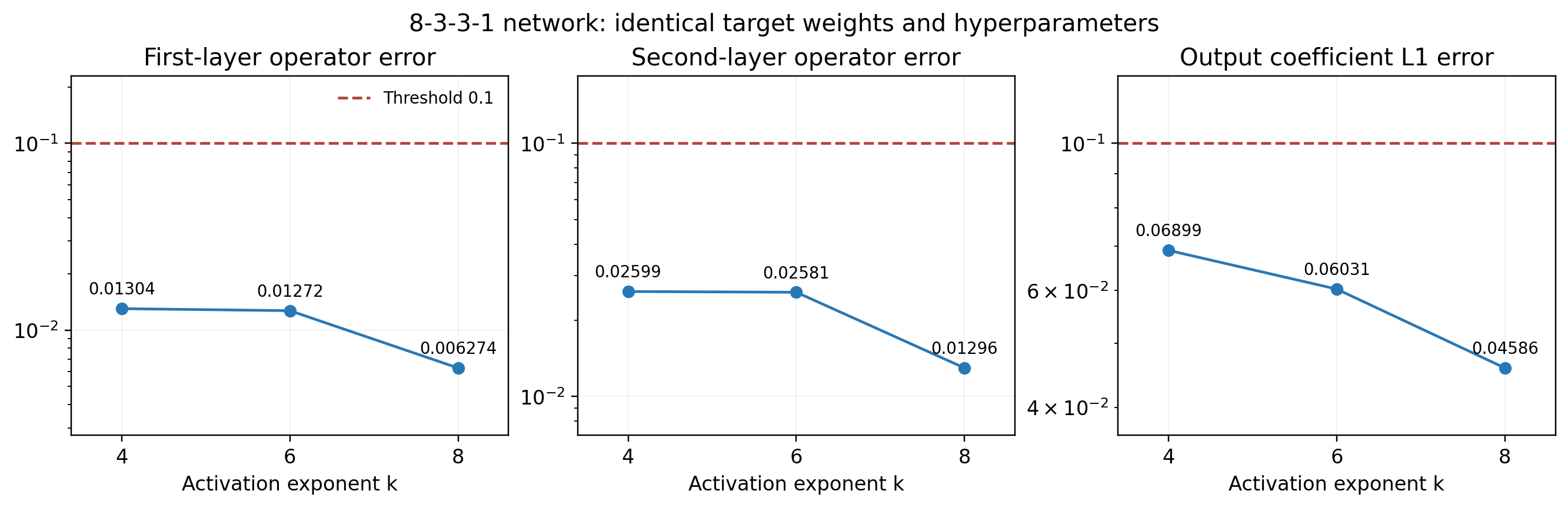}
    \caption{
    Recovery errors comparison for three activation degrees. 
    }
    \label{fig:activation_errors}
\end{figure}

The first hidden layer is recovered by column-space estimation followed by ASPIRE, using empirical position and gradient moments constructed from independent Hit-and-Run chains. The treatment of the second hidden layer and output layer differs slightly from the procedure analyzed in the main text: we use multiple Hessian evaluations for the second hidden layer and ordinary least squares for the output weights. Thus, the experiment is primarily intended to validate the core ASPIRE mechanism for recovering hidden directions.

Before evaluation, we align hidden-unit permutations across adjacent layers and resolve first-layer column signs.
Figure~\ref{fig:k4_weight_recovery} visualizes the recovered parameters for $k=4$, while Figure~\ref{fig:activation_errors} reports the recovery errors for all three activation degrees.
The results show accurate recovery across all tested settings, providing a numerical validation of the core ASPIRE procedure.

\end{document}